\documentclass{article}
\usepackage[a4paper,margin=25mm]{geometry}
\usepackage{graphicx} 
\usepackage[table]{xcolor}
\usepackage{xurl}
\usepackage{hyperref}
\usepackage{amssymb}
\usepackage{amsmath}
\usepackage{amsthm}
\usepackage[capitalize]{cleveref}
\usepackage{mathtools}
\usepackage{todonotes}
\usepackage{tikz-cd}
\usepackage{tikz-dimline}
\usetikzlibrary{positioning, calc, circuits.logic.US, circuits.logic.IEC, patterns, calc, decorations.pathreplacing, decorations.pathmorphing, backgrounds, shadows, arrows.meta, fit, shapes}
\usepackage{authblk}

\usepackage{array,booktabs,tabularx}
\usepackage{makecell}

\usepackage{rotating}
\usepackage{comment}

\usepackage{enumitem}

\usepackage{pgfplots} \pgfplotsset{compat=1.18}
\usepgfplotslibrary{groupplots}

\usepackage[backend=bibtex]{biblatex} 
\AtBeginBibliography{\emergencystretch=2em}

\usepackage{subcaption}

\theoremstyle{plain}
\newtheorem{theorem}{Theorem}
\newtheorem{corollary}[theorem]{Corollary}
\newtheorem{lemma}[theorem]{Lemma}
\newtheorem{definition}[theorem]{Definition}
\newtheorem{proposition}[theorem]{Proposition}
\theoremstyle{definition}
\newtheorem*{informalresult}{Results (informal)}

\numberwithin{equation}{section}

\usepackage{thmtools}
\makeatletter
\renewcommand\thmt@autorefsetup{
  \@xa\def\csname\thmt@envname autorefname\@xa\endcsname
    \@xa{\thmt@thmname}
}
\makeatother

\declaretheoremstyle[
  bodyfont=\normalfont,
  headfont=\bfseries
]{examplestyle}

\declaretheorem[style=examplestyle,sibling=theorem]{example}
\declaretheorem[style=examplestyle,sibling=theorem]{remark}

\AtBeginEnvironment{corollary}{\crefalias{theorem}{corollary}}
\AtBeginEnvironment{lemma}{\crefalias{theorem}{lemma}}
\AtBeginEnvironment{definition}{\crefalias{theorem}{definition}}
\AtBeginEnvironment{proposition}{\crefalias{theorem}{proposition}}
\AtBeginEnvironment{example}{\crefalias{theorem}{example}}
\AtBeginEnvironment{remark}{\crefalias{theorem}{remark}}

\DeclareMathOperator*{\argmin}{argmin}
\DeclareMathOperator{\sgn}{sgn}

\usepackage{colortbl}

\newcommand{\Pclass}{\mathsf{P}}
\newcommand{\NPclass}{\mathsf{NP}}
\newcommand{\FPclass}{\mathsf{FP}}
\newcommand{\SharpPclass}{\mathsf{\#P}}
\newcommand{\PH}{\mathsf{PH}}
\newcommand{\CoNPclass}{\mathsf{co}\text{-}\mathsf{NP}}
\newcommand{\Aclass}{\mathsf{A}}
\newcommand{\Bclass}{\mathsf{B}}
\newcommand{\CoAclass}{\mathsf{co}\text{-}\mathsf{A}}
\newcommand{\MAclass}{\mathsf{MA}}

\tikzset{
 vm/.style={font=\small,text=VMInk,line join=round,line cap=round},
 vm label/.style={font=\small\bfseries,inner sep=3pt,align=center},
 vm point/.style={circle,fill=VMInk,inner sep=0pt,minimum size=4pt},
 vm region/.style={line width=1pt,rounded corners=7pt}
}

\newcommand{\simul}{\longrightarrow}
\newcommand{\polysim}{\overset{ \mathrm{p}}{\longrightarrow}}
\newcommand{\vsim}{\overset{\mathrm{v}}{\longrightarrow}}
\newcommand{\vpolysim}{\overset{\mathrm{v}, \mathrm{p}}{\longrightarrow}}
\newcommand{\simset}{\mathcal{C}_{\mathrm{sim}}}

\AtBeginEnvironment{definition}{\let\emph\textbf}

\newcommand{\drawcustomlattice}[5]{
    \begin{scope}[shift={(#2,#3)}, scale=0.75, local bounding box=#1]
        \foreach \x in {1,...,#4} {
            \foreach \y in {1,...,#5} {
                \node[vis] (#1a\x-\y) at (\x, \y) {};
            }
        }
        \foreach \x in {1,...,#4} {
            \foreach \y in {1,...,#5} {
                \ifnum\x<#4
                    \pgfmathtruncatemacro\xa{\x+1}
                    \draw[-, thick,  black!80] (#1a\x-\y) to (#1a\xa-\y);
                \fi
                \ifnum\y<#5
                    \pgfmathtruncatemacro\ya{\y+1}
                    \draw[-, thick,  black!80] (#1a\x-\y) to (#1a\x-\ya);
                \fi
            }
        }
    \end{scope}
}

\newcommand{\drawRBM}[6]{
    \begin{scope}[shift={(#2,#3)}, scale=1, local bounding box=#1]
        \foreach \x in {1,...,#5} {
                \node[vis] (#1a\x-v) at (\x, 0) {};
        }
        \foreach \x in {1,...,#6} {
                \node[hid] (#1a\x-h) at (\x+#4, 1) {};
        }
        
        \foreach \x in {1,...,#5} {
            \foreach \y in {1,...,#6} {
                    \draw[-,  black!80] (#1a\x-v) to (#1a\y-h);
            }
        }
    \end{scope}
}

\newcommand{\drawDBM}[7]{
    \begin{scope}[shift={(#2,#3)}, scale=1, local bounding box=#1]
        \foreach \x in {1,...,#5} {
                \node[vis] (#1a\x-v) at (\x, 0) {};
        }
        \foreach \x in {1,...,#6} {
            \foreach \y in {1,...,#7} {
                \node[hid] (#1a\x-h\y) at (\x+#4, \y) {};
            }       
        }
        \foreach \x in {1,...,#5} {
            \foreach \y in {1,...,#6} {
                    \draw[-,  black!80] (#1a\x-v) to (#1a\y-h1);
            }
        }

        \foreach \x in {2,...,#7} {
            \foreach \y in {1,...,#6} {
                \foreach \z in {1,...,#6}{
                    \pgfmathtruncatemacro\xa{\x-1}
                    \draw[-,  black!80] (#1a\y-h\x) to (#1a\z-h\xa);
                }
            }
        }
        
    \end{scope}
}

\title{
Going deep, reaching further: What can Boltzmann machines efficiently represent? 
}

\author[1]{Tobias Reinhart\thanks{\url{tobias.reinhart@uibk.ac.at}}}

\author[1,2,3]{Gemma De les Coves}

\affil[1]{{\small Institute for Theoretical Physics, University of Innsbruck,\\ Technikerstr.\ 21a,\ A-6020 Innsbruck, Austria}}
\affil[2]{{\small ICREA, Instituci\'o Catalana d'Estudis i Recerca Avan\c{c}ats, Passeig Llu\'is Companys 23, 08010 Barcelona}}
\affil[3]{{\small Departament d'Enginyeria, Universitat Pompeu Fabra, Carrer T\`anger 122, 08018 Barcelona}}

\date{\today}

\begin{document}
\maketitle
\begin{abstract}
While universal approximation theorems theoretically underpin the representation capabilities of machine learning models, they have two shortcomings: They do not focus on efficient representations and they tend to be insensitive to depth and other heuristically important properties. In this work we show that these two shortcomings are intertwined. We prove that machine learning models which are deeper or have higher-order interactions can efficiently represent a larger set of functions, i.e.\ have larger efficient reach. It follows that deep Boltzmann machines (DBMs) have maximal efficient reach, restricted Boltzmann machines (RBMs) sub-maximal, and $k$-body generalizations of RBMs form a hierarchy between the two, even though all are universal. More generally, (i) we provide a full characterization of machine learning models with maximal efficient reach; (ii) we derive separations between deep and shallow models for three representation tasks; and (iii) we prove that the efficient reach of shallow models constitutes an infinite hierarchy indexed by the interaction order. The proofs draw on different methodologies, namely on (i) a recent characterization of universal spin models; (ii) computational complexity bounds for spin models; and (iii) communication complexity bounds applied to shallow models. This work shows that universal machine learning models can strongly differ in their efficient representation capabilities, and provides a theoretical foundation for heuristic observations.
\end{abstract}
\tableofcontents

\section{Introduction}\label{sec:intro}

Several machine learning models are theoretically supported by so-called \emph{universal approximation theorems} (UATs), that is, proofs that the respective model can represent or approximate any function from a certain function class. Restricted Boltzmann machines (RBMs), deep Boltzmann machines (DBMs) and deep belief networks (DBNs) are universal approximators of probability distributions \cite{Lr08, Ro10, Mo15}, feedforward neural networks with at least one hidden layer are universal approximators of continuous functions with compact support \cite{Ho89,Cy89}, and transformers are universal approximators of sequence-to-sequence functions \cite{Yu20}.

Although theoretically foundational, practically, these UATs have two major shortcomings. First, they do not consider the efficiency of the representation. For instance, the UATs for RBMs, DBMs and DBNs, require exponentially many hidden units and hence exponentially many parameters to represent a generic target function \cite{Lr08, Ro10, Mo15}.
Consequently, they provide limited insight into the representation capabilities of efficient models. 
Second, these UATs often are insensitive to heuristically important model hyperparameters such as depth.
For example UATs exist for feedforward networks with just one hidden layer as well as for both RBMs and DBMs. They therefore do not explain the widely observed superiority of deep models \cite{Mh17, Sa15, Te16, Ro15}. 
This raises two questions: 
\begin{quote}
  \emph{Are all universal machine learning models equal with regard to their efficient representation capabilities?} \hfill (Q1) \phantomsection \label{Q1}

  \emph{If not, what properties characterize ``efficient universal models''? 
  } 
  \hfill (Q2) \phantomsection \label{Q2}
\end{quote}
For \hyperref[Q2]{(Q2)}, we are thinking of properties such as depth or other model hyperparameters.  

Recently, we have shown that UATs of certain, energy-based machine learning models such as RBMs, DBMs and generalizations thereof can be proven by means of a characterization of so-called universal spin models \cite{Re25}. The latter are spin models which can encode arbitrary energy functions in their low-energy sector.
This way of deriving UATs has two main advantages for addressing \hyperref[Q1]{(Q1)} and \hyperref[Q2]{(Q2)}. 
First, the characterization of universal spin models provides a handle on the resource requirements of UATs in the following sense. Spin models have a local structure that is respected by the universality characterization. If the target is compatible with this local structure, the derived UAT  will only require a polynomial overhead. 
Second, the characterization of universal spin models is fundamentally based on a notion of composition. As a consequence, the resulting UATs should favour deep models, since depth arises by constructing representations in a compositional way, layer by layer.

In this work, we leverage the characterization of universal spin models to tackle \hyperref[Q1]{(Q1)} and \hyperref[Q2]{(Q2)}. 
We find the following answers: 
\begin{quote}
\emph{No, they have different efficient representation capabilities, forming a rich landscape (see \cref{fig:polynomial-reach}).}\hfill (A1)

\emph{The properties that determine the efficient representation capabilities are depth (or, more precisely, visible closure) and interaction order. 
}\hfill (A2)
\end{quote}
Our work provides a theoretical underpinning for the heuristic observation that deep models perform better. It also shows that the two shortcomings of UATs mentioned above are deeply intertwined. 
More generally, it illustrates how insightful connecting different fields through a common notion of universality can be\,---\,the characterization of universal spin models gives us an edge in the study of ``efficient universality''  of machine learning models. 

\begingroup
\definecolor{PRpurple}{HTML}{8866AA}
\definecolor{VMBlue}{HTML}{628FB5}
\definecolor{VMInk}{HTML}{28343D}
\definecolor{VMUniversalColour}{HTML}{A9CCE3}
\definecolor{VMMaximalColour}{HTML}{628FB5}
\definecolor{VMVisibleClosedColour}{HTML}{D0E3EF}
\definecolor{VMIntersectionColour}{HTML}{285577}
\definecolor{VMShallow}{HTML}{4F9487}
\definecolor{VMDeep}{HTML}{BE995B}
\definecolor{VMShallowHigher}{HTML}{83B5AB}
\definecolor{VMShallowLower}{HTML}{B5D3CB}
\definecolor{VMRBM}{HTML}{E2EEE9}
\definecolor{VMReachSurface}{HTML}{D6D2C9}
\definecolor{VMGuideFace}{HTML}{ECE9E3}
\definecolor{VMLimitPlane}{HTML}{E8E8E8}
\colorlet{VMShallowFill}{VMShallow!20}
\colorlet{VMDeepFill}{VMDeep!15}
\colorlet{VMMaximalFill}{VMMaximalColour}
\colorlet{VMIntersectionFill}{VMIntersectionColour!40}
\tikzset{
 pr visible/.style={circle,draw=black!80,fill=blue!60,line width=.65pt,minimum size=5.5pt,inner sep=0pt},
 pr hidden/.style={circle,draw=black!80,fill=red!60,line width=.65pt,minimum size=5.5pt,inner sep=0pt},
 pr edge/.style={draw=black!35,line width=.45pt},
 pr leader/.style={draw=VMInk!55,line width=.65pt,rounded corners=10pt},
 pics/pr rbm/.style={code={
  \foreach \v in {0,1,2} \foreach \h in {0,1,2,3,4}
   \draw[pr edge] ({.7*\v-.7},0)--({.47*\h-.94},1);
  \foreach \v in {0,1,2} \node[pr visible] at ({.7*\v-.7},0) {};
  \foreach \h in {0,1,2,3,4} \node[pr hidden] at ({.47*\h-.94},1) {};
 }},
 pics/pr dbm/.style={code={
  \foreach \v in {0,1,2} \foreach \h in {0,1,2,3,4}
   \draw[pr edge] ({.7*\v-.7},0)--({.47*\h-.94},.8);
  \foreach \h in {0,1,2,3,4} \foreach \j in {0,1,2,3}
   \draw[pr edge] ({.47*\h-.94},.8)--({.55*\j-.825},1.6);
  \foreach \v in {0,1,2} \node[pr visible] at ({.7*\v-.7},0) {};
  \foreach \h in {0,1,2,3,4} \node[pr hidden] at ({.47*\h-.94},.8) {};
  \foreach \j in {0,1,2,3} \node[pr hidden] at ({.55*\j-.825},1.6) {};
 }},
 pics/pr triple/.style={code={
  \coordinate (v0) at (-.8,0);
  \coordinate (v1) at (0,-.28);
  \coordinate (v2) at (.8,0);
  \foreach \h in {0,1,2,3,4} {
   \coordinate (h\h) at ({.52*\h-1.04},1.05);
   \foreach \a/\b in {0/1,0/2,1/2}
    \path[fill=PRpurple,fill opacity=.035,draw=PRpurple,
          draw opacity=.20,line width=.4pt]
     (v\a)--(v\b)--(h\h)--cycle;
  }
  \foreach \v in {0,1,2} \node[pr visible] at (v\v) {};
  \foreach \h in {0,1,2,3,4} \node[pr hidden] at (h\h) {};
 }}
}
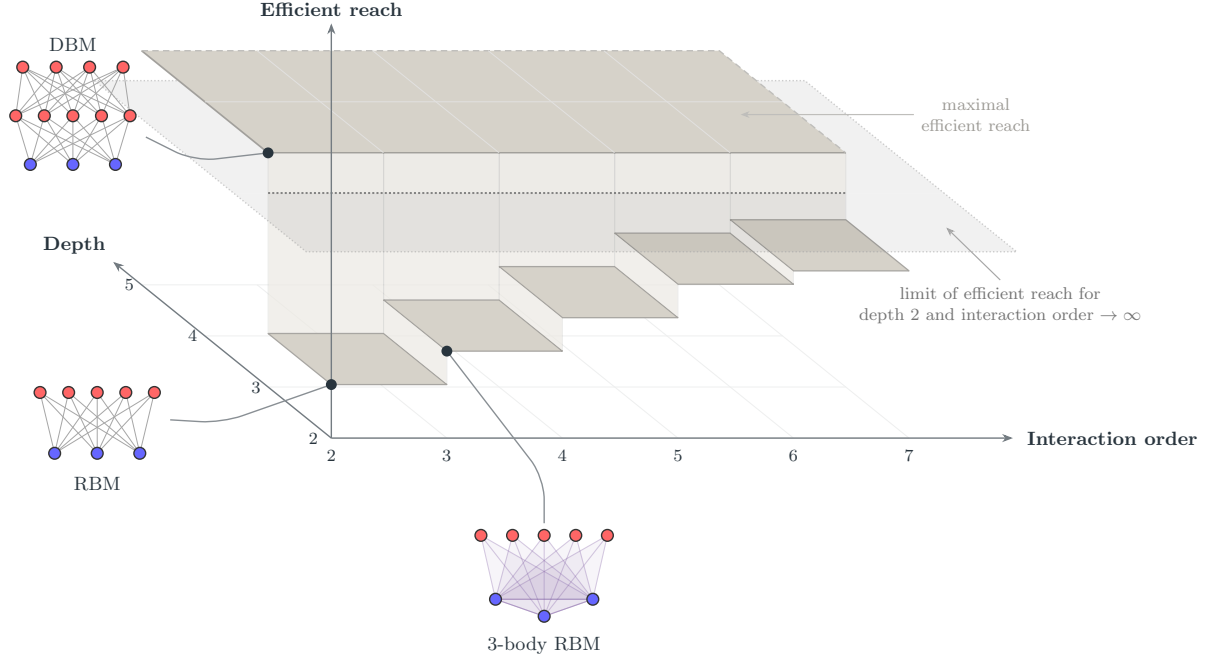
\begin{figure}[!t]
\centering
\resizebox{\linewidth}{!}{
\begin{tikzpicture}[font=\small,text=VMInk,line join=round,
 x={(.95cm,0cm)},y={(-.52cm,.42cm)},z={(0cm,1.1cm)},
 axis/.style={draw=VMInk!65,line width=.65pt,-{Stealth[length=1.8mm]}},
 point/.style={circle,fill=VMInk,inner sep=1.8pt}]
\foreach \x in {0,2,4,6,8,10}
 \draw[VMInk!8,thin] (\x,0,0)--(\x,6,0);
\foreach \y in {0,2,4,6}
 \draw[VMInk!8,thin] (0,\y,0)--(10,\y,0);
\draw[axis] (0,0,0)--(11.8,0,0) node[anchor=west,xshift=3pt] {\textbf{Interaction order}};
\draw[axis] (0,0,0)--(0,6.9,0) node[above left] {\textbf{Depth}};
\foreach \x/\t in {0/2,2/3,4/4,6/5,8/6,10/7}
 \node[below,font=\footnotesize] at (\x,0,-.06) {$\t$};
\foreach \y/\t in {2/3,4/4,6/5}
 \node[left,font=\footnotesize] at (0,\y,0) {$\t$};
\node[left,font=\footnotesize] at (-.1,0,0) {$2$};

\path[fill=black,fill opacity=.055,draw=none]
 (-.6,2,2.9)--(11.7,2,2.9)--(11.7,6.4,2.9)--(-.6,6.4,2.9)--cycle;
\draw[black!35,draw opacity=.55,densely dotted,line width=.65pt]
 (-.6,6.4,2.9)--(11.7,6.4,2.9)--(11.7,2,2.9);

\path[fill=VMReachSurface,draw=none]
 (0,2,3.5)--(10,2,3.5)--(10,6,3.5)--(0,6,3.5)--cycle;
\draw[VMReachSurface!65!black!80,line width=.8pt]
 (0,6,3.5)--(0,2,3.5)--(10,2,3.5);
\draw[VMReachSurface!65!black!55,densely dashed,line width=.65pt]
 (0,6,3.5)--(10,6,3.5)--(10,2,3.5);
\foreach \x in {2,4,6,8}
 \draw[VMReachSurface!65!black!25,line width=.35pt] (\x,2,3.5)--(\x,6,3.5);
\draw[VMReachSurface!65!black!25,line width=.35pt] (0,4,3.5)--(10,4,3.5);

\foreach \a/\b/\h in {8/10/2.5,6/8/2.3,4/6/1.8,2/4/1.3,0/2/.8} {
 \path[fill=VMGuideFace,fill opacity=.7,draw=VMReachSurface!65!black,
       draw opacity=.16,line width=.45pt]
  (\a,2,\h)--(\b,2,\h)--(\b,2,3.5)--(\a,2,3.5)--cycle;
 \path[fill=VMReachSurface,draw=VMReachSurface!65!black!75,line width=.6pt]
  (\a,0,\h)--(\b,0,\h)--(\b,2,\h)--(\a,2,\h)--cycle;
}
\foreach \x/\lo/\hi in {2/.8/1.3,4/1.3/1.8,6/1.8/2.3,8/2.3/2.5}
 \path[fill=VMGuideFace,fill opacity=.7,draw=VMReachSurface!65!black,
       draw opacity=.16,line width=.45pt]
  (\x,0,\lo)--(\x,2,\lo)--(\x,2,\hi)--(\x,0,\hi)--cycle;

\path[fill=black,fill opacity=.055,draw=none]
 (-.6,-.3,2.9)--(11.7,-.3,2.9)--(11.7,2,2.9)--(-.6,2,2.9)--cycle;
\draw[black!35,draw opacity=.55,densely dotted,line width=.65pt]
 (-.6,6.4,2.9)--(-.6,-.3,2.9)--(11.7,-.3,2.9)--(11.7,2,2.9);
\draw[black!55,draw opacity=.85,densely dotted,line width=.95pt]
 (0,2,2.9)--(10,2,2.9);

\coordinate (separator edge) at (11,.65,2.9);
\begin{scope}[x={(1cm,0cm)},y={(0cm,1cm)}]
\node[font=\footnotesize,anchor=north,align=center,text=black!55]
  (separator label) at (11,2.6) {limit of efficient reach for\\ depth $2$ and interaction order $\to \infty$};
 \draw[draw=black!35,line width=.55pt,-{Stealth[length=1.5mm]}]
  (separator label.north)--(separator edge);
\end{scope}

\coordinate (max edge) at (9,3.5,3.5);
\node[font=\footnotesize,anchor=west,align=center,text=VMReachSurface!65!black!85]
  (max label) at (12,3.5,3.5) {maximal\\
  efficient reach};
\draw[draw=VMReachSurface!65!black!55,line width=.55pt,-{Stealth[length=1.5mm]}]
   (max label)--(max edge);

\draw[axis] (0,0,0)--(0,0,6.2) node[above] {\textbf{Efficient reach}};
\coordinate (RBM) at (0,0,.8);
\coordinate (TRBM) at (2,0,1.3);
\coordinate (DBM) at (0,2,3.5);
\begin{scope}[x={(1cm,0cm)},y={(0cm,1cm)}]
 \draw[pr leader] (-2.65,.3)--(-1.8,.25)--(RBM);
 \draw[pr leader] (-3.05,4.95)--(-2.25,4.55)--(DBM);
 \draw[pr leader] (3.5,-1.4)--(3.5,-.65)--(TRBM);
 \begin{scope}[shift={(-3.85,-.25)}]
  \pic {pr rbm};
  \node[] at (0,-.5) {RBM};
 \end{scope}
 \begin{scope}[shift={(-4.25,4.5)}]
  \pic {pr dbm};
  \node[] at (0,1.98) {DBM};
 \end{scope}
 \begin{scope}[shift={(3.5,-2.65)}]
  \pic {pr triple};
  \node[] at (0,-.75) {3-body RBM};
 \end{scope}
\end{scope}
\node[point] at (RBM) {};
\node[point] at (TRBM) {};
\node[point] at (DBM) {};
\end{tikzpicture}
}
\caption{\textbf{Efficient reach of machine learning models.} 
The efficient reach (vertical axis), i.e.\ the efficient representation capabilities of several machine learning models, as a function of their depth and interaction order (horizontal axes), where the latter is the number of units involved in each interaction. 
Height depicts ordinal inclusion, so that larger efficient reach means the respective models can efficiently represent more targets. 
We prove that, for depth $\geq 3$, there are models whose efficient reach is maximal, irrespective of the interaction order.  
Models of depth 2 form an infinite hierarchy for increasing interaction order that saturates strictly below the maximal efficient reach (gray limit surface for depth $2$ and interaction order $\to \infty$). 
RBMs (depth $2$ and interaction order $2$), $3$-body RBMs (depth $2$ and interaction order $3$) and DBMs (depth $\geq 3$ and interaction order $2$) are located in the figure with visible spins/units in blue and hidden spins/units in red. They all have different efficient reach.}
\label{fig:polynomial-reach}
\end{figure}
\endgroup

\paragraph{Statement of results.} We study the efficient reach, i.e.\ the efficient representation capabilities, of machine learning models from three different angles. 
\begin{enumerate}[label=(\roman*)]
    \item \textbf{Characterization of maximal models.} 
    Models which have maximal efficient reach are called \emph{maximal}. Here we provide a characterization of maximal models. 
    More precisely, we study the efficient reach of spin and machine learning models through the notion of \emph{simulations} (that is, encoding target functions in the low-energy sector) 
    and the characterization of universal spin models, using that simulations imply \emph{approximations} (in terms of marginal distributions).
    We prove that, for spin models, universality and maximality are equivalent (\cref{thm:univ eff spin}), so
    the characterization of universality also applies to maximal efficient reach.
    In contrast, for machine learning models with visible and hidden spins (or units) another property becomes decisive: 
    \emph{visible closure} (\cref{def:visible closed}), which formalizes the intuition that hidden-hidden interactions are required to compose simulations and approximations.
    We then prove that any two of the three properties of \emph{universality}, \emph{visible closure} and
    \emph{maximality} imply the third (\cref{thm:main}, see also \cref{fig:main}).
    Thus, among universal machine learning models, visible closure precisely
    characterizes those with maximal efficient reach.
    This extends the connection between universality and universal
    approximation established in \cite{Re25} to the setting of
    polynomial resources. It follows that DBMs with at least two
    hidden layers are maximal (\cref{cor:DBM max}).
    
    \item \textbf{Separation between shallow and maximal models.} 
    
    We establish separations for three representation tasks:
    simulating functions (in the low-energy sector), approximating functions as
    unnormalized visible spin marginals, and approximating functions as normalized
    visible spin marginals (see \cref{tab:complexity separation}).     
    All three are based on reasonable assumptions on computational complexity, namely the non-collapse of the polynomial hierarchy.\footnote{This is in a similar spirit as the $\Pclass\neq\NPclass$ conjecture; for details see \cref{ssec:comp compl basics}.}
    For the first task, we show that maximal models can efficiently simulate
    certain $\NPclass$-hard functions that shallow models cannot (\cref{thm:separation 1}).
    For the latter two tasks, on the one hand we show that maximal models 
    can approximate certain $\SharpPclass$-hard functions up to a common positive factor, and
    the corresponding normalized distributions to arbitrary
    uniform relative precision, provided that the respective maximal models are visible closed in a strong sense. On the other hand, 
    shallow models cannot achieve these two approximations efficiently (\cref{thm:separation 2,thm:separation 3}). 
    Consequently, these separation results imply that shallow models such as RBMs do not have maximal efficient reach.
    
    \item \textbf{Separation hierarchy for shallow models.}
    
    We establish that the efficient simulation reach of shallow models is described by an infinite hierarchy indexed by their interaction order, that is, their maximal number of spins per local interaction (see \cref{fig:efficient-reach}). Namely, for each $k\geq 2$ there exist functions that shallow $k$-local models cannot efficiently simulate (\cref{thm:shallow rand,cor:rbm ip}) but shallow $k+1$-local models can (\cref{lem:IP k sim}), under reasonable assumptions on the size of the model weights.
To derive these results, we relate the efficient reach of shallow $k$-local models to sums of polynomial threshold functions (PTFs) of degree at most $k-1$ (\cref{lem:sim PTF}), and use communication complexity bounds to
    exclude efficient simulations of certain target functions, such as the $k$-party inner product modulo $2$ (\cref{lem:shallow PP,thm:shallow rand}).
    We then prove that increasing the interaction order from $k$ to $k+1$
    allows one to efficiently simulate these targets, by first
    characterizing the efficient reach of RBMs through
    sums of small ridge functions (\cref{lem:rbm ridge}), and then extending them to shallow $k$-local models (\cref{lem:RBM decomp}).
\end{enumerate}

\begingroup
\definecolor{VMBlue}{HTML}{628FB5}
\definecolor{VMInk}{HTML}{28343D}
\definecolor{VMUniversalColour}{HTML}{A9CCE3}
\definecolor{VMMaximalColour}{HTML}{628FB5}
\definecolor{VMVisibleClosedColour}{HTML}{D0E3EF}
\definecolor{VMIntersectionColour}{HTML}{285577}
\definecolor{VMShallow}{HTML}{4F9487}
\definecolor{VMDeep}{HTML}{BE995B}
\definecolor{VMShallowHigher}{HTML}{83B5AB}
\definecolor{VMShallowLower}{HTML}{B5D3CB}
\definecolor{VMRBM}{HTML}{E2EEE9}
\definecolor{VMReachSurface}{HTML}{D6D2C9}
\definecolor{VMGuideFace}{HTML}{ECE9E3}
\definecolor{VMLimitPlane}{HTML}{E8E8E8}
\colorlet{VMShallowFill}{VMShallow!20}
\colorlet{VMDeepFill}{VMDeep!15}
\colorlet{VMMaximalFill}{VMMaximalColour}
\colorlet{VMIntersectionFill}{VMIntersectionColour!40}
\tikzset{
 vm/.style={font=\small,text=VMInk,line join=round,line cap=round},
 vm label/.style={font=\small\bfseries,inner sep=3pt,align=center},
 vm point/.style={circle,fill=VMInk,inner sep=0pt,minimum size=4pt}
}
\definecolor{PRpurple}{HTML}{8866AA}
\tikzset{
 pr visible/.style={circle,draw=black!80,fill=blue!60,line width=.65pt,minimum size=5.5pt,inner sep=0pt},
 pr hidden/.style={circle,draw=black!80,fill=red!60,line width=.65pt,minimum size=5.5pt,inner sep=0pt},
 pr edge/.style={draw=black!35,line width=.45pt},
 pr leader/.style={draw=VMInk!55,line width=.65pt,rounded corners=10pt},
 pics/pr rbm/.style={code={
  \foreach \v in {0,1,2} \foreach \h in {0,1,2,3,4}
   \draw[pr edge] ({.7*\v-.7},0)--({.47*\h-.94},1);
  \foreach \v in {0,1,2} \node[pr visible] at ({.7*\v-.7},0) {};
  \foreach \h in {0,1,2,3,4} \node[pr hidden] at ({.47*\h-.94},1) {};
 }},
 pics/pr dbm/.style={code={
  \foreach \v in {0,1,2} \foreach \h in {0,1,2,3,4}
   \draw[pr edge] ({.7*\v-.7},0)--({.47*\h-.94},.8);
  \foreach \h in {0,1,2,3,4} \foreach \j in {0,1,2,3}
   \draw[pr edge] ({.47*\h-.94},.8)--({.55*\j-.825},1.6);
  \foreach \v in {0,1,2} \node[pr visible] at ({.7*\v-.7},0) {};
  \foreach \h in {0,1,2,3,4} \node[pr hidden] at ({.47*\h-.94},.8) {};
  \foreach \j in {0,1,2,3} \node[pr hidden] at ({.55*\j-.825},1.6) {};
 }},
 pics/pr triple/.style={code={
  \coordinate (v0) at (-.8,0);
  \coordinate (v1) at (0,-.28);
  \coordinate (v2) at (.8,0);
  \foreach \h in {0,1,2,3,4} {
   \coordinate (h\h) at ({.52*\h-1.04},1.05);
   \foreach \a/\b in {0/1,0/2,1/2}
    \path[fill=PRpurple,fill opacity=.035,draw=PRpurple,
          draw opacity=.20,line width=.4pt]
     (v\a)--(v\b)--(h\h)--cycle;
  }
  \foreach \v in {0,1,2} \node[pr visible] at (v\v) {};
  \foreach \h in {0,1,2,3,4} \node[pr hidden] at (h\h) {};
 }}
}
\begin{figure}[!t]
\centering
\resizebox{\linewidth}{!}{
\begin{tikzpicture}[vm]
 \path[fill=VMDeepFill,draw=black!30,rounded corners=6pt]
  (0,0) rectangle (14,8);
 \begin{scope}
  \clip[rounded corners=6pt] (0,0) rectangle (14,8);
  \fill[VMShallowFill] (0,0) rectangle (4.8,8);
 \end{scope}
 \draw[black!25] (4.8,0)--(4.8,8);
 \node[vm label] at (2.4,7.62) {shallow};
 \node[vm label] at (11.5,7.62) {deep};

 \coordinate (A) at (5.1,6.55);
 \coordinate (B) at (5.1,2.8);
 \coordinate (C) at (10.8,6.55);
 \def\VMAB{(A) .. controls (5.5,5.35) and (5.5,4.0) .. (B)}
 \def\VMBC{.. controls (7.8,2.8) and (10.8,4.25) .. (C)}
 \def\VMCA{.. controls (9.1,6.1) and (6.9,6.1) .. (A)}
 \def\VMUniversal{(A)
  .. controls (6.9,6.1) and (9.1,6.1) .. (C)
  .. controls (10.8,4.25) and (7.8,2.8) .. (B)
  .. controls (2.5,2.5) and (.7,2.9) .. (.7,4.75)
  .. controls (.7,6.8) and (3.2,7.1) .. (A) -- cycle}
 \def\VMClosed{\VMAB
  .. controls (4.8,2.15) and (4.35,1.75) .. (4.35,1.3)
  .. controls (4.35,.65) and (4.85,.55) .. (5.5,.55)
  -- (12.65,.55)
  .. controls (13.25,.55) and (13.4,.8) .. (13.4,1.4)
  -- (13.4,6.35)
  .. controls (13.4,7.1) and (12.6,7.35) .. (C)
  \VMCA -- cycle}
 \def\VMMaximal{\VMAB \VMBC
  .. controls (10.8,7.35) and (5.3,7.35) .. (A) -- cycle}

 \path[fill=VMUniversalColour] \VMUniversal;
 \path[fill=VMVisibleClosedColour] \VMClosed;
 \path[fill=VMMaximalFill] \VMMaximal;
 \begin{scope}
  \clip \VMUniversal;
  \clip \VMClosed;
  \clip \VMMaximal;
  \fill[VMIntersectionFill] (0,0) rectangle (14,8);
 \end{scope}
 \draw[VMUniversalColour!65!black,line width=.85pt] \VMUniversal;
 \draw[VMVisibleClosedColour!65!black,line width=.85pt] \VMClosed;
 \draw[VMMaximalColour!80!black,line width=.85pt] \VMMaximal;

 \node[vm label] at (2.8,6.15) {universal};
 \node[vm label] at (8.0,6.82) {maximal};
 \node[vm label] at (10.5,1.4) {visible closed};
 \coordinate (RBM) at (2.15,4.85);
 \coordinate (TRBM) at (3.85,3.9);
 \coordinate (DBM) at (8.1,4.75);
 \draw[pr leader]
  (-.55,5.0)--(.35,5.0)--(RBM);
 \draw[pr leader]
  (-.55,1.8)--(.35,1.8)--(TRBM);
 \draw[pr leader]
  (14.55,4.9)--(13.7,4.9)--(DBM);

 \pic[shift={(-1.85,4.55)},scale=.85] {pr rbm};
 \node[font=\small] at (-1.85,4.05) {RBM};
 \pic[shift={(-1.85,1.45)},scale=.85] {pr triple};
 \node[font=\small] at (-1.85,.75) {3-body RBM};
 \pic[shift={(15.85,4.25)},scale=.85] {pr dbm};
 \node[font=\small] at (15.85,3.75) {DBM};
 \foreach \p in {RBM,TRBM,DBM} \node[vm point] at (\p) {};
\end{tikzpicture}
}
\caption{\textbf{Characterization of maximal models.}
Set diagram showing properties of machine learning models and their interplay in our characterization of maximal models. Inclusions are not necessarily strict.
The three properties given by universal, visible closed and maximal, are such that any two of them imply the third. The area where the three properties hold is depicted in a 
 blue-gray shade.
Maximal models are necessarily deep.
DBMs are maximal. 
RBMs and $3$-body RBMs are universal, but are not deep, and thus not maximal. 
We expect most ``natural'' deep models to be visible closed, and hence most deep universal models to be maximal. 
}
\label{fig:main}
\end{figure}
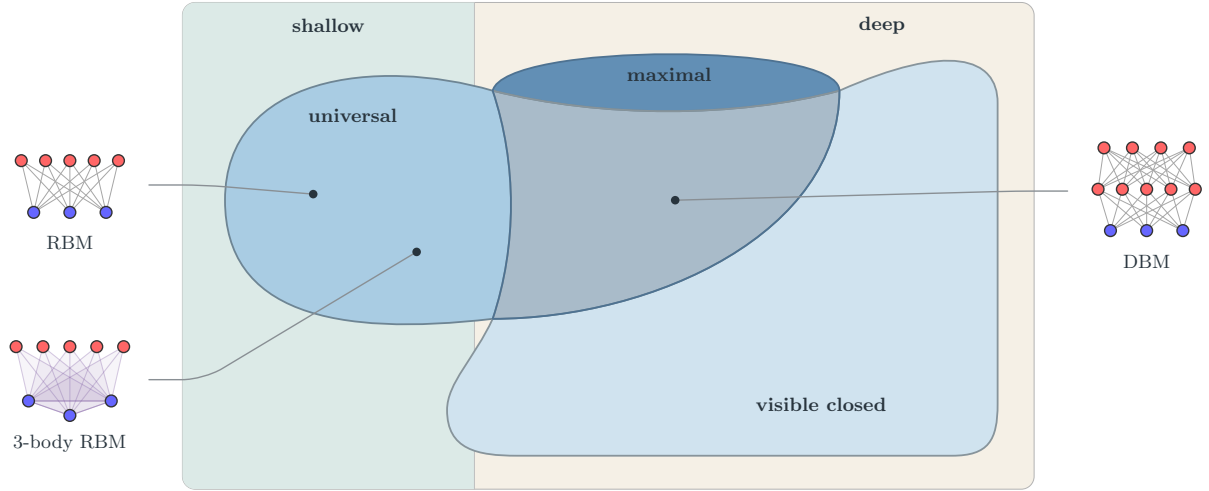
\endgroup

\begingroup
\definecolor{VMBlue}{HTML}{628FB5}
\definecolor{VMInk}{HTML}{28343D}
\definecolor{VMUniversalColour}{HTML}{A9CCE3}
\definecolor{VMMaximalColour}{HTML}{628FB5}
\definecolor{VMVisibleClosedColour}{HTML}{D0E3EF}
\definecolor{VMIntersectionColour}{HTML}{285577}
\definecolor{VMShallow}{HTML}{4F9487}
\definecolor{VMDeep}{HTML}{BE995B}
\definecolor{VMShallowHigher}{HTML}{83B5AB}
\definecolor{VMShallowLower}{HTML}{B5D3CB}
\definecolor{VMRBM}{HTML}{E2EEE9}
\definecolor{VMReachSurface}{HTML}{D6D2C9}
\definecolor{VMGuideFace}{HTML}{ECE9E3}
\definecolor{VMLimitPlane}{HTML}{E8E8E8}
\colorlet{VMShallowFill}{VMShallow!20}
\colorlet{VMDeepFill}{VMDeep!15}
\colorlet{VMMaximalFill}{VMMaximalColour}
\colorlet{VMIntersectionFill}{VMIntersectionColour!40}

\begin{table}[th]
\centering
\small
\renewcommand{\arraystretch}{1.6}
\setlength{\tabcolsep}{4pt}
\begin{tabularx}{\linewidth}{@{}
    X
    >{\centering\arraybackslash}p{1.25cm}
    >{\centering\arraybackslash}p{2.3cm}
    >{\centering\arraybackslash}p{3.05cm}
@{}}
\toprule
\textbf{Representation task}
& \cellcolor{VMShallowFill}\textbf{shallow}
& \cellcolor{VMMaximalFill}\textbf{maximal}
& \cellcolor{VMIntersectionFill}
  \makecell[c]{\textbf{maximal}\\
              \textbf{and strongly}\\
              \textbf{visible closed}}
\\
\midrule

Simulate $\NPclass$-hard ground-state energies $G_n$
(\cref{thm:separation 1})
& \cellcolor{VMShallowFill}$\times$ No
& \cellcolor{VMMaximalFill}$\checkmark$ Yes
& \cellcolor{VMIntersectionFill}$\checkmark$ Yes
\\

Approximate $\SharpPclass$-hard partition functions $Z_n$
in unnormalized marginals
(\cref{thm:separation 2})
& \cellcolor{VMShallowFill}$\times$ No
& \cellcolor{VMMaximalFill}---
& \cellcolor{VMIntersectionFill}$\checkmark$ Yes
\\

Approximate $\SharpPclass$-hard distributions $q_n\propto Z_n$
in normalized marginals
(\cref{thm:separation 3})
& \cellcolor{VMShallowFill}$\times$ No
& \cellcolor{VMMaximalFill}---
& \cellcolor{VMIntersectionFill}$\checkmark$ Yes
\\
\bottomrule
\end{tabularx}
\caption{\textbf{Separation between shallow and maximal models.}
We consider three representation tasks and prove separations between shallow and maximal models. 
$\NPclass$-hard functions such as the Ising ground-state function $G_n$ can be efficiently simulated by maximal models, but not by shallow models. 
$\SharpPclass$-hard functions such as the Ising partition function $Z_n$ (and respective  probability distribution $q_n \propto Z_n$) can be efficiently approximated by maximal and strongly visible closed models in their unnormalized (normalized) visible spin marginals. 
Shallow models cannot carry out these tasks. 
}
\label{tab:complexity separation}
\end{table}
\endgroup

\begingroup
\definecolor{VMInk}{HTML}{28343D}
\definecolor{VMMaximalColour}{HTML}{628FB5}
\definecolor{VMShallow}{HTML}{4F9487}
\colorlet{VMMaximalFill}{VMMaximalColour}
\colorlet{VMShallowFill}{VMShallow!20}
\tikzset{vm/.style={font=\small,text=VMInk,line join=round,line cap=round},vm region/.style={rounded corners=5pt,line width=.7pt},vm label/.style={font=\small\bfseries,inner sep=3pt,align=center},vm point/.style={circle,fill=VMInk,inner sep=0pt,minimum size=4pt,scale=.8}}
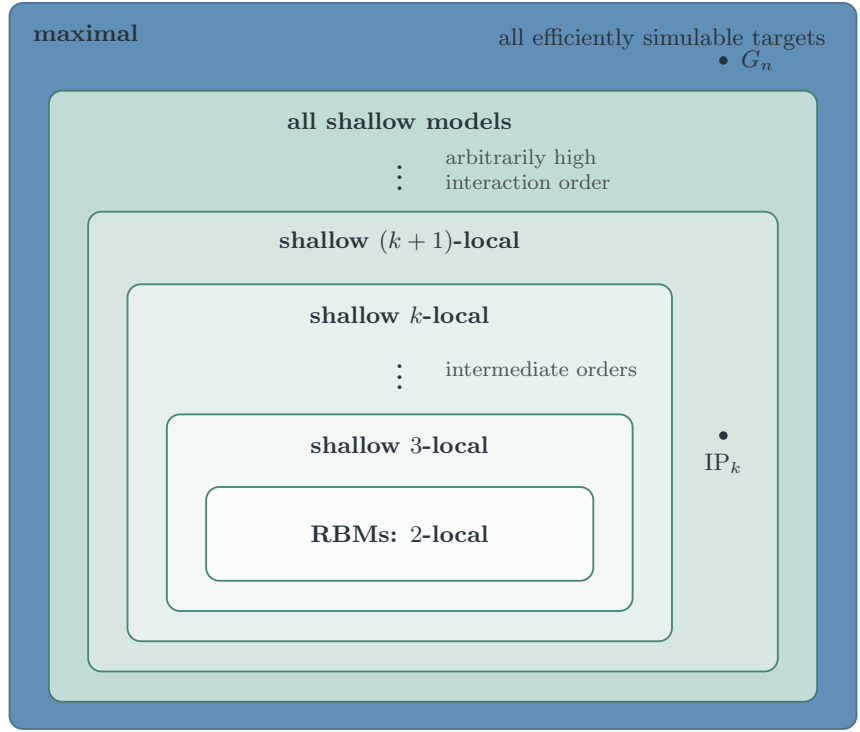
\begin{figure}[!t]
\centering
\begin{tikzpicture}[vm,scale=.8]
\path[vm region,draw=VMMaximalColour!80!black,fill=VMMaximalFill]
 (0,0) rectangle (14,12);
\path[vm region,draw=VMShallow!85!black,fill=VMShallow!34]
 (.65,.45) rectangle (13.35,10.55);
\path[vm region,draw=VMShallow!90!black,fill=VMShallow!23]
 (1.3,.95) rectangle (12.7,8.55);
\path[vm region,draw=VMShallow!90!black,fill=VMShallow!14]
 (1.95,1.45) rectangle (10.95,7.35);
\path[vm region,draw=VMShallow!90!black,fill=VMShallow!7]
 (2.6,1.95) rectangle (10.3,5.2);
\path[vm region,draw=VMShallow!90!black,fill=VMShallow!2]
 (3.25,2.45) rectangle (9.65,4.0);

\node[vm label,anchor=north west] at (.25,11.78) {maximal};
\node[font=\small,anchor=north east] at (13.65,11.73)
 {all efficiently simulable targets};
\node[vm point] (g) at (11.8,11.05) {};
\node[anchor=west] at ($(g)+(.14,0)$) {$G_n$};

\node[vm label] at (6.45,10.05) {all shallow models};
\node[font=\Large] at (6.45,9.25) {$\vdots$};
\node[font=\footnotesize,anchor=west,text=VMInk!85,
      align=left,text width=3.25cm] at (7.05,9.25)
 {arbitrarily high\\interaction order};

\node[vm label] at (6.45,8.05) {shallow $(k+1)$-local};
\node[vm label] at (6.45,6.85) {shallow $k$-local};

\node[font=\Large] at (6.45,5.95) {$\vdots$};
\node[font=\footnotesize,anchor=west,text=VMInk!85]
 at (7.05,5.95) {intermediate orders};

\node[vm label] at (6.45,4.7) {shallow $3$-local};
\node[vm label] at (6.45,3.225) {RBMs: $2$-local};

\node[vm point] (ip) at (11.8,4.85) {};
\node[anchor=north] at ($(ip)+(0,-.14)$) {$\mathrm{IP}_k$};
\end{tikzpicture}
\caption{\textbf{Separation hierarchy for shallow models.}
Set diagram showing the efficient reach of various machine learning model classes. Each rounded rectangle depicts the maximal set of target functions that the respective model class can efficiently simulate. All shown inclusions are proven to be strict.
By definition maximal models have maximal efficient reach, i.e.\ can efficiently simulate all targets that can be efficiently simulated, including $\NPclass$-hard functions such as the Ising ground-state function $G_n$.
Shallow models have strictly smaller efficient reach; for example, they cannot efficiently simulate $G_n$.
The efficient reach of shallow models is given by an infinite, strict inclusion hierarchy indexed by their interaction order $k$. Shallow $(k+1)$-local models can efficiently simulate certain functions with large $k$-party communication complexity such as the $k$-party inner product modulo $2$, $\mathrm{IP}_k$, whereas shallow $k$-local models cannot.
}
\label{fig:efficient-reach}
\end{figure}
\endgroup

\paragraph{Comparison with previous work.}
The characterization of maximal models is based on \cite{Re25} and extends those results to the setting of efficient representations. 

The separation between shallow and maximal models extends the work of \cite{Ga17} to include new representation tasks and reach beyond example models.
Namely, whereas \cite{Ga17} derives a separation between RBMs and DBMs for the task of efficiently representing quantum states in their unnormalized visible spin marginals,  we derive separations for general deep versus shallow models and for the additional tasks of simulating functions and approximating functions in the normalized visible spin marginals. 
The latter reveals that the potentially $\SharpPclass$-hard normalization factor of shallow models does not prevent the separation by computational complexity bounds. 

Finally, the separation hierarchy for shallow models extends results
from \cite{Ma13}, which studies the efficient approximation capabilities of RBMs through a relation to linear threshold functions, in particular proving the non-existence of efficient RBM approximations of the inner product modulo $2$.
We derive a similar relation between arbitrary shallow models and PTFs, proving similar non-existence results for shallow models and the $k$-party inner product modulo two functions. We further relate these results to the communication complexity of the inner product functions. We also prove that increasing the interaction order frees the model from these non-existence results.  

\paragraph{Structure of the paper.}
In \cref{sec:universality} we provide basic concepts and definitions regarding spin models and simulations. We then devote one section to each of our three main results: \cref{sec:eff univ} for the characterization of maximal models, \cref{sec:consequences} for the separations between shallow and maximal models, and \cref{sec:RBM separation} for the separation hierarchy of shallow models. Each of these sections contains an informal statement of the main results, before deriving these results in more detail.
In \cref{sec:conclusion and outlook} we summarize our main findings and provide an outlook. 
Finally, an appendix provides background on complexity theory (\cref{sec:comp compl}), additional details to our main results, such as proofs (\cref{sec:proofs}) as well as further  supplemental results (\cref{sec:supplement}), and numerical learning experiments on the separation hierarchy of shallow models (\cref{sec:ip-learning-methods}).

\section{Universal spin models and their role in machine learning}\label{sec:universality}
Here we introduce basic definitions and concepts regarding spin models, related machine learning models and the respective notions of simulation and approximation. Details can be found in \cite{Re24, Re25}.

In \cref{ssec:spin sys basics} we first define \emph{spin systems}, finite systems of interacting classical degrees of freedom, the spins, and \emph{spin models}, potentially infinite families of spin systems often containing systems with common local interaction pattern but different system sizes, e.g.\ all Ising-type interactions on 2d-lattices of arbitrary system size.
We then define simulations between spin systems. Intuitively, $S$ simulates $T$ if $S$ encodes $T$ in a subset of its spins, in the low-energy sector.
Finally, we define universal spin models,
that is, spin models which can simulate arbitrary target systems, 
and provide an overview of existing characterizations of universality.

In \cref{ssec:spin ml}, we extend these definitions to machine learning models. We define \emph{spin ML models}, that is, spin models with an additional partition of spins into visible and hidden ones, for each of their spin systems. The respective notions of simulation (in their low-energy sector) and approximation (in the marginals of their Boltzmann distribution) then must all be compatible with this visible-hidden split.
\subsection{Spin systems, spin models and simulations}\label{ssec:spin sys basics}
\emph{Spin systems} are systems consisting of finitely many, interacting, classical degrees of freedom, the spins. The interactions between spins are prescribed in terms of local energy functions whose sum defines the global energy function or Hamiltonian.
\begin{definition}[spin system]\label{def:spin system}
A \emph{spin system} $S$ is a hypergraph $(V_S,E_S)$ together with a real-valued, local energy function $J_S(e)$ on each hyperedge $e\in E_S$. Evaluating $J_S(e)$ on a configuration $\mathbf{v}\colon V_S \to \mathbf{2}\coloneqq  \{-1,+1\}$ yields the local energy contribution of spins from $e$. The sum of local energy contributions defines the Hamiltonian of $S$, 
\begin{equation}
    H_S(\mathbf{v}) = \sum_{e\in E_S} J_S(e)(\mathbf{v}\vert_e).
\end{equation}
\end{definition}
We define the \emph{size} of a spin system $S$ as the number of real parameters required to fully specify all of its local energy functions, that is, 
\begin{equation}
    \vert S \vert = \sum_{e\in E_S} 2^{\vert e \vert}. 
\end{equation}
Further, we call these parameters the couplings or weights of $S$. Similarly, $\vert f \vert = 2^n$ for a generic function $f\colon \mathbf{2}^n \to \mathbb{R}$, considering $f$ a trivial spin system. We also define $\vert f \vert_{\min}$ to be the minimal size of any spin system with Hamiltonian $f$. 

We call a spin system $S$ that has local energy functions involving at most $k$ spins \emph{$k$-local} (see \cref{ssec:bool basics}) and refer to $k$ as the interaction order or locality.
Note that our choice of representing local energy functions w.r.t.\ the indicator basis (\cref{ssec:bool basics}), i.e.\ defining them by listing their values implies local energy functions of $k$ spins might have a Fourier degree of $<k$, i.e.\ $k$-local spin systems might have a Hamiltonian with Fourier degree $<k$. The interaction order of a spin system hence only provides an upper bound for the Fourier degree of its Hamiltonian.

The size of any $k$-local spin system is 
essentially determined by the number of spins $n$, i.e.\ $\vert S \vert = \mathrm{poly}(n)$, assuming that $k$ does not depend on the system size.
Here $f= \mathrm{poly}(n)$ is shorthand for $f= \mathcal{O}(p)$ for some polynomial $p$ in $n$, i.e.\ $f$ is asymptotically upper bounded by some polynomial $p$.
However, the size of a generic spin system is exponential in the number of spins, i.e.\ $2^n$.

A \emph{spin model} is a set of spin systems. Typically, spin models are infinite and defined in terms of properties shared by the respective interaction hypergraphs and local energy functions.
\begin{example}[Ising system and model]\label{ex:2d Ising}
    Consider a 2d square lattice and for any of its edges $e=\{a,b\}$ let the corresponding local energy function $J_S(e)$ be defined by 
    \begin{equation}
        J_S(e)(v_a,v_b) = \lambda_ev_av_b,
    \end{equation}
    where $v_a=\mathbf{v}(a) \in \mathbf{2}$ denotes the state of configuration $\mathbf{v}$ on spin $a$ and $\lambda_e$ is a coupling parameter that depends on the edge $e$.
    This defines a 2d Ising system (\cref{fig:spin system a}). If $\lambda_e$ is constant (i.e.\ does not depend on $e$) the 2d Ising system is called homogeneous, else inhomogeneous. 
    The (homogeneous) 2d Ising model is the spin model that consists of all such (homogeneous) 2d Ising systems, i.e.\ defined on arbitrarily sized square lattices and for all possible choices of $\lambda_e$ (\cref{fig:spin system b}).
\end{example}

\begin{figure}[th]
    \begin{subfigure}[t]{0.49\textwidth}
    \caption{}\label{fig:spin system a}
    \centering
    \vspace{0.45cm}
    \begin{tikzpicture}[vis/.style = {draw=black!80, fill=blue!60, thick, circle, minimum size= 0.5em, inner sep=2pt},
    scale = 0.75]
    \drawcustomlattice{Lone}{0}{0}{6}{6}    
    \end{tikzpicture}
    \end{subfigure}
    \begin{subfigure}[t]{0.49\textwidth}
    \caption{}\label{fig:spin system b}
    \centering
    \begin{tikzpicture}[vis/.style = {draw=black!80, fill=blue!60, thick, circle, minimum size= 0.5em, inner sep=2pt},
    setBoundary/.style = {
    draw=gray!80, 
    line width=1.5pt, 
    ellipse, 
    fill=gray!5, inner sep=4pt},
    scale = 0.75]
    \drawcustomlattice{Lone}{0.5}{0.5}{2}{2}    
    \drawcustomlattice{Ltwo}{2.5}{0.5}{4}{4}    
    \drawcustomlattice{Lthree}{0}{2}{3}{3}  
    \begin{scope}[on background layer]
        \node[setBoundary, fit=(Lone) (Ltwo) (Lthree)] (theSet) {};
    \end{scope}
    \end{tikzpicture}
    \end{subfigure}
    \caption{\textbf{Ising system and model.}
    A 2d (square lattice) Ising system \ref{fig:spin system a}, spins are blue vertices, their interactions are attached to edges of the graph and are not explicitly shown, and the 2d Ising model \ref{fig:spin system b}, that is, the set of all such 2d Ising systems, for arbitrary system size and local couplings. Both are defined in \cref{ex:2d Ising}.}
    \label{fig:spin system}
\end{figure}
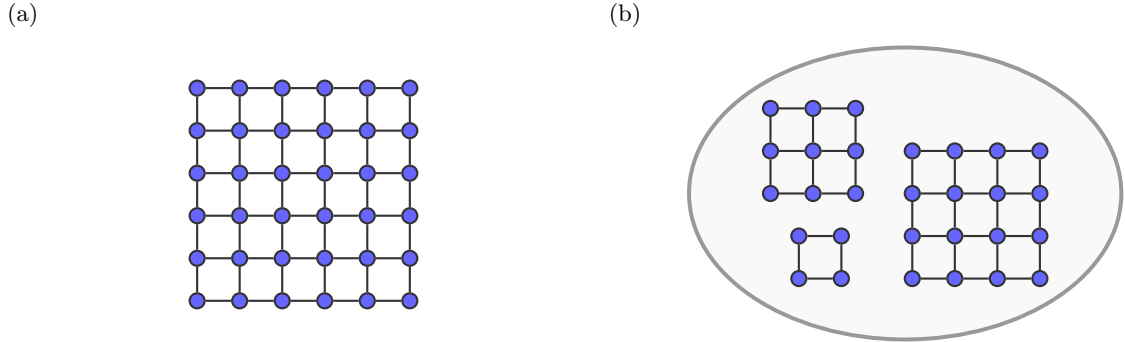

\begin{definition}[simulation]\label{def:simulation}
A spin system $S$ is said to \emph{simulate} a spin system $T$ if $V_T\subseteq V_S$ and there exists a global energy shift $\Gamma$ such that for each configuration $\mathbf{v}\in \mathbf{2}^{V_T}$ 
\begin{equation}\label{eq:effective energy}
\min_{\mathbf{h}}H_S(\mathbf{v},\mathbf{h})=H_T(\mathbf{v})+\Gamma.
\end{equation}
We then 
\begin{enumerate}
    \item call $S$ the \emph{source} and $T$ the \emph{target} of the simulation, and call the respective spins and configurations source and target spins and configurations,
    \item call $V_T$ the \emph{physical spins} and $V_S \setminus V_T$ the \emph{auxiliary spins} of the simulation and the respective configurations $\mathbf{v}$ and $\mathbf{h}$ physical and auxiliary configurations, 
    \item write $S \simul T$ as shorthand for $S$ simulates $T$,
    \item for a given physical configuration $\mathbf{v}$, denote the set of ``correct'' auxiliary configurations that minimize $H_S(\mathbf{v},-)$ as $\simset(\mathbf{v})$, i.e.\ $\simset(\mathbf{v})=\argmin(H_S(\mathbf{v},-))$ and for any $\mathbf{h}_{\min}\in \simset(\mathbf{v})$, it holds that $H_S(\mathbf{v}, \mathbf{h}_{\min})=H_T(\mathbf{v})+\Gamma$, 
    \item call the function $\mathrm{d}(\mathbf{v}) = \vert \simset(\mathbf{v}) \vert$ the \emph{degeneracy} of the simulation and say $S \simul T$ has constant degeneracy if $d$ does not depend on $\mathbf{v}$, 
    and 
    \item call the minimal energy gap 
    \begin{equation}
        \Delta = \min_{\mathbf{v}}\min_{\mathbf{h}\notin \simset(\mathbf{v})}\bigl( H_S(\mathbf{v},\mathbf{h})- H_T(\mathbf{v})-\Gamma \bigr)
    \end{equation}
    between ``correct'' and ``incorrect'' auxiliary configurations (see \cref{fig:sim error}) the \emph{cut-off} of the simulation.
\end{enumerate}
\end{definition}
We call the function on the LHS of \cref{eq:effective energy} the \emph{effective Hamiltonian} of $S$ over physical spins $V_T$. 
This means that whenever $S \simul T$ then the effective Hamiltonian of $S$ over physical spins $V_T$ agrees with the target Hamiltonian, up to an energy shift.
Similarly, we define the effective Hamiltonian of $S$ over any subset $V'\subseteq V_S$, $H_S\vert_{V'}$ as the minimum of $H_S$ over all spins not contained in $V'$. 
Note that trivially $S \simul H_S\vert_{V'}$.   
Note that the cut-off of $S \simul T$ serves as an error parameter where large cut-off means small error. In the following, when saying that $S \simul T$ with cut-off $\Delta$ we mean that the cut-off of this simulation is at least $\Delta$ or that there exists a simulation of target $T$ with source $S$ and cut-off at least $\Delta$. We also introduce the short-hand notation $\Delta(S,T)$ for the cut-off of simulation $S\simul T$.

Finally, note that pseudo-Boolean functions $f\colon \mathbf{2}^n \to \mathbb{R}$ are trivial examples of spin systems (namely such with a single local energy function equal to $f$). Hence, simulations in particular apply to pseudo-Boolean functions.

\begin{remark}[comparison of definition of simulation to \cite{Re24,Re25}]
    This definition of simulation slightly differs from the respective definitions in \cite{Re24} and \cite{Re25}. First, in the definition of the cut-off: both in \cite{Re24} and \cite{Re25} the cut-off $\Delta'$ of a simulation is defined as global energy threshold such that source configurations with ``incorrect'' auxiliary spins have energy above $\Delta'$ and the matching-energy condition (\cref{eq:effective energy}) is only required for target configurations with energy below $\Delta'$. 
Note that this definition of cut-off is related to \cref{def:simulation} by the minimum of $H_S$ and 
whenever simulations of a target $T$ exist for arbitrary cut-offs the two definitions are equivalent. 

Second, in the identification of target configurations with source configurations:
\cite{Re24} requires constant degeneracy and \cite{Re25} even requires degeneracy $d=1$. \cite{Re24} also allows for a more general way of identifying physical spins as subset of $V_S$, which 
is necessary to extend several results to systems with $q$-level spins, for varying $q$.
\end{remark}

Simulations can be composed, scaled and added, allowing the construction of complicated simulations from simpler ones.
This property is termed \emph{modularity} of simulations in \cite{Re24,Re25}.
Specifically, the composition of simulations means given simulation $S\simul T$ with cut-off $\Delta$ and $T \simul R$ with cut-off $\Delta'$ then, one can construct a simulation $S \simul R$ with cut-off $\min(\Delta,\Delta')$. Note that if both initial simulations have constant degeneracy then so does their composition.

Both the addition and the scaling of simulations are based on the respective spin system operations.
Scaling a spin system $S$ with $\lambda>0$\footnote{Note that we require $\lambda>0$ as scaling a spin system by a negative number would swap high-energy and low-energy configurations and hence fundamentally change the spin system. Given \cref{eq:Boltzmann distribution}, scaling by $\lambda>0$ can also be interpreted as a temperature change (i.e.\ change of $\beta$).} means scaling all its local energy functions pointwise and hence resulting in a system $\lambda S$ with $H_{\lambda S}=\lambda H_S$. It is easy to see that then, 
whenever $S \simul T$ with cut-off $\Delta$ also  $\lambda S \simul \lambda T$ with cut-off $\lambda\Delta$. 

The addition of spin systems $S+T$ is based on the hypergraph union of their interaction hypergraphs. 
Note that this allows for a potential overlap between $S$ and $T$, that is, spins that are contained in both systems. 
Specifically, $S+T$ has spins $V_S\cup V_T$, hyperedges $E_S \cup E_T$ and local energy functions 
\begin{equation}
    J_{S+T}(e) = \begin{cases}
        J_S(e) + J_T(e) \ & \text{if} \ e \in E_T \cap E_S \\ 
        J_T(e) \ &\text{if} \ e \in E_T\setminus E_S \\
        J_S(e) \ &\text{if} \ e \in E_S\setminus E_T. 
    \end{cases}
\end{equation}
This implies that 
\begin{equation}
    H_{S+T}(\mathbf{v}) = H_S(\mathbf{v}\vert_{V_S}) + H_T(\mathbf{v}\vert_{V_T}).
\end{equation}
Now,  let $S_{1,2} \simul T_{1,2}$ with cut-off $\Delta_{1,2}$ such that $V_{S_1} \cap V_{S_2} = V_{T_1} \cap V_{T_2}$. This means that the auxiliary spins of each individual simulation $S_1 \simul T_1$ ($S_2 \simul T_2$) do not intersect with the source spins of the respective other simulation $V_{S_2}$ ($V_{S_1}$).
Then $S_1+S_2 \simul T_1+T_2$ with cut-off $\min(\Delta_1,\Delta_2)$. If both initial simulations have constant degeneracy then so does their addition. 
Details on the composition, scaling and addition of simulations can be found in \cite{Re24}.

Given a simulation $T \simul R$, it is not necessary to simulate all of $T$, i.e.\ have $S \simul T$ in order to construct a simulation $S \simul R$. 
Instead, it suffices if $S$ simulates the regime of $T$ in which $H_T = H_R$. By \cref{eq:effective energy} this is exactly the effective Hamiltonian of $T$ over physical spins $V_R$. Thus, given $T \simul R$ and $S \simul H_T\vert_{V_R}$ then $S \simul R$ and the cut-off of this simulation equals that of $S \to  H_T\vert_{V_R}$.

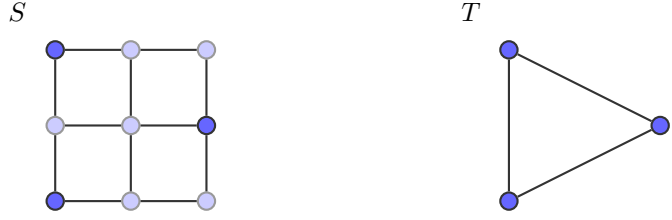
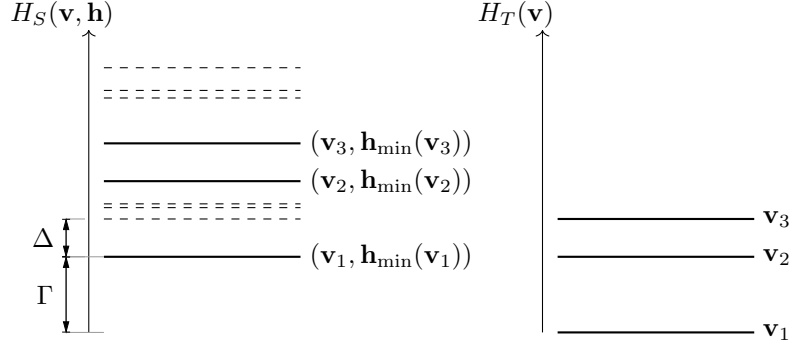
\begin{figure}[th]
\begin{subfigure}[t]{0.95\textwidth}
    \caption{}\label{fig:simulation a}
    \centering
    \begin{tikzpicture}[vis/.style = {draw=black!80, fill=blue!60, thick, circle, minimum size= 0.65em, inner sep=2pt},
    hid/.style = {draw=black!40, fill=blue!20, thick, circle, minimum size= 0.65em, inner sep=2pt},
    scale = 1]
    \node[vis] (l1) at (0,0) {};
    \node[hid] (l2) at (0,1) {};
    \node[vis] (l3) at (0,2) {};
     \node[hid] (l4) at (1,0) {};
    \node[hid] (l5) at (1,1) {};
    \node[hid] (l6) at (1,2) {};
     \node[hid] (l7) at (2,0) {};
    \node[vis] (l8) at (2,1) {};
    \node[hid] (l9) at (2,2) {};

    \draw[-, thick, black!80] (l1) to (l2);
    \draw[-, thick, black!80] (l2) to (l3);
    \draw[-, thick, black!80] (l4) to (l5);
    \draw[-, thick, black!80] (l5) to (l6);
    \draw[-, thick, black!80] (l7) to (l8);
    \draw[-, thick, black!80] (l8) to (l9);
    \draw[-, thick, black!80] (l1) to (l4);
    \draw[-, thick, black!80] (l4) to (l7);
    \draw[-, thick, black!80] (l2) to (l5);
    \draw[-, thick, black!80] (l5) to (l8);
    \draw[-, thick, black!80] (l3) to (l6);
    \draw[-, thick, black!80] (l6) to (l9);

    \node[] (label1) at (-0.5,2.5) {$S$};

    \node[vis] (r1) at (6,0) {};
    \node[vis] (r2) at (8,1) {};
    \node[vis] (r3) at (6,2) {};
    \draw[-, thick, black!80] (r1) to (r2);
    \draw[-, thick, black!80] (r2) to (r3);
    \draw[-, thick, black!80] (r1) to (r3);

    \node[] (label2) at (5.5,2.5) {$T$};

    \end{tikzpicture}
\end{subfigure}
\begin{subfigure}[t]{0.95\textwidth}
    \caption{}\label{fig:simulation b}
    \centering

        \begin{tikzpicture}
        \draw[->] (0,0) to (0,4);
        \node[] (e) at (-0.35,4.2) {$H_S(\mathbf{v},\mathbf{h})$};

        \draw[-, thick] (0.2,1) to (2.8,1);
        \node[] (c1) at (4,1) {$(\mathbf{v}_1, \mathbf{h}_{\min}(\mathbf{v}_1))$};
        
        \draw[-, thick] (0.2,2) to (2.8,2);
        \node[] (c12) at (4,2) {$(\mathbf{v}_2, \mathbf{h}_{\min}(\mathbf{v}_2))$};
        
        \draw[-, thick] (0.2,2.5) to (2.8,2.5);
        \node[] (c13) at (4,2.5) {$(\mathbf{v}_3, \mathbf{h}_{\min}(\mathbf{v}_3))$};

        \draw[dashed] (0.2, 1.5) to (2.8, 1.5);
        \draw[dashed] (0.2, 1.65) to (2.8, 1.65);
        \draw[dashed] (0.2, 1.7) to (2.8, 1.7);

        \draw[dashed] (0.2, 3.1) to (2.8, 3.1);
        \draw[dashed] (0.2, 3.2) to (2.8, 3.2);
        \draw[dashed] (0.2, 3.5) to (2.8, 3.5);

        \dimline[line style = {line width=0.5},label style={above=0.2ex, rotate=0},extension start length=0.5,
        extension end length=0.5]{(-0.3,1)}{(-0.3,1.5)}{    \rotatebox[origin=c]{-90}{$\Delta$}};

        \dimline[line style = {line width=0.5},label style={above=0.2ex, rotate=0},extension start length=0.5,
        extension end length=0.5]{(-0.3,0)}{(-0.3,1)}{    \rotatebox[origin=c]{-90}{$\Gamma$}};
        
        \draw[->] (6,0) to (6,4);
        \node[] (e2) at (5.65,4.2) {$H_T(\mathbf{v})$};

        \draw[-, thick] (6.2,0) to (8.8,0);
        \node (c2) at (9.1,0) {$\mathbf{v}_1$};
        
        \draw[-, thick] (6.2,1) to (8.8,1);
        \node (c3) at (9.1,1) {$\mathbf{v}_2$};
        
        \draw[-, thick] (6.2,1.5) to (8.8,1.5);
        \node (c4) at (9.1,1.5) {$\mathbf{v}_3$};

        \end{tikzpicture}
    \end{subfigure}
    \caption{\textbf{Example of a simulation.}
    A source system $S$ (\ref{fig:simulation a}, LHS) physical spins in opaque, auxiliary spins in transparent) simulates a target system $T$ (\ref{fig:simulation a}, RHS)  if $\min_{\mathbf{h}}H_S(\mathbf{v},\mathbf{h})=H_T(\mathbf{v})+\Gamma$. This condition relates the two respective energy functions, as shown in \ref{fig:simulation b}. In this example, the simulation has constant degeneracy $d=1$, so each target configuration (\ref{fig:simulation b}, RHS) $\mathbf{v}_i$ has a unique source configuration (\ref{fig:simulation b}, LHS) $(\mathbf{v}_i,\mathbf{h}_{\min}(\mathbf{v}_i))$ with up to a shift of $\Gamma$, equal energy. The cut-off $\Delta$ of the simulation is the
    minimal energy gap between ``correct'' auxiliary configuration ($\mathbf{h}_{\min}(\mathbf{v})$, thick lines), and ``incorrect'' auxiliary configurations (dashed lines). }\label{fig:sim error}
\end{figure}  

Simulations are promoted to spin models as follows: 
\begin{enumerate}
    \item A spin model $\mathcal{M}$ \emph{simulates} a target system $T$ with cut-off $\Delta$, if there exists a system $S\in \mathcal{M}$ such that $S \simul T$ with cut-off $\Delta$. We then write $\mathcal{M} \simul T$.
    \item A spin model $\mathcal{M}$ \emph{simulates} a target model $\mathcal{N}$ with cut-off $\Delta$ if for every $T\in \mathcal{N}$, $\mathcal{M} \simul T$ with cut-off $\Delta$. We then write $\mathcal{M} \simul \mathcal{N}$.
\end{enumerate}

We call the maximal set of targets (i.e.\ the maximal spin model) that a given spin model $\mathcal{M}$ can simulate its \emph{reach} (under simulations). 
We call $\mathcal{M}$ universal, if its constant degeneracy reach is maximal, for all cut-offs, i.e.\ if it can simulate all spin systems (and hence also all pseudo-Boolean functions) for all cut-offs and with constant degeneracy.
\begin{definition}[universality]\label{def:universality}
   A spin model is called \emph{universal} if it can simulate all spin systems for all cut-offs and for constant degeneracy, such that any target $T$ can be simulated by an at most polynomially larger source $S$, i.e.\ $\vert S \vert = \mathrm{poly}(\vert T \vert)$, and the respective polynomial does not depend on the cut-off.
\end{definition}

Interestingly, seemingly weak properties of spin systems suffice to characterize universality. 
In \cite{Re24} (building on \cite{Ge16}) universal spin models are fully characterized by \emph{functional completeness} and \emph{closure}. 
A spin model is functionally complete if it can simulate arbitrary pair interactions ---note that there are only four linearly independent pair interactions. A spin model is closed if up to simulation (subject to the same size constraint as \cref{def:universality}, $\vert S \vert = \mathrm{poly}(\vert T \vert)$) it contains all pairwise sums and scalings of its spin systems. Intuitively, functional completeness ensures that the spin model can encode two-bit Boolean logic in its ground states.
Closure, together with the modular structure of simulations, allows the model to then also simulate linear combinations and certain compositions of these Boolean functions, which ultimately suffice to simulate arbitrary targets.

In \cite{Re25}, universal spin models are instead fully characterized by \emph{flag completeness} and closure. A spin model is flag complete if it contains sufficiently many flag systems, spin systems that can flag an arbitrary configuration $\mathbf{v}$ in their ground state, i.e.\ in the state of a distinguished flag spin (or the configurations of a distinguished set of flag spins) signal whether or not a fixed subset of spins is in configuration $\mathbf{v}$. Scaling such flag systems appropriately then allows one to adjust the energy of the flagged configuration at will. Linear combinations of flag systems hence can simulate arbitrary energy functions and closure ensures that these linear combinations can be simulated by the universal model\footnote{Note that  adding an appropriately scaled flag system for configuration $\mathbf{v}$ to any system $S$ effectively increases the Boltzmann probability of $\mathbf{v}$ by a factor that depends on the scaling, while uniformly decreasing the Boltzmann probabilities of all other configurations. The concept of flag systems is hence closely related to UATs of certain machine learning models \cite{Lr08,Ro10}}.

Both these characterizations imply that already fairly simple spin models, such as the 2d Ising model with fields, the 3d Ising model or the spin models that underlie restricted Boltzmann machines and deep, narrow Boltzmann machines are universal \cite{Re24, Re25}.

\subsection{Universal spin models are universal approximators}\label{ssec:spin ml}
Energy-based machine learning models such as restricted Boltzmann machines (RBMs) and deep Boltzmann machines (DBMs) can be defined in terms of spin models.
The spins (often called units) then are additionally divided into \emph{visible spins}, with configuration $\mathbf{v}=(v_1, \ldots, v_n)$ and \emph{hidden spins}, with configuration $\mathbf{h}=(h_1, \ldots, h_m)$. 
RBMs (\cref{fig:spin RBM a}) are spin systems with complete bipartite interaction graph (between visible and hidden spins) and local energy functions that lead to the following Hamiltonian:
\begin{equation}\label{eq:RBM}
H(\mathbf{v},\mathbf{h})=\sum_{i=1}^m h_i\mathbf{W}_i\mathbf{v} + \sum_{i=1}^m b^{(h)}_ih_i + \sum_{j=1}^n b^{(v)}_jv_j.
\end{equation}
The weights of an RBM are usually divided into the vector $\mathbf{W}_i=(W_{i,1}, \ldots,W_{i,n})$ called the weight vector of hidden spin $i$ and the real numbers $b^{(v)}_i$ and $b^{(h)}_i$ called the bias of visible spin $i$ and hidden spin $i$, respectively.

DBMs (\cref{fig:spin RBM b}) are defined in terms of interaction graphs that consist of one visible and 
$k>1$ hidden layers, such that any two neighbouring layers are complete bipartite. Denoting hidden configurations as $\mathbf{h}= (\mathbf{h}^{(1)}, \ldots, \mathbf{h}^{(k)})$ and $\mathbf{h}^{(j)}=(h^{(j)}_1, \ldots, h^{(j)}_{m_j})$ the  DBM Hamiltonian reads
\begin{equation}\label{eq:DBM}
    H(\mathbf{h}^{(0)}, \ldots, \mathbf{h}^{(k)})= \sum_{j=1}^k  \sum_{i=1}^{m_j} h^{(j)}_i\mathbf{W}^{(j)}_i \mathbf{h}^{(j-1)}  + 
    \sum_{j=0}^k
    \sum_{i=1}^{m_j} b_i^{(j)}h^{(j)}_i,
\end{equation}
where $\mathbf{W}^{(j)}_i$ is the weight vector and $b^{(j)}_i$ is the bias of spin $i$ of layer $j$ and we use
the identification $\mathbf{v}=\mathbf{h}^{(0)}$ and $m_0=n$. Often we assume that all hidden layers contain an equal number of $m$ hidden spins. We call the maximal number of hidden spins per layer the width of the RBM/DBM and the number of total layers, visible or hidden, the depth of the RBM/DBM. 

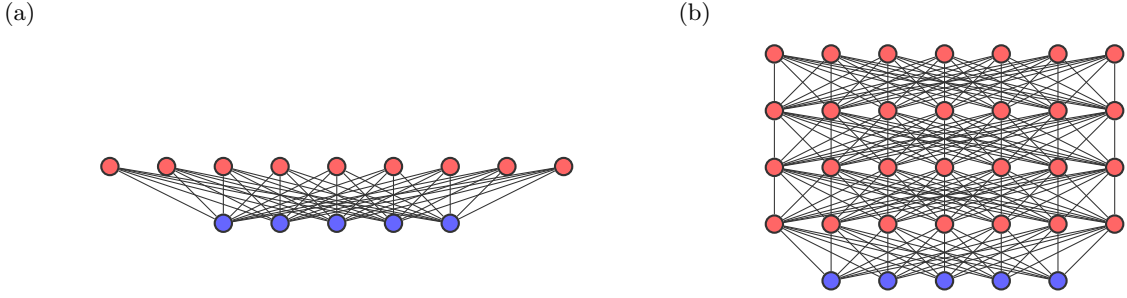
\begin{figure}
    \begin{subfigure}[t]{0.55\textwidth}
    \caption{}\label{fig:spin RBM a}
    \centering
    \begin{tikzpicture}[vis/.style = {draw=black!80, fill=blue!60, thick, circle, minimum size= 0.65em, inner sep=2pt},
    hid/.style = {draw=black!80, fill=red!60, thick, circle, minimum size= 0.65em, inner sep=2pt},
    scale = 0.75]
    \drawRBM{Rone}{0}{0}{-2}{5}{9}
    \node[] (1) at (0,3) {};
    \end{tikzpicture}
    \end{subfigure}
    \begin{subfigure}[t]{0.44\textwidth}
    \caption{}\label{fig:spin RBM b}
    \centering
    \begin{tikzpicture}[vis/.style = {draw=black!80, fill=blue!60, thick, circle, minimum size= 0.65em, inner sep=2pt},
    hid/.style = {draw=black!80, fill=red!60, thick, circle, minimum size= 0.65em, inner sep=2pt},
    scale = 0.75]
    \drawDBM{Rone}{0}{0}{-1}{5}{7}{4}
    \end{tikzpicture}
    \end{subfigure}
    \caption{RBM \ref{fig:spin RBM a} and DBM \ref{fig:spin RBM b}, visible spins (or units) in blue, hidden spins (or units) in red. 
    Both have pair-interactions, attached to edges between spins and fields, attached to the individual spins. The respective couplings (or weights) of both are not explicitly shown. }\label{fig:RBM}
\end{figure}

Given some machine learning model $\mathcal{M}$ the learning objective then is of the following form: 
Given some target probability distribution $q(\mathbf{v})$ compute weights such that the system $S \in \mathcal{M}$ defined by these weights best approximates\footnote{Here approximation is understood w.r.t.\ some distance measure between probability distributions, e.g.\ the total variation distance or the Kullback--Leibler divergence.} $q(\mathbf{v})$ over its visible spins, that is, 
\begin{equation}\label{eq:prob sim}
    \sum_{\mathbf{h}} p_{S}(\mathbf{v},\mathbf{h}) \approx q(\mathbf{v}),
\end{equation}
where 
\begin{equation}\label{eq:Boltzmann distribution}
    p_{S}(\mathbf{v},\mathbf{h})\coloneqq \frac{1}{Z_{S}} e^{-H(\mathbf{v},\mathbf{h})}
\end{equation}
is the \emph{Boltzmann distribution} of $S \in \mathcal{M}$ (at $\beta=1$) and the normalization $Z_{S}$ its \emph{partition function}. 
In the following we call the maximal set of distributions $p(\mathbf{v})$ that a given machine learning model can approximate to arbitrary precision its \emph{approximation reach} and call models that can approximate arbitrary probability distributions \emph{universal approximators}.
Thus, we can consider the reach of $\mathcal{M}$ under simulations, as well as under approximations (in terms of its visible-spin marginals).\footnote{In \cref{sec:consequences} we also consider approximations in terms of the unnormalized visible-spin  marginals.}

In \cite{Re25} we prove that simulations with constant degeneracy preserve Boltzmann distributions, that is, whenever $S \simul T$ with cut-off $\Delta$ and constant degeneracy then $\sum_{\mathbf{h}}p_S(\mathbf{v},\mathbf{h})\approx p_T(\mathbf{v})$ up to errors of order $\mathcal{O}(e^{-\Delta})$. This implies that
\cref{eq:prob sim} is satisfied whenever $S$ simulates $-\log(q(\mathbf{v}))$, with constant degeneracy, subject to the additional constraint that the physical spins of the simulation are precisely the visible spins of $S$.
This leads to the following observations.
\begin{enumerate}
    \item The approximation reach of an energy-based machine learning model $\mathcal{M}$ can be studied through its reach under constant degeneracy simulations (\cref{def:simulation}). Whenever $\mathcal{M}$ is universal (according to \cref{def:universality}) it is a universal approximator.
    \item To prove universal approximation theorems for $\mathcal{M}$ it suffices to prove its universality which is attainable by means of the characterizations of universality \cite{Re24,Re25}.
\end{enumerate}

This motivates the following definitions. 
\begin{definition}[spin ML model]\label{def:mlm}
\leavevmode 
    \begin{enumerate}
        \item\label{it:spin ML model 1} A \emph{spin machine learning model} (spin ML model) $\mathcal{M}$ is a spin model together with a partition into visible and hidden spins, $V_S = \mathrm{Vis}_S \cup \mathrm{Hid}_S$ for each of its spin systems $S$. 
        \item\label{it:spin ML model 2} A simulation of some target system or model $\mathcal{M} \simul \ldots$ is called \emph{visible} if additionally to \cref{def:simulation} physical spins and visible spins agree, we then write $\mathcal{M} \vsim \ldots$
        \item\label{it:spin ML model 3} A spin ML model is called \emph{v-universal}\footnote{Whenever visibility is clear from the context, we might also simply say ``universal''.} if it is universal (\cref{def:universality}) using only visible simulations.
    \end{enumerate}
\end{definition}

In \cite{Re25} we explicitly apply this new perspective to RBMs and DBMs and, by means of the characterization of universality, prove that both are v-universal and hence universal approximators. Note that proving in this sense universality is sufficient for universal approximation, in \cref{rem:sim approx} we study to what extent it can also be considered as necessary. 

\section{Efficiency of universal models}\label{sec:eff univ}
Here we introduce the central notion of efficient simulations and characterize those models with maximal reach under efficient simulations.

In \cref{ssec:eff sim def}, we first define \emph{efficient simulations} as simulations with polynomial overhead, that is, simulating a target system of $n$ spins requires a source system that can be specified by most $\mathrm{poly}(n)$ parameters. We then define \emph{maximal models} as models that can efficiently simulate all targets that can be efficiently simulated, i.e.\ that have maximal reach under efficient simulations.

In \cref{ssec:char max}, we  characterize maximality both for spin models and spin ML models. For spin ML models this leads to the definition of \emph{visible closure}, the property that up to simulation, visible spins can be made hidden. 
Visible closure is sufficient to ensure that the visible-hidden split does not obstruct the composition of spin ML model simulations. Hence, for visible closed models, simulations can be constructed in a modular, compositional way.
Since visible closure might introduce additional hidden-hidden interactions it favors deep over shallow models. 

The main results of this section are the following:
\begin{informalresult}[characterization of maximality]
\leavevmode
    \begin{enumerate}[labelwidth=2cm,
    labelsep=0.6em,
    leftmargin=!,
    align=left]
        \item[\cref{thm:univ eff spin}] A spin model is maximal if and only if it is universal.
        \item[\cref{thm:main}] For spin ML models, the following holds:\\
            maximal and visible closed $\Leftrightarrow$ maximal and universal $\Leftrightarrow$ universal and visible closed.
    \end{enumerate}
\end{informalresult}

\subsection{Efficient simulations and maximality}\label{ssec:eff sim def}
Specifying a generic target system $T$ of $n$ spins, in general, requires $\vert T \vert = \mathcal{O}(2^n)$ real parameters, so simulating it requires a source of size $\vert S \vert \geq \vert T \vert = \mathcal{O}(2^n)$. Thus v-universal spin ML models, require an exponential overhead, the number of weights (of the source) grows exponentially with the number of target spins.
This agrees with all established universal approximation theorems (\cite{Lr08,Ro10,Mo15}, see also \cite{Re25} for an overview),
yet, in practice, limits the applicability of v-universal spin ML models or universal approximation theorems.

In contrast to this, we call a simulation efficient if its overhead is at most polynomial. 
\begin{definition}[efficient simulation]\label{def:efficient simulation}
\leavevmode
\begin{enumerate}
    \item The \emph{overhead} of a simulation $\mathcal{M} \simul \mathcal{N}$ is the function that maps a given number of target spins $n$ to the maximal source size $\vert S \vert$ that is required to simulate any target system $T$ of $n$ spins.
    \item A simulation is called \emph{efficient} if its overhead is asymptotically upper bounded by some polynomial in $n$, i.e.\ is (at most) $\mathrm{poly}(n)$. We then write $\mathcal{M} \polysim \mathcal{N}$.
\end{enumerate}
\end{definition}
Efficiency of a simulation characterizes its asymptotic behaviour when changing the number of target spins. As such it is only meaningful when the target is an infinite spin model in contrast to a single system. In the following, whenever we call a simulation of a single target system $S \simul T$ efficient we implicitly consider it in the context of a simulation $\mathcal{M} \simul \mathcal{N}$. If we state that $\mathcal{M} \polysim \mathcal{N}$ for multiple cut-offs (e.g.\ for arbitrary cut-off) then we further require that the specific polynomial overhead does not depend on the cut-off. 

Note that in contrast to the above notion of efficiency, universality (see \cref{def:universality} and \cite{Re24}) requires that $\vert S \vert = \mathrm{poly}(\vert T \vert)$. This implies efficiency of the simulation only if $\vert T \vert = \mathrm{poly}(n)$.
Further note that the composition and scaling of efficient simulations is efficient, so efficient simulations admit the same modular structure as regular simulations do. 
Finally, whenever a simulation is both visible and efficient, we write $S \vpolysim T$.

The existence of spin models that cannot be efficiently simulated (e.g.\ any infinite set of generic target systems) implies that there exists no spin model that is universal when restricted to efficient simulations. 
However, there might still exist spin models that under efficient simulations can simulate at least as many targets as any other spin model. We call such spin models \emph{p-maximal} (see \cref{fig:efficient universal}).
\begin{definition}[p-maximality]\label{def:efficient universality}
    A spin model $\mathcal{M}$ is called \emph{p-maximal}\footnote{Whenever efficiency is clear from the context we might also simply say ``maximal''.} if   whenever there exists some efficient simulation $\mathcal{M}' \polysim \mathcal{N}$ with cut-off $\Delta$ then $\mathcal{M}\polysim \mathcal{N}$ with cut-off $\Delta$ and whenever $\mathcal{M}' \polysim \mathcal{N}$ has constant degeneracy then also $\mathcal{M}\polysim \mathcal{N}$ with constant degeneracy.
\end{definition}

\begin{figure}
\definecolor{VMUniversal}{HTML}{A9CCE3}
\definecolor{VMMaximal}{HTML}{628FB5}
\definecolor{VMTarget}{HTML}{D5DDE3}
\definecolor{VMOther}{HTML}{737D85}
\definecolor{VMOutline}{HTML}{455563}

\begin{subfigure}[t]{0.49\textwidth}
\caption{}\label{fig:univ a}
\centering
\begin{tikzpicture}[
sys/.style={
    draw=VMOutline, fill=VMOther, thick, diamond,
    minimum size=0.35em, inner sep=1.5pt
},
sys_univ/.style={
    draw=VMOutline, fill=VMUniversal, thick, diamond,
    minimum size=0.35em, inner sep=1.5pt
},
setBoundary/.style={
    draw=gray!60,
    line width=1.5pt,
    ellipse,
    fill=gray!5, inner sep=4pt
},
scale=0.5]
\node[sys] (1) at (0,0) {};
\node[sys] (2) at (1,1) {};
\node[sys] (3) at (1.5,0) {};
\node[sys] (4) at (0,2) {};
\node[sys] (5) at (3,1) {};
\node[sys] (6) at (2,2) {};
\node[sys_univ] (7) at (2,4) {};

\foreach \x in {1,...,6}{
    \draw[->,draw=VMOutline] (7) to (\x);
}

\begin{scope}[on background layer]
    \node[setBoundary, fit=(1) (2) (3) (4) (5) (6) (7)] (s1) {};
\end{scope}
\end{tikzpicture}
\end{subfigure}
\begin{subfigure}[t]{0.49\textwidth}
\caption{}\label{fig:univ b}
\centering
\begin{tikzpicture}[
sys/.style={
    draw=VMOutline, fill=VMOther, thick, diamond,
    minimum size=0.35em, inner sep=1.5pt
},
sys_univ/.style={
    draw=VMOutline, fill=VMMaximal, thick, diamond,
    minimum size=0.35em, inner sep=1.5pt
},
sys_target/.style={
    draw=VMOutline, fill=VMTarget, thick, diamond,
    minimum size=0.35em, inner sep=1.5pt
},
setBoundary/.style={
    draw=gray!60,
    line width=1.5pt,
    ellipse,
    fill=gray!5, inner sep=4pt
},
scale=0.5]

\node[sys_target] (1) at (0,0) {};
\node[sys_target] (2) at (1,1.5) {};
\node[sys_target] (3) at (1.5,0) {};
\node[sys_target] (4) at (0,2) {};
\node[sys_target] (5) at (3,1) {};
\node[sys_target] (6) at (2,2) {};
\node[sys_univ] (7) at (2,4) {};

\foreach \x in {1,...,6}{
    \draw[->,draw=VMOutline] (7) to (\x);
}

\node[sys] (11) at (-1,3.5) {};
\foreach \x in {1,4}{
    \draw[->,draw=VMOutline] (11) to (\x);
}
\node[sys] (12) at (4,3) {};
\foreach \x in {3,5,6}{
    \draw[->,draw=VMOutline] (12) to (\x);
}

\node[sys] (13) at (4,0) {};
\node[sys] (14) at (3,-0.5) {};
\node[sys] (15) at (-1,0.5) {};
\node[sys] (16) at (1,4) {};

\draw[->,draw=VMOutline] (15) to (2);
\draw[->,draw=VMOutline] (13) to (5);

\begin{scope}[on background layer]
    \node[setBoundary, fit=(1) (2) (3) (4) (5) (6) (7) (11) (12)] (s1) {};
\end{scope}

\end{tikzpicture}
\end{subfigure}
\caption{The elliptic shapes depict the set of all spin models, with squares representing individual models.
\cref{fig:univ a} A universal spin model (light blue square) can simulate (indicated by arrows) all spin models (dark grey squares). \cref{fig:univ b} In contrast to that a p-maximal model (dark blue square) can efficiently simulate (indicated by arrows) all models that can be efficiently simulated, i.e.\ are reached by an arrow (light blue-grey squares). There exist models that cannot be efficiently simulated, i.e.\ are not reached by an arrow (dark grey squares).}
\label{fig:efficient universal}
\end{figure}
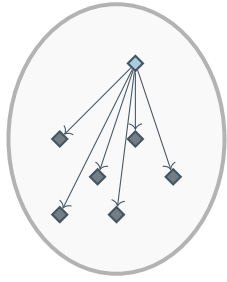
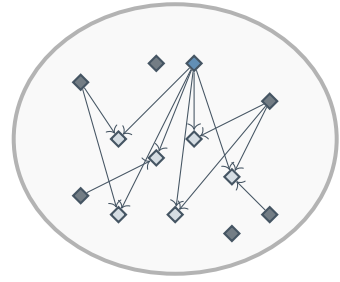

In words, a p-maximal spin model can simulate all those targets that can be efficiently simulated and whenever a target can be efficiently simulated with constant degeneracy then it can also the p-maximal model can simulate it with constant degeneracy. 
When considering efficient simulations, having access to a single p-maximal spin model hence is as good as having access to all spin systems. 
Thus, considering efficient simulations, p-maximal spin models play a similar role as universal spin models do for generic simulations.

Similar to efficient simulation we define the concept of \emph{efficient approximation}. A spin model $\mathcal{M}$ efficiently approximates a target function $f$ of $n$ spins through a quantity $A_T$ defined in terms of spin systems $T\in \mathcal{M}$ if there exists such a $T$, such that
\begin{enumerate}
    \item $A_T$ approximates $f$, i.e.\ for all configurations $\mathbf{v}$ $A_T(\mathbf{v})\approx f(\mathbf{v})$, and 
    \item the overhead of $T$ is at most polynomial, i.e.\ $\vert T \vert= \mathrm{poly}(n)$ and, whenever the approximation works for multiple approximation errors, the respective polynomial does not depend on the error.
\end{enumerate}
The most important examples of this are efficient approximations in terms of (unnormalized) marginals of the Boltzmann distributions of $\mathcal{M}$.

Importantly, similar to universal spin models being universal approximators of probability distributions, \cite{Re25}, the efficient approximation reach of p-maximal spin models is maximal.
\begin{proposition}[simulations preserve efficient approximations]\label{thm:p-max approx}
    Whenever some target probability distribution $q$ over $n$ spins can be efficiently approximated by the marginal distribution\footnote{Whenever clear from the context we denote marginal distributions by the respective arguments, i.e.\ for a probability distribution $p(x,y)$, $p(x) = \sum_y p(x,y)$.} of some spin system $T$, that is, 
    \begin{equation}
        q(\mathbf{v}) \approx \sum_{\mathbf{h}} p_T(\mathbf{v},\mathbf{h})
    \end{equation}
    and $\vert T \vert = \mathrm{poly}(n)$, then $q$ can be efficiently approximated by the marginals of any p-maximal spin model. 
\end{proposition}
\begin{proof}
    Let $\mathcal{M}$ be p-maximal. Since $\vert T \vert = \mathrm{poly}(n)$, trivially $T\polysim T$ so $\mathcal{M} \polysim T$ with constant degeneracy. The claim then follows since constant degeneracy simulations preserve Boltzmann distributions \cite{Re25}.  
\end{proof}
Thus, also when considering efficient approximations of target distributions, having access to a single p-maximal spin model is as good as having access to all spin systems. 

Note, however that this does not directly provide insights into which targets can be efficiently simulated (or approximated).
While the reach of a universal model under generic simulations by definition is the set of all spin systems, the reach of a p-maximal model under efficient simulations is left implicit in \cref{def:efficient universality}. In fact characterizing the latter, i.e.\ the set of all spin systems that can be efficiently simulated is not straightforward. 
Whenever $\vert T \vert = \mathrm{poly}(\vert V_T \vert)$ then trivially $T \polysim T$, so it contains at least all such systems. However, we shall see later that it also contains functions that cannot be expressed as Hamiltonian of polynomially sized spin systems (\cref{lem:eff sim GSE}).

\subsection{Characterizations of p-maximality and (v,p)-maximality}\label{ssec:char max}
For spin models p-maximality and universality are equivalent. This is a consequence of the compositional structure of simulations and the universality characterization of \cite{Re24}. Note that it in particular implies that the characterization of universality also applies to p-maximality. 
\begin{theorem}[characterization of p-maximality]\label{thm:univ eff spin}
    For spin models 
    \begin{equation}
        \text{universality} \Leftrightarrow \text{p-maximality}.
    \end{equation}
\end{theorem}
\begin{proof}
We first prove the ``$\Rightarrow$''-implication then the ``$\Leftarrow$''-implication.

\smallskip
\noindent\textbf{``$\Rightarrow$''}
Let $\mathcal{M}$ be universal and assume that some target $R$ of $n$ spins can be efficiently simulated by some system $T$ with cut-off $\Delta_1$. 

Since $\mathcal{M}$ is universal, for any cut-off $\Delta_2$ there exists $S\in \mathcal{M}$ that simulates $T$ with constant degeneracy.
Since simulations can be composed, then $S \simul R$, with cut-off $\Delta=\min(\Delta_1,\Delta_2)$, so choosing $\Delta_2$ large enough implies $\Delta=\Delta_1$.
Note that if  $T \simul R$ has constant degeneracy, then so does $S \simul R$.

Efficiency of $T \polysim R$ means that $\vert T \vert = \mathrm{poly}(n)$ and universality ensures that 
$\vert S \vert = \mathrm{poly}(\vert T \vert)$ \cite{Re24}. Hence, also the simulation $S\polysim R$ is efficient. 

\smallskip
\noindent\textbf{``$\Leftarrow$''}
Let $\mathcal{M}$ be p-maximal.
Let $R$ be any target and take any universal spin model $\mathcal{N}$ that for all its systems $T$ satisfies $\vert T \vert = \mathrm{poly}(\vert V_{T} \vert)$. Existence of such a model is proven in \cite{Re24}, an explicit example is the 2d Ising with fields. Since $\mathcal{N}$ is universal, for any cut-off, there exists some $T\in \mathcal{N}$ with $T \simul R$ with constant degeneracy.

Now since $\vert T \vert = \mathrm{poly}(\vert V_{T} \vert)$, trivially $T \polysim T$, for any cut-off and with constant degeneracy. Thus $T$ can be efficiently simulated for any cut-off and with constant degeneracy and since $\mathcal{M}$ is p-maximal, for any cut-off, there exists some $S\in \mathcal{M}$ such that $S \polysim T$ with constant degeneracy. Since simulations can be composed, then, for any cut-off $S \simul R$ with constant degeneracy.
The final simulation, $S \simul R$ satisfies $\vert S \vert = \mathrm{poly}(\vert R \vert)$ since $S \polysim T$  and since universality of $\mathcal{N}$ implies that $\vert T \vert = \mathrm{poly}(\vert R \vert)$.
\end{proof}

Following \cref{def:mlm}, an appropriate notion of p-maximality for spin ML models is obtained by requiring that all relevant simulations (of \cref{def:efficient universality}) are visible, i.e.\ it is p-maximal with visible simulations only. We call such spin ML models \emph{(v,p)-maximal}.
\begin{definition}[(v,p)-maximality]\label{def:vp-max}
      A spin ML model $\mathcal{M}$ is called \emph{(v,p)-maximal}\footnote{Similar to p-maximality, whenever both visibility and efficiency are clear from the context we might simply say ``maximal''.} if
      it is p-maximal using only visible simulations.
\end{definition}
 
Characterizing (v,p)-maximality is less straightforward than characterizing p-maximality.
The additional partition into visible and hidden spins obstructs the composition of simulations, specifically, 
the composition of a visible simulation $S \vsim T$ with any other simulation $T\simul R$ is not necessarily visible, even when $T\vsim R$. 
The reason for this is that when considering the composition $S \simul R$,
physical spins agree with $V_R$ but since $S \vsim T$ visible spins of $S$ agree with $V_T$ and whenever $T \simul R$ uses auxiliary spins, $V_R$ is a proper subset of $V_T$ (see \cref{fig:sim comp vis}).
This, in particular prevents us from replicating the proof of \cref{thm:univ eff spin} for spin ML models.
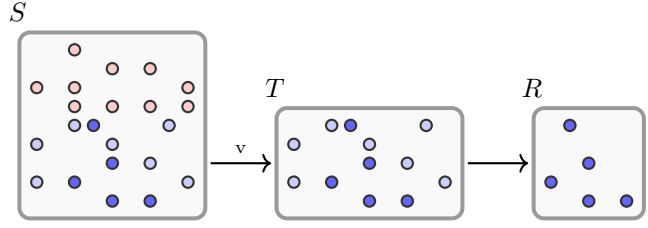
\begin{figure}
    \centering
    \begin{tikzpicture}[vis_phys/.style = {draw=black!80, fill=blue!60, thick, circle, minimum size= 0.35em, inner sep=1.5pt},
    vis_aux/.style = {draw=black!80, fill=blue!20, thick, circle, minimum size= 0.35em, inner sep=1.5pt},
    hid_phys/.style = {draw=black!80, fill=red!60, thick, circle, minimum size= 0.35em, inner sep=1.5pt},
    hid_aux/.style = {draw=black!80, fill=red!20, thick, circle, minimum size= 0.35em, inner sep=1.5pt},
    setBoundary/.style = {
    draw=gray!80, 
    line width=1.5pt, 
    rectangle, rounded corners, 
    fill=gray!5, inner sep=4pt},
    scale = 0.5]

    \begin{scope}[shift={(6.8,0)}]
    \node[vis_phys] (1p1) at (1,0) {};
    \node[vis_phys] (1p2) at (2,0) {};
    \node[vis_phys] (1p3) at (0,0.5) {};
    \node[vis_phys] (1p4) at (0.5,2) {};
    \node[vis_phys] (1p5) at (1,1) {};
    \begin{scope}[on background layer]
        \node[setBoundary, fit=(1p1) (1p2) (1p3) (1p4) (1p5)] (sys1) {};
    \end{scope}
    \node[] (r) at (-0.5,3) {$R$};
    \end{scope}

    \begin{scope}[shift={(0,0)}]
    \node[vis_phys] (2p1) at (2,0) {};
    \node[vis_phys] (2p2) at (3,0) {};
    \node[vis_phys] (2p3) at (1,0.5) {};
    \node[vis_phys] (2p4) at (1.5,2) {};
    \node[vis_phys] (2p5) at (2,1) {};
    \node[vis_aux] (2a1) at (1,2) {};
    \node[vis_aux] (2a2) at (2,1.5) {};
    \node[vis_aux] (2a3) at (3,1) {};
    \node[vis_aux] (2a4) at (4,0.5) {};
    \node[vis_aux] (2a5) at (3.5,2) {};
    \node[vis_aux] (2a6) at (0,0.5) {};
    \node[vis_aux] (2a7) at (0,1.5) {};
    \begin{scope}[on background layer]
        \node[setBoundary, fit=(2p1) (2p2) (2p3) (2p4) (2p5) (2a1) (2a2) (2a3) (2a4) (2a5) (2a6) (2a7)] (sys2) {};
    \end{scope}
    \node[] (t) at (-0.5,3) {$T$};
    \end{scope}

    \begin{scope}[shift={(-6.8,0)}]
    \node[vis_phys] (3p1) at (2,0) {};
    \node[vis_phys] (3p2) at (3,0) {};
    \node[vis_phys] (3p3) at (1,0.5) {};
    \node[vis_phys] (3p4) at (1.5,2) {};
    \node[vis_phys] (3p5) at (2,1) {};
    \node[vis_aux] (3a1) at (1,2) {};
    \node[vis_aux] (3a2) at (2,1.5) {};
    \node[vis_aux] (3a3) at (3,1) {};
    \node[vis_aux] (3a4) at (4,0.5) {};
    \node[vis_aux] (3a5) at (3.5,2) {};
    \node[vis_aux] (3a6) at (0,0.5) {};
    \node[vis_aux] (3a7) at (0,1.5) {};

    \node[hid_aux] (3h1) at (0,3) {};
    \node[hid_aux] (3h2) at (1,2.5) {};
    \node[hid_aux] (3h3) at (1,4) {};
    \node[hid_aux] (3h4) at (2,3.5) {};
    \node[hid_aux] (3h5) at (3,2.5) {};
    \node[hid_aux] (3h6) at (3,3.5) {};
    \node[hid_aux] (3h7) at (4,2.5) {};
    \node[hid_aux] (3h8) at (4,3) {};
    \node[hid_aux] (3h9) at (1,3) {};
    \node[hid_aux] (3h10) at (2,2.5) {};

    \begin{scope}[on background layer]
        \node[setBoundary, fit=(3p1) (3p2) (3p3) (3p4) (3p5) (3a1) (3a2) (3a3) (3a4) (3a5) (3a6) (3a7) (3h1) (3h2) (3h3) (3h4) (3h5) (3h6) (3h7) (3h8) (3h9) (3h10)] (sys3) {};
    \end{scope}
    \node[] (S) at (-0.5,5) {$S$};
    \end{scope}

    \draw[->, thick] (-2.2,1) to node [midway, above] {\scriptsize $\mathrm{v}$} (-0.6,1);
    \draw[->, thick] (4.6,1) to (6.2,1);
    \end{tikzpicture}
    \caption{\textbf{Compositions of visible simulations are not visible.} A spin ML system $S$ and two spin systems $T$ and $R$, 
    all depicted as abstract sets (gray boxes) of spins (blue and red dots, blue spins are visible red spins are hidden). The systems provide a chain of simulations $T \simul R$ and $S \vsim T$ so their composition yields $S \simul R$. Physical spins of this composition simulation are opaque, auxiliary spins are transparent.
    The composition $S \simul R$ is not visible as visible spins of $S$ agree with spins of $T$ while physical spins of $S \simul R$ (blue, opaque) agree with spins of $R$.}
    \label{fig:sim comp vis}
\end{figure}

Constructing a visible simulation of $R$ is nevertheless possible, if there exists some system $S'$ that suffices to replace $S$ at least in the context of the simulation $S \simul R$ but has visible spins $V_R\subseteq V_T$, i.e.\ compared to $S$, restricted to a subset.
Therefore, whenever a spin ML model, in the context of simulations is closed under restricting its visible spins to any subset we can replicate the proof of \cref{thm:univ eff spin}. We call such spin ML models \emph{visible closed} (\cref{fig:vis subset closed}).
\begin{definition}[visible closure] \label{def:visible closed}
Let $\mathcal{M}$ be a spin ML model, $S\in \mathcal{M}$ and $V'\subseteq \mathrm{Vis}_S$. Further, let $S'$ be the system that equals $S$ but has visible spins $V'$. Then, $\mathcal{M}$ is  
\begin{enumerate}
    \item \emph{strictly visible closed} if for any such  $(S,V')$, $S'\in \mathcal{M}$,
    \item\label{it:vis closed 2} \emph{strongly visible closed} if for any such $(S,V')$, and any $\Delta$ there exists some $S'' \in \mathcal{M}$ with visible spins $V'$, with $\vert S'' \vert = \mathrm{poly}(\vert S \vert)$ such that $S'' \simul S$ with constant degeneracy and with cut-off at least $\Delta$, and 
    \item\label{it:vis closed 3} \emph{visible closed} if for any such $(S,V')$, there exists an $S'' \in \mathcal{M}$ with $\vert S'' \vert = \mathrm{poly}(\vert S \vert + \vert H_S\vert_{V'} \vert_{\min})$ such that $S'' \vsim H_S\vert_{V'}$ with cut-off at least $\Delta(S, H_S\vert_{V'})$. Whenever the trivial simulation $S \simul H_S\vert_{V'}$ has constant degeneracy, then so does $S'' \vsim H_S\vert_{V'}$.
\end{enumerate}
Whenever $V'$ is such that $\vert S \vert = \mathrm{poly}(\vert V' \vert)$ we additionally require that the simulations in \ref{it:vis closed 3} are efficient, for \ref{it:vis closed 2} efficiency then follows from $\vert S'' \vert = \mathrm{poly}(\vert S \vert)$. 
\end{definition}
Note that visible closure requires that $S''$ is at most polynomially larger than $\vert S \vert + \vert H_S\vert_{V'} \vert_{\min}$, where $\vert H_S\vert_{V'} \vert_{\min}$ is the minimal size of a spin system with Hamiltonian $H_S\vert_{V'}$.
This ensures that the additional overhead introduced by visible closure is compatible with the polynomial size constraint of universality (\cref{def:universality}).

Note that strictly visible closed implies strongly visible closed since then $S'\in \mathcal{M}$ trivially satisfies $S' \simul S$ and strongly visible closed implies visible closed since trivially $S \simul H_S\vert_{V'}$ so when the cut-off of $S'' \simul S$ is chosen high enough, $S'' \vsim H_S\vert_{V'}$ with cut-off $\Delta(S, H_S\vert_{V'})$.
\begin{figure}
    \centering
     \begin{tikzpicture}[vis_phys/.style = {draw=black!80, fill=blue!60, thick, circle, minimum size= 0.35em, inner sep=1.5pt},
    vis_aux/.style = {draw=black!80, fill=blue!20, thick, circle, minimum size= 0.35em, inner sep=1.5pt},
    hid_phys/.style = {draw=black!80, fill=red!60, thick, circle, minimum size= 0.35em, inner sep=1.5pt},
    hid_aux/.style = {draw=black!80, fill=red!20, thick, circle, minimum size= 0.35em, inner sep=1.5pt},
    setBoundary/.style = {
    draw=gray!80, 
    line width=1.5pt, 
    rectangle, rounded corners, 
    fill=gray!5, inner sep=4pt},
    scale = 0.5]

    \begin{scope}[shift={(8,0)}]
    \node[vis_phys] (5p1) at (2,0) {};
    \node[vis_phys] (5p2) at (3,0) {};
    \node[vis_phys] (5p3) at (1,0.5) {};
    \node[vis_phys] (5p4) at (1.5,2) {};
    \node[vis_phys] (5p5) at (2,1) {};
    \node[vis_phys] (5a1) at (1,2) {};
    \node[vis_phys] (5a2) at (2,1.5) {};
    \node[vis_phys] (5a3) at (3,1) {};
    \node[vis_phys] (5a4) at (4,0.5) {};
    \node[vis_phys] (5a5) at (3.5,2) {};
    \node[vis_phys] (5a6) at (0,0.5) {};
    \node[vis_phys] (5a7) at (0,1.5) {};
    \node[hid_phys] (5h1) at (0,3) {};
    \node[hid_phys] (5h2) at (1,2.5) {};
    \node[hid_phys] (5h3) at (1,4) {};
    \node[hid_phys] (5h4) at (2,3.5) {};
    \node[hid_phys] (5h5) at (3,2.5) {};
    \node[hid_phys] (5h6) at (3,3.5) {};
    \node[hid_phys] (5h7) at (4,2.5) {};
    \node[hid_phys] (5h8) at (4,3) {};
    \node[hid_phys] (5h9) at (1,3) {};
    \node[hid_phys] (5h10) at (2,2.5) {};

    \begin{scope}[on background layer]
        \node[setBoundary, fit=(5p1) (5p2) (5p3) (5p4) (5p5) (5a1) (5a2) (5a3) (5a4) (5a5) (5a6) (5a7) (5h1) (5h2) (5h3) (5h4) (5h5) (5h6) (5h7) (5h8) (5h9) (5h10)] (sys5) {};
    \end{scope}
    \node[] (s) at (-0.5,5) {$S$};
    \end{scope}

    \begin{scope}[shift={(0,0)}]
    \node[vis_phys] (3p1) at (2,0) {};
    \node[vis_phys] (3p2) at (3,0) {};
    \node[vis_phys] (3p3) at (1,0.5) {};
    \node[vis_phys] (3p4) at (1.5,2) {};
    \node[vis_phys] (3p5) at (2,1) {};
    \node[hid_phys] (3a1) at (1,2) {};
    \node[hid_phys] (3a2) at (2,1.5) {};
    \node[hid_phys] (3a3) at (3,1) {};
    \node[hid_phys] (3a4) at (4,0.5) {};
    \node[hid_phys] (3a5) at (3.5,2) {};
    \node[hid_phys] (3a6) at (0,0.5) {};
    \node[hid_phys] (3a7) at (0,1.5) {};
    \node[hid_phys] (3h1) at (0,3) {};
    \node[hid_phys] (3h2) at (1,2.5) {};
    \node[hid_phys] (3h3) at (1,4) {};
    \node[hid_phys] (3h4) at (2,3.5) {};
    \node[hid_phys] (3h5) at (3,2.5) {};
    \node[hid_phys] (3h6) at (3,3.5) {};
    \node[hid_phys] (3h7) at (4,2.5) {};
    \node[hid_phys] (3h8) at (4,3) {};
    \node[hid_phys] (3h9) at (1,3) {};
    \node[hid_phys] (3h10) at (2,2.5) {};

    \begin{scope}[on background layer]
        \node[setBoundary, fit=(3p1) (3p2) (3p3) (3p4) (3p5) (3a1) (3a2) (3a3) (3a4) (3a5) (3a6) (3a7) (3h1) (3h2) (3h3) (3h4) (3h5) (3h6) (3h7) (3h8) (3h9) (3h10)] (sys3) {};
    \end{scope}
    \node[] (s) at (-0.5,5) {$S'$};
    \end{scope}

    \begin{scope}[shift={(-8,0)}]
    \node[vis_phys] (4p1) at (2,0) {};
    \node[vis_phys] (4p2) at (3,0) {};
    \node[vis_phys] (4p3) at (1,0.5) {};
    \node[vis_phys] (4p4) at (1.5,2) {};
    \node[vis_phys] (4p5) at (2,1) {};
    \node[hid_phys] (4a1) at (1,2) {};
    \node[hid_phys] (4a2) at (2,1.5) {};
    \node[hid_phys] (4a3) at (3,1) {};
    \node[hid_phys] (4a4) at (4,0.5) {};
    \node[hid_phys] (4a5) at (3.5,2) {};
    \node[hid_phys] (4a6) at (0,0.5) {};
    \node[hid_phys] (4a7) at (0,1.5) {};
    \node[hid_phys] (4h1) at (0,3) {};
    \node[hid_phys] (4h2) at (1,2.5) {};
    \node[hid_phys] (4h3) at (1,4) {};
    \node[hid_phys] (4h4) at (2,3.5) {};
    \node[hid_phys] (4h5) at (3,2.5) {};
    \node[hid_phys] (4h6) at (3,3.5) {};
    \node[hid_phys] (4h7) at (4,2.5) {};
    \node[hid_phys] (4h8) at (4,3) {};
    \node[hid_phys] (4h9) at (1,3) {};
    \node[hid_phys] (4h10) at (2,2.5) {};

    \node[hid_aux] (4h11) at (0,3.5) {};
    \node[hid_aux] (4h12) at (1,4.5) {};
    \node[hid_aux] (4h13) at (2,4) {};
    \node[hid_aux] (4h14) at (3,5) {};
    \node[hid_aux] (4h15) at (3,4) {};
    \node[hid_aux] (4h16) at (4,3.5) {};
    \node[hid_aux] (4h17) at (0,4) {};
    \node[hid_aux] (4h18) at (0.5,5.5) {};
    \node[hid_aux] (4h19) at (2.5,5) {};
    \node[hid_aux] (4h20) at (4,5.5) {};

    \begin{scope}[on background layer]
        \node[setBoundary, fit=(4p1) (4p2) (4p3) (4p4) (4p5) (4a1) (4a3) (4a5) (4a7) (4h1) (4h3) (4h5) (4h7) (4h9)  (4h11) (4h12) (4h13) (4h14) (4h15) (4h16) (4h17) (4h18) (4h19) (4h20)] (sys4) {};
    \end{scope}
    \node[] (s') at (-0.5,6.5) {$S''$};
    \end{scope}
    \draw[<-, dashed] (5,2) to (7,2);
     \draw[->, thick] (-3,2) to (-1,2);
    \end{tikzpicture}
    \caption{\textbf{Visible closure and composition of simulations.}
    Spin ML systems $S, S'$ and $S''$ depicted as abstract sets (gray boxes) of spins (visible spins in blue hidden spins in red).
    $S'$ is obtained from $S$ by
    restricting visible spins to a subset $V'$, i.e.\ visible spins outside of $V'$ are made hidden. This yields a system $S'$. If $S \in \mathcal{M}$ and $\mathcal{M}$ is strictly visible closed then also $S'\in \mathcal{M}$. If $\mathcal{M}$ is strongly visible closed then it simulates $S'$, i.e.\ there exists some system $S''\in \mathcal{M}$ with visible spins $V'$ that simulates $S'$ (physical spins opaque, auxiliary spins transparent, note that by definition, all auxiliary spins must be hidden). If $\mathcal{M}$ is visible closed then it visibly simulates the effective Hamiltonian of $S'$ over its visible spins $V'$ (not shown).}
    \label{fig:vis subset closed}
\end{figure}
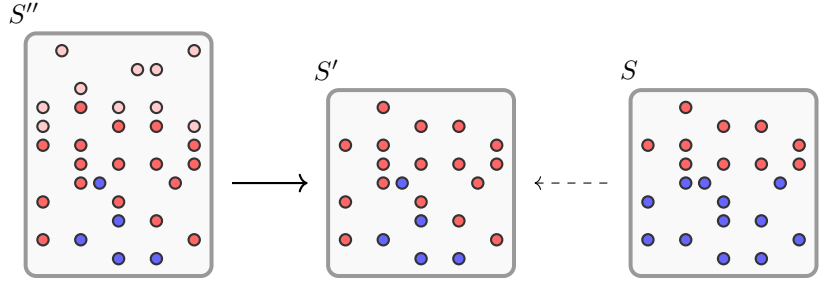

Importantly, visible closure ensures that visible simulations can be composed, similar to generic ones. 
\begin{lemma}[visible closure and composition of simulations]\label{lem:visible closed}
    If $\mathcal{M}$ is visible closed, $\mathcal{M} \vsim T$ and $T \simul R$ with cut-offs $\Delta_1$ and $\Delta_2$ then $\mathcal{M} \vsim R$ with cut-off $\min(\Delta_1,\Delta_2)$. 
    If $\mathcal{M} \vsim T$ and $T \simul R$ are efficient or have constant degeneracy then the same holds true for $\mathcal{M} \vsim R$. Moreover, denote the source of $\mathcal{M} \vsim T$ by $S$ then the constructed source $S''$ of $\mathcal{M} \vsim R$ can be chosen 
    with size $\mathrm{poly}(|S|+|R|)$.
\end{lemma}
\begin{proof}
    Choose $S\in\mathcal M$ with $S\vsim T$ and cut-off
    $\Delta_1$. Composing with $T\simul R$ gives
    $S\simul R$ with cut-off
    $\Delta\coloneqq\min(\Delta_1,\Delta_2)$,
    but this simulation need not be visible.
    
    If $\mathcal{M}$ is visible closed then there exists $S''\in \mathcal{M}$ with $S'' \vsim H_{S}\vert_{V_R}$ with cut-off at least that of the trivial simulation, $\Delta(S, H_{S}\vert_{V_R})$. Note that since $S \simul R$ with cut-off $\Delta$,  $\Delta(S, H_{S}\vert_{V_R}) = \Delta$. Thus, $S''\vsim R$ with cut-off $\Delta$. 
    
    If $S \vsim T$ and $T \simul R$ are efficient then so is $S \simul R$. This means that $\vert S \vert = \mathrm{poly}(\vert V_R \vert)$, which by \cref{def:visible closed} implies that $S'' \vpolysim H_S\vert_{V_R}$ and hence also $S''\vpolysim R$, i.e.\ both these simulations are efficient.

    If the two initial simulations have constant degeneracy then so does their composition $S \simul R$. Constant degeneracy of $S \simul R$ in particular implies that the trivial simulation $S \simul H_S\vert_{V_R}$ has constant degeneracy, which by \cref{def:visible closed} implies that also $S'' \vsim H_S\vert_{V_R}$ has constant degeneracy. Thus, also the final composition $S''\vsim R$ has constant degeneracy. 

    Finally, by \cref{def:visible closed}, the final source satisfies $\vert S'' \vert = \mathrm{poly}(\vert S \vert + \vert H_S\vert_{V_R} \vert_{\min})$. Since $S\simul R$, up to a potential energy shift, $H_S\vert_{V_R} = H_R$. This implies that $\vert H_S\vert_{V_R} \vert_{\min} \leq \vert R \vert$, so in total, $\vert S'' \vert = \mathrm{poly}(|S|+|R|)$.
\end{proof}

\begin{lemma}[characterization of (v,p)-maximality for visible closed models]\label{thm:univ eff spin ML model}
For visible closed spin ML models 
        \begin{equation}
            \text{v-universality} \Leftrightarrow \text{(v,p)-maximality}.
        \end{equation}
\end{lemma}
\begin{proof}
The ``$\Rightarrow$''-direction follows immediately from the ``$\Rightarrow$''-direction of (the proof of) \cref{thm:univ eff spin} and \cref{lem:visible closed}.

We now prove the ``$\Leftarrow$''-direction. 
Let $R$ be any target.
Following the ``$\Leftarrow$''-direction of the proof of \cref{thm:univ eff spin}, existence of the universal model $\mathcal{N}$ (see \cref{thm:univ eff spin}) implies that for any cut-off there exists some spin system $T$ with $\vert T \vert = \mathrm{poly}(\vert V_{T} \vert)$ such that $T \simul R$ with constant degeneracy. Since for any cut-off, trivially $T \polysim T$ with constant degeneracy and $\mathcal{M}$ is (v,p)-maximal, for any cut-off there exists a $S\in \mathcal{M}$ such that $S \vpolysim T$ with constant degeneracy. Thus, by \cref{lem:visible closed}, for any cut-off there exists $S'' \in \mathcal{M}$ such that
$S'' \vsim R$ with constant degeneracy.
By \cref{lem:visible closed}, the final simulation, $S'' \vsim R$ satisfies $\vert S'' \vert = \mathrm{poly}(\vert R \vert)$ since $S \vpolysim T$  and since universality of $\mathcal{N}$ implies that $\vert T \vert = \mathrm{poly}(\vert R \vert)$.
\end{proof}

\begin{lemma}[sufficient conditions for visible closure]\label{lem: univ visible closed}
    For spin ML models 
    \begin{equation}
        \text{v-universal} \ \text{and} \ \text{(v,p)-maximal} \Rightarrow \text{visible closed}.
    \end{equation}
\end{lemma}
\begin{proof}
    Let $\mathcal{M}$ be a v-universal and (v,p)-maximal spin ML model and let $S\in \mathcal{M}$ and $V'$ be a subset of the visible spins of $S$. 
    Since $\mathcal{M}$ is v-universal, for any cut-off $\mathcal{M}\vsim H_S\vert_{V'}$ with constant degeneracy. 
    Moreover, by definition of $\vert - \vert_{\min}$, $H_S\vert_{V'}$ can be realized by a spin system of size at most $\vert H_S\vert_{V'} \vert_{\min}$, so universality of $\mathcal{M}$ yields a source $S''\in \mathcal{M}$ that simulates $H_S\vert_{V'}$ and satisfies $\vert S'' \vert = \mathrm{poly}(\vert H_S\vert_{V'} \vert_{\min})$.
    
    If additionally $\vert S \vert = \mathrm{poly}(\vert V' \vert)$ (cf. \cref{def:visible closed}), then, the trivial simulation $S \polysim H_S\vert_{V'}$ with cut-off $\Delta$ is efficient.  Since $\mathcal{M}$ is (v,p)-maximal there exists $E\in\mathcal M$ that efficiently simulates $H_S\vert_{V'}$ with cut-off at least
    $\Delta$, and with constant degeneracy whenever the
    original simulation has constant degeneracy.
    Thus, both $S''$ (from universality) and $E$ (from (v,p)-maximality) simulate $H_S\vert_{V'}$, $S''$ satisfies the size-bound $\vert S'' \vert = \mathrm{poly}(\vert H_S\vert_{V'} \vert_{\min})$ and $E$ satisfies the efficiency requirement of \cref{def:visible closed}.
    Taking the smaller source between $S''$ and $E$ then satisfies both the size bound and the efficiency requirement of \cref{def:visible closed}.
\end{proof}

Combining \cref{thm:univ eff spin ML model} and \cref{lem: univ visible closed} the relation between visible closure, (v,p)-maximality and v-universality can be summarized as follows.
\begin{theorem}[characterization of (v,p)-maximality, v-universality and visible closure]\label{thm:main}
    For spin ML models any two of the three properties ``visible closed'', ``(v,p)-maximal'' and ``v-universal'' imply the respective other.
\end{theorem}
\begin{proof}
    Follows directly from \cref{thm:univ eff spin ML model} and \cref{lem: univ visible closed}.
\end{proof}

In the context of simulations, visible closure allows one to change visible spins to hidden spins.  
In general doing so introduces new hidden-hidden interactions. 
Thus, intuitively, 
while models with enough hidden-hidden interactions might be trivially, i.e.\ strictly visible closed, models without hidden-hidden connections should not be visible closed.
In the following, we call a spin ML model \emph{deep} if it has hidden-hidden interactions (e.g.\ DBMs), that is, if it has local energy functions that contain more than one hidden unit, else we call it \emph{shallow} (e.g.\ RBMs). 
Then, visible closure and by \cref{thm:univ eff spin ML model} (v,p)-maximality should require deep models.
This intuition is supported by \cref{cor:DBM max}. In \cref{sec:consequences} we study consequences of (v,p)-maximality that prove that shallow models and hence in particular RBMs are not (v,p)-maximal.
\begin{corollary}[(v,p)-maximality of DBMs]\label{cor:DBM max}
    DBMs of depth at least 3 are (v,p)-maximal.
\end{corollary}
\begin{proof}
 We prove that DBMs of depth at least 3 are strictly visible closed.
 In \cite{Re25} it is proven that DBMs are v-universal, (v,p)-maximality then follows from \cref{thm:main}.
  Consider any DBM $S$ and change any number of its visible spins to hidden spins, resulting in a system $S'$ that contains hidden spins in the bottommost layer (\cref{fig:DBM part}, right).
  These hidden spins in the bottom layer  can equivalently be considered as spins in the second hidden layer (\cref{fig:DBM part}, left), meaning that $S'$ is also a DBM.
\end{proof}

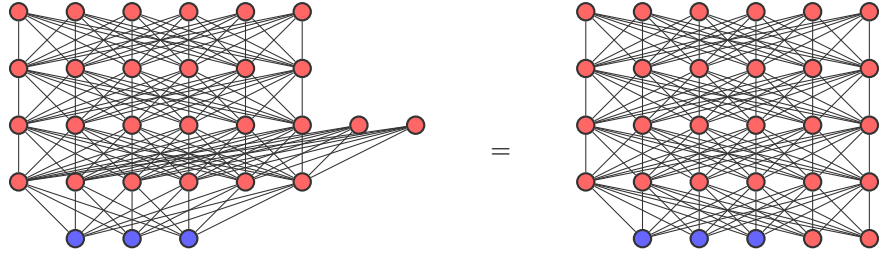
\begin{figure}
    \centering
    \begin{tikzpicture}[vis/.style = {draw=black!80, fill=blue!60, thick, circle, minimum size= 0.65em, inner sep=2pt},
    hid/.style = {draw=black!80, fill=red!60, thick, circle, minimum size= 0.65em, inner sep=2pt},
    scale = 0.75]    
    \drawDBM{Rone}{0}{0}{-1}{3}{6}{4}
    \node[hid] (h1) at (4,0) {};
    \node[hid] (h2) at (5,0) {};
    \foreach \x in {1,...,6} {
    \draw[-, black!80] (h1) to (Ronea\x-h1);
    \draw[-, black!80] (h2) to (Ronea\x-h1);
    }
    \drawDBM{Rtwo}{-10}{0}{-1}{3}{6}{4}
     \node[hid] (h3) at (-4,2) {};
    \node[hid] (h4) at (-3,2) {};
    \foreach \x in {1,...,6} {
    \draw[-, black!80] (h3) to (Rtwoa\x-h1);
    \draw[-, black!80] (h4) to (Rtwoa\x-h1);
    }

    \node[] (eq) at (-1.5,1.5) {$=$}; 
    \end{tikzpicture}
    \caption{\textbf{Visible closure of DBMs.}
    DBMs of depth at least 3 are strictly visible closed; making visible spins hidden introduces hidden spins (red) in the bottom layer (RHS system). However, these can equivalently be considered as belonging to the second hidden layer (LHS system). The two systems are equivalent and only differ w.r.t.\ their drawing.}
    \label{fig:DBM part}
\end{figure}

\section{Complexity separations between maximal and shallow models}\label{sec:consequences}
Here we derive explicit separations between (v,p)-maximal and shallow models for three related representation tasks, simulating functions, approximating functions in their unnormalized visible spin marginals and approximating probability distributions in their normalized visible spin marginals.

We first lay out our assumptions on computational complexity, in \cref{ssec:comp compl assumptions}. We then define a computationally maximally hard reference model $\mathcal{M}_{\mathrm{Ising}}$, in \cref{ssec:hard reference model}. Based on this reference model we construct the three target functions, the $\NPclass$-hard function $\mathcal{G}= \{ G_n \}$ of computing ground-state energies of $\mathcal{M}_{\mathrm{Ising}}$, the $\SharpPclass$-hard function $\mathcal{Z}=\{Z_n\}$ of computing partition functions from $\mathcal{M}_{\mathrm{Ising}}$ and the 
$\SharpPclass$-hard probability distributions $\mathcal{Q}= \{q_n \propto Z_n \}$ corresponding to the partition functions of $\mathcal{M}_{\mathrm{Ising}}$.
Then, we study the tasks of efficiently simulating $\mathcal{G}$ (in \cref{ssec: separation 1}), efficiently approximating $\mathcal{Z}$ in the unnormalized visible spin marginals (in \cref{ssec: separation 2}) and efficiently approximating $\mathcal{Q}$ in the normalized visible spin marginals (in \cref{ssec: separation 3}). For each task we derive a separation between shallow and maximal models, under reasonable assumptions on computational complexity. The respective results are

\begin{informalresult}[separation between shallow and maximal models]
Maximal models can efficiently and shallow models cannot 
    \begin{enumerate}[labelwidth=2cm,
    labelsep=0.6em,
    leftmargin=!,
    align=left]
    \item[\cref{thm:separation 1}] simulate the $\NPclass$-hard ground-state energy function $\mathcal{G}$,
    \item[\cref{thm:separation 2}] approximate the $\SharpPclass$-hard 
    partition function $\mathcal{Z}$ to arbitrary relative precision in their unnormalized visible marginals
    (assuming the maximal model is additionally strongly visible closed), and 
    \item[\cref{thm:separation 3}] approximate the $\SharpPclass$-hard
     probability distributions $\mathcal{Q}$ to arbitrary uniform relative precision in their normalized visible marginals.
    (assuming the maximal model is additionally strongly visible closed).
\end{enumerate}
\end{informalresult}
  
\subsection{Computational complexity assumptions}\label{ssec:comp compl assumptions}
Here we lay out our assumptions on computational complexity. Details on computational complexity can be found in \cref{ssec:comp compl basics}.

Throughout this section and \cref{ssec: complexity hard reference model} we use the
following conventions for the representation of parameters
and the evaluation of numerical quantities.
\begin{enumerate}
    \item All local energy values specifying a spin system
    $S$ are rational numbers represented in binary.
    For the polynomial-size families considered below,
    these values have polynomial bit length, so the full
    description of each system also has polynomial length.

    \item Simulation shifts and cut-offs are likewise
    rational numbers represented in binary. We assume
    that the simulations used below admit parameters
    of polynomial bit length in the target description
    length and the bit length of the requested cut-off.
    For approximations with relative error $\varepsilon$,
    parameter lengths may additionally depend polynomially
    on $\log(1/\varepsilon)$.
    Thus, a size bound independent of
    $\varepsilon$ does not require the parameter
    descriptions to be independent of $\varepsilon$.

    \item Irrational quantities are evaluated numerically.
    Computing such a quantity to additive precision $b$
    means producing a rational approximation with error
    at most $2^{-b}$; whenever relative error is intended,
    this is stated explicitly.
    Polynomial-time evaluation means polynomial runtime
    in the input description length and $b$.
\end{enumerate}

For rational computations, a polynomial number of
arithmetic operations and comparisons suffices provided
that all intermediate values have polynomial bit length.
This condition holds for the sums and minima used to
evaluate shallow effective energies.
For expressions involving exponentials or logarithms,
numerical evaluation requires a separate precision
analysis. In particular, the logarithmic evaluation
of shallow marginals is established in
\cref{cor:shallow log eval}.

The rational arithmetic involved in the compositions,
additions and scalings of simulations used below preserves polynomial
bit length, provided that only polynomially many terms
and polynomial-bit scaling factors are involved.
The parameters of the reference model
$\mathcal M_{\mathrm{Ising}}$ satisfy these conditions.
We additionally assume that applications of maximality
and visible closure admit simulating systems with the
parameter bounds specified above.
This is an assumption on the existence of simulations
with polynomial-length descriptions; it does not require
that these descriptions can be constructed in polynomial
time (cf. \cite{Re24}). Where efficient construction is needed, we state
that additional assumption explicitly.

Finally, the following separation results between shallow and maximal models are derived based on the assumption that there is no collapse of the \emph{polynomial hierarchy} \cite{St76}.  The polynomial hierarchy is a hierarchy of complexity classes that generalizes the classes $\Pclass$, $\NPclass$ and the assumption that there is no collapse of the polynomial hierarchy is a generalization of the widely believed conjecture $\Pclass\neq \NPclass$. Hence, assuming there is no such collapse is considered very reasonable.

\subsection{The hard reference model}\label{ssec:hard reference model}
We consider the homogeneous, antiferromagnetic Ising model without fields as reference model. That is, given any graph $G=(V,E)$, where for simplicity we take $V=[n]\coloneqq \{1,\ldots, n\}$, we define the Ising system $S_G$ with interaction graph $G$ and Hamiltonian
\begin{equation}\label{eq:Ising ham reference model}
        H_{S_G}(\mathbf{v})= \sum_{\{i,j\} \in E} v_iv_j.
\end{equation}
and consider the spin model $\mathcal{M}_{\mathrm{Ising}}$ that contains all such systems, for arbitrary interaction graphs $G$.

Note that any graph with vertices $[n]$ can be fully specified by a ``binary'' \emph{graph encoding} $\mathbf{m}\in \mathbf{2}^{\frac{n(n-1)}{2}}$ with $\mathbf{m}=(m_{1,2},m_{1,3}, \ldots, m_{n-1,n})$, where 
\begin{equation}
    m_{i,j} = \begin{cases}
        +1 \ &\text{if} \ \{i,j\} \in E \\
        -1 \ &\text{if} \ \{i,j\} \notin E.
    \end{cases}
\end{equation}
We thus identify $G \equiv \mathbf{m}$.
With this, $\mathcal{M}_{\mathrm{Ising}}$ is given by 
\begin{equation}\label{eq: def refernce ising}
    \mathcal{M}_{\mathrm{Ising}} = \left\{S_{\mathbf{m}} \mid \mathbf{m}\in \mathbf{2}^{\frac{n(n-1)}{2}}, \ n \in \mathbb{N} \right\}.
\end{equation}

In \cref{ssec: complexity hard reference model} we prove that $\mathcal{M}_{\mathrm{Ising}}$ is maximally hard. Specifically we prove that 
\begin{enumerate}
    \item  Computing ground-state energies for $\mathcal{M}_{\mathrm{Ising}}$, i.e.\ the function \begin{equation}\label{eq:Ising GSE}
        \mathcal{G}(\mathbf{m}) = \min_{\mathbf{v}}H_{S_{\mathbf{m}}}(\mathbf{v})
    \end{equation}
    is $\NPclass$-hard (\cref{lem:Ising gse}).
    \item Computing partition functions for $\mathcal{M}_{\mathrm{Ising}}$, i.e.\ the function 
    \begin{equation}
        \mathcal{Z}(\mathbf{m}) = Z_{S_{\mathbf{m}}}(\beta=1) = \sum_{\mathbf{v}} e^{-H_{S_{\mathbf{m}}}(\mathbf{v})}
    \end{equation}
    is $\SharpPclass$-hard (\cref{lem:part fun comp}).
    \item Computing partition function ratios for $\mathcal{M}_{\mathrm{Ising}}$, i.e.\ the function 
    \begin{equation}
        \mathcal{C}(\mathbf{m}_1,\mathbf{m}_2) = \frac{\mathcal{Z}(\mathbf{m}_1)}{\mathcal{Z}(\mathbf{m}_2)},
    \end{equation}
    for $\mathbf{m}_1$ and $\mathbf{m}_2$ of the same length, is $\SharpPclass$-hard (\cref{lem:part fun ratios comp}).
\end{enumerate}
Note that the functions $\mathcal{G}, \mathcal{Z}, \mathcal{C}$ are defined for graph encodings $\mathbf{m}$ for arbitrarily many spins. We can equivalently consider them as function families $\mathcal{G}=\{G_n \mid n \in \mathbb{N}\}$, where $G_n$ is the restriction of $\mathcal{G}$ to graph encodings of $n$ spins, and similar for $\mathcal{Z}$ and $\mathcal{C}$. Then $n$ is the number of spins of the reference Ising systems for which $G_n$ denotes ground-state energies, $Z_n$ denotes partition functions and $C_n$ denotes partition function ratios.
\subsection{Separation 1: Simulations}\label{ssec: separation 1}
Here we show that (v,p)-maximal spin ML models can simulate certain $\NPclass$-hard functions, while shallow spin ML models cannot (\cref{thm:separation 1}). 
We start with the first part of the statement. 

\begin{lemma}[efficient GSE simulation]\label{lem:eff sim GSE}
    Any (v,p)-maximal spin ML model can efficiently simulate the $\NPclass$-hard ground-state energy function $\mathcal{G}$ of $\mathcal{M}_{\mathrm{Ising}}$ with cut-off at least $\Delta=1$.
\end{lemma}
\begin{proof}
    For system size $n$, we construct a spin system $S_n$ acting on spins $(\mathbf{m}, \mathbf{v})$ that efficiently simulates  the restriction of $\mathcal{G}$ to systems of $n$ spins,  $G_n(\mathbf{m})$. The claim then follows by definition of (v,p)-maximality (\cref{def:vp-max}).

    For $i,j$ consider the spin system $S_{i,j}$ with Hamiltonian
    \begin{equation}
        H_{S_{i,j}}(v_i,v_j,m_{i,j}) = \begin{cases}
            v_iv_j \ &\text{if} \ m_{i,j}=+1 \\
            0 \ &\text{if} \ m_{i,j}=-1.
        \end{cases}
    \end{equation}
    Define
    \begin{equation}
        S_n\coloneqq \sum_{i=1}^{n-1}\sum_{j=i+1}^n S_{i,j}.       
    \end{equation}
    Then 
    \begin{equation}
        H_{S_n}(\mathbf{m}, \mathbf{v}) = 
        H_{S_{\mathbf{m}}}(\mathbf{v}).
    \end{equation}
    This implies that 
    \begin{equation}
        \min_{\mathbf{v}}H_{S_n}(\mathbf{m}, \mathbf{v}) =\min_{\mathbf{v}} H_{S_{\mathbf{m}}}(\mathbf{v}) 
        = G_n(\mathbf{m}).  
    \end{equation}
    Hence, $S_n$ simulates $G_n$. Since $\vert S_{i,j} \vert = 2^3$, $\vert S_n \vert = \frac{n(n-1)}{2}2^3 = \mathcal{O}(n^2)$ this simulation is efficient. Since by construction $H_{S_n}$ is integer valued, the cut-off of this simulation is at least $\Delta=1$.   
\end{proof}
Note that scaling all target systems $\lambda S_{\mathbf{m}}$ and denoting the respective ground-state function by $\lambda\mathcal{G}$, the above construction yields an efficient simulation of the $\NPclass$-hard function $\lambda\mathcal{G}$ with cut-off $\Delta=\lambda$. However, without scaling the target, in general, increasing the cut-off is not possible (see \cref{ssec: supp cut-off}).
An explicit example simulation of $\mathcal{G}$ with DBMs as source model is provided in \cref{ssec: supp DBM} and extension to other ground-state energy functions can be found in \cref{rem:generalization NP-hard}.

We now turn to the second part of the separation statement we seek to prove, namely that  shallow models cannot simulate hard functions. 
\begin{lemma}[computability of shallow effective energies]\label{lem:shallow polytime}
    For a shallow spin ML model $\mathcal{M}$, its visible spin effective energies $\min_{\mathbf{h}}H_S(\mathbf{v},\mathbf{h})$ are polytime computable.
\end{lemma}
\begin{proof}
    The statement claims that given a full description\footnote{Namely, a list of all its local energy functions. That is,  for each $e \in E_S$, all $2^{\vert e \vert}$ values of $J_S(e)$. The length of this list is $\mathrm{poly}(\vert S \vert)$.} of any $S\in \mathcal{M}$ and any of its visible configurations $\mathbf{v}$, then $\min_{\mathbf{h}}H_S(\mathbf{v},\mathbf{h})$ is polytime computable.
    To prove it, note first that the input contains a full description of $S$, so evaluating any of its local energy functions is trivially polytime computable.
    In the following, we show that computing $\min_{\mathbf{h}}H_S(\mathbf{v},\mathbf{h})$ requires at most $\mathrm{poly}(\vert S \vert)$ such evaluations.

    Given $S \in \mathcal{M}$ and $\mathbf{v}$, we first compute any minimizer $\mathbf{h}_{\min}(\mathbf{v})$ of  $H_S(\mathbf{v},-)$.  
    Take any hidden spin $i$ and consider all those hyperedges of $S$ that contain $i$, i.e.\ $E_i\coloneqq \{e \in E_S \mid i \in e \}$. Since $S$ is shallow, each of these contains no hidden spins other than $i$. 
    So the respective local energy functions satisfy
    \begin{equation}
        J_S(e)(\mathbf{v}\vert_e,\mathbf{h}\vert_e) = J_S(e)(\mathbf{v}\vert_e,h_i).
    \end{equation}
    Let us define 
    \begin{equation}
        H_i(\mathbf{v},h_i) = \sum_{e \in E_i} J_S(e)(\mathbf{v}\vert_e,h_i).
    \end{equation}
    Since $H_i$ only depends on one hidden spin, computing the state $h_i$ that minimizes $H_i$ for given $\mathbf{v}$ requires only 2 evaluations of $H_i$, and thus $2\vert E_i \vert$ evaluations of local energy functions of $S$. 

    Since for different hidden spins $i\neq j$, $E_i$ and $E_j$ are disjoint,
    the Hamiltonian of $S$ can be expressed as 
\begin{equation}\label{eq:sum decomp shallow}
        H_S(\mathbf{v},\mathbf{h} ) = \sum_{i=1}^m H_i(\mathbf{v},h_i) + H'(\mathbf{v})
    \end{equation}
    where $m$ is the number of hidden spins of $S$, and $H'$ is the sum of all local energy functions that contain no hidden spins. Minimizing $H_S$ for given $\mathbf{v}$ hence amounts to minimizing each $H_i$. Thus, computing a minimizer $\mathbf{h}_{\min}(\mathbf{v})$ requires 
    \begin{equation}
       \sum_{i=1}^m 2\vert E_i \vert \leq 2 \vert E_S \vert \leq 2 \vert S \vert 
    \end{equation}
    evaluations of local energy functions of $S$.

    Finally, computing the  minimum $H_S(\mathbf{v},\mathbf{h}_{\min}(\mathbf{v}))$ requires another evaluation of each local energy function. In total, computing $\min_{\mathbf{h}}H_S(\mathbf{v},\mathbf{h})$ requires at most $3\vert S \vert=\mathrm{poly}(\vert S \vert)$ evaluations of local energy functions.
    Under our rational-encoding assumptions, these
    sums and comparisons can be performed exactly in
    polynomial time, since all intermediate rational values
    have polynomial bit length. Hence, $H_S(\mathbf{v},\mathbf{h}_{\min}(\mathbf{v}))$ is polytime computable.
\end{proof}

This leads to the separation of simulations. 
\begin{theorem}[complexity separation for simulations]\label{thm:separation 1} The $\NPclass$-hard function $\mathcal{G}$ can be efficiently simulated by any (v,p)-maximal spin ML model, but not by shallow spin ML models, unless the polynomial hierarchy collapses to its second level.
\end{theorem}
\begin{proof}
By \cref{lem:eff sim GSE}, any (v,p)-maximal model can efficiently simulate the ground-state energy function $\mathcal{G}$ of $\mathcal{M}_{\mathrm{Ising}}$. By \cref{lem:Ising gse}, computing $\mathcal{G}$ is $\NPclass$-hard.
     
We now prove that if there existed a shallow model $\mathcal{M}$ with $\mathcal{M} \vpolysim \mathcal{G}$, then  $\mathcal{G} \in \FPclass/\mathrm{poly}$. Since $\mathcal{G}$ is $\NPclass$-hard, 
    by the Karp--Lipton theorem \cite{Ka82} this would imply that the polynomial hierarchy collapses to its second level, which is believed to be very unlikely.

    Assuming a shallow $\mathcal{M}$ with $\mathcal{M} \vpolysim \mathcal{G}$, we construct the following $\FPclass/\mathrm{poly}$-algorithm for $\mathcal{G}$.
    For graph encodings of $n$-spin systems, take as advice string the full description of the system $S_n \in \mathcal{M}$ that simulates $G_n$, the restriction of $\mathcal{G}$ to encodings  of $n$ spins, and the respective energy shift $\Gamma_n$.
   Since $\mathcal{M} \vpolysim \mathcal{G}$, $\vert S_n \vert = \mathrm{poly}(n)$ so assuming that all parameters have polynomial bit complexity, the length of this is polynomial in $n$.
   Computing $G_n(\mathbf{v})$ amounts to computing the effective energy $\min_{\mathbf{h}}H_{S_n}(\mathbf{v},\mathbf{h})$, up to a potential energy shift $\Gamma_n$. Since $\mathcal{M}$ is shallow, by \cref{lem:shallow polytime}, given a full description of $S_n$, this requires a runtime of $\mathrm{poly}(\vert S_n \vert) = \mathrm{poly}(n)$. Thus, for each $n$, given a full description of $S_n$ and the respective energy shift $\Gamma_n$ as advice, $G_n$ can be computed in polynomial time. This implies that $\mathcal{G} \in \FPclass/\mathrm{poly}$.
\end{proof}
Note that, in general, if a shallow model efficiently simulates a function, then that function is necessarily in $\FPclass/\mathrm{poly}$. 

The separation \cref{thm:separation 1} can be further strengthened whenever the efficient simulations can additionally be efficiently computed, an additional assumption that we study in \cref{rem:polytime simulations} (see also \cite{Re24}).

\subsection{Separation 2: Unnormalized marginals}\label{ssec: separation 2}

Since the complexity separation for simulations (\cref{thm:separation 1}) uses a constant cut-off (of $\Delta=1$) it
need not imply a separation for the corresponding approximations (cf. \cref{ssec: supp cut-off}). 
Here we extend the complexity separation to unnormalized visible spin marginals, in \cref{thm:separation 2}.

We denote the unnormalized Boltzmann distribution of $S$ by $p^*_S=Z_Sp_S$.
Unnormalized visible spin marginals are given by 
\begin{equation}\label{eq:visibl spin marg}
    p^*_S(\mathbf{v})=\sum_{\mathbf{h}}e^{-H_S(\mathbf{v},\mathbf{h})}.
\end{equation}
Constant-degeneracy simulations preserve unnormalized visible spin marginals up to a common positive factor and a relative error controlled by the cut-off.
For normalized distributions, this common factor cancels. We consider probability ratios of normalized marginals in \cref{ssec: separation 3}.

We use that the partition function   $\mathcal{Z}$ of our reference model  $\mathcal{M}_{\mathrm{Ising}}$ is $\SharpPclass$-hard  by \cref{lem:part fun comp}. 

\begin{lemma}[efficient partition function approximation]\label{lem:sharp P}
    Any (v,p)-maximal and strongly visible closed spin ML model can
    efficiently approximate the partition function $\mathcal Z$
    of $\mathcal M_{\mathrm{Ising}}$ in its unnormalized visible
    spin marginals, to arbitrary relative error, up to a common positive factor.
    
    More precisely, for every $n$ and $0<\varepsilon<1$, there
    exist a system $T_{n,\varepsilon}$ of size $\mathrm{poly}(n)$
    and a factor $\alpha_{n,\varepsilon}>0$ such that
    \begin{equation}
    \max_{\mathbf m}
    \left|
        \frac{p^*_{T_{n,\varepsilon}}(\mathbf m)}
             {\alpha_{n,\varepsilon}Z_n(\mathbf m)}
        -1
    \right|
    \leq\varepsilon.
    \end{equation}
    The polynomial size bound is independent of $\varepsilon$.
    The factor $\alpha_{n,\varepsilon}$ may depend on $n$ and
    $\varepsilon$, but not on $\mathbf m$.
    \end{lemma}

Approximating $\mathcal{Z}$ to sufficient precision is itself $\SharpPclass$-hard (\cref{cor:part fun approx mult}). Thus, in the above sense, (v,p)-maximal and strongly visible closed models
can efficiently represent $\SharpPclass$-hard functions in their unnormalized visible spin marginals.

\begin{proof}
Fix $n$ and $0<\varepsilon<1$.
We construct a system $T_{n,\varepsilon}\in\mathcal M$
of size $\mathrm{poly}(n)$ and a factor
$\alpha_{n,\varepsilon}>0$ such that, for every graph
encoding $\mathbf m$,
\begin{equation}
    p^*_{T_{n,\varepsilon}}(\mathbf m)
    =
    \alpha_{n,\varepsilon}Z_n(\mathbf m)
    \bigl(1+r_{n,\varepsilon}(\mathbf m)\bigr),
    \qquad
    0\leq r_{n,\varepsilon}(\mathbf m)\leq\varepsilon.
\end{equation}

Consider the system $S_n$ from \cref{lem:eff sim GSE}.
It has size $\mathrm{poly}(n)$ and satisfies
\begin{equation}
    H_{S_n}(\mathbf m,\mathbf v)
    =
    H_{S_{\mathbf m}}(\mathbf v),
    \qquad
    \sum_{\mathbf v}
    e^{-H_{S_n}(\mathbf m,\mathbf v)}
    =
    Z_n(\mathbf m).
\end{equation}

Since $\mathcal M$ is (v,p)-maximal, for any cut-off
there exists a system $U\in\mathcal M$ that visibly and
efficiently simulates $S_n$ with constant degeneracy.
The visible spins of $U$ correspond to
$(\mathbf m,\mathbf v)$.
By strong visible closure, for any cut-off there exists
a system $T\in\mathcal M$ that simulates $U$ with constant
degeneracy and has visible spins corresponding only to
$\mathbf m$.
Both simulations are efficient, so $|T|=\mathrm{poly}(n)$.

Their composition simulates $S_n$ with constant degeneracy.
Moreover, the efficiency assumptions provide a polynomial
$q(n)$, independent of the cut-offs, that bounds the number
of auxiliary spins in this composed simulation.
Choose both cut-offs to be at least
\begin{equation}
    \Delta
    =
    q(n)\log 2+\log(1/\varepsilon),
\end{equation}
and denote the resulting system by $T_{n,\varepsilon}$.
The composed simulation then has cut-off at least $\Delta$.

Write $\mathbf h$ for its auxiliary configuration, excluding
the original spins $(\mathbf m,\mathbf v)$.
Let $a\leq q(n)$ be the number of these auxiliary spins,
$\Gamma$ the energy shift, and $d$ the constant degeneracy.
For each fixed $(\mathbf m,\mathbf v)$, exactly $d$
auxiliary configurations have energy
$H_{S_n}(\mathbf m,\mathbf v)+\Gamma$, while every other
auxiliary configuration has energy at least
$H_{S_n}(\mathbf m,\mathbf v)+\Gamma+\Delta$.
Consequently,
\begin{equation}
\begin{aligned}
    d e^{-\Gamma}e^{-H_{S_n}(\mathbf m,\mathbf v)}
    &\leq
    \sum_{\mathbf h}
    e^{-H_{T_{n,\varepsilon}}(\mathbf m,\mathbf v,\mathbf h)}
    \\
    &\leq
    d e^{-\Gamma}e^{-H_{S_n}(\mathbf m,\mathbf v)}
    \left(1+\frac{2^a-d}{d}e^{-\Delta}\right).
\end{aligned}
\end{equation}
Since $d\geq1$, our choice of $\Delta$ ensures that
\begin{equation}
    \frac{2^a-d}{d}e^{-\Delta}
    \leq 2^{q(n)}e^{-\Delta}
    =\varepsilon.
\end{equation}

Summing over $\mathbf v$ and setting
$\alpha_{n,\varepsilon}=d e^{-\Gamma}>0$ therefore gives
\begin{equation}\label{eq: unnormalized approx}
    \alpha_{n,\varepsilon}Z_n(\mathbf m)
    \leq
    p^*_{T_{n,\varepsilon}}(\mathbf m)
    \leq
    \alpha_{n,\varepsilon}Z_n(\mathbf m)
    (1+\varepsilon).
\end{equation}
This proves the required relative-error bound uniformly
over $\mathbf m$.
The common factor $\alpha_{n,\varepsilon}$ may depend on
the chosen accuracy through the simulation's degeneracy
and energy shift, but it is independent of $\mathbf m$.
\end{proof}

\begin{lemma}[computability of shallow unnormalized marginals]\label{lem:shallow marg}
    If $\mathcal{M}$ is shallow then its unnormalized visible spin marginals can be evaluated using polynomially
    many arithmetic operations and exponential evaluations.
\end{lemma}
\begin{proof}
     By \cref{eq:sum decomp shallow}, the Hamiltonian of a shallow system $S$ can be decomposed into a sum of $m$ terms each depending on a single hidden spin $i$, and a term  independent of hidden spins. This implies that 
\begin{equation}\label{eq:vis spin marg decomp}
    p^*_{S} (\mathbf{v})=\sum_{\mathbf{h}}e^{-H_{S}(\mathbf{v},\mathbf{h})} = e^{-H'(\mathbf{v})}\cdot \prod_{i=1}^m\sum_{h_i} e^{-H_i(\mathbf{v},h_i)}.
    \end{equation}
   It follows that computing $p^*_{S}(\mathbf{v})$ requires one evaluation of $H'$ and two evaluations of each $H_i$, hence $\mathrm{poly}(\vert S \vert)$ evaluations of local energy functions of $S$ (see the proof of \cref{lem:shallow polytime} for details).
\end{proof}

\begin{corollary}[computability of approximating shallow unnormalized marginals and ratios]\label{cor:shallow log eval}
For a shallow spin ML system $S$, both
\begin{equation}
    \log p^*_S(\mathbf v)
    \qquad\text{and}\qquad
    \log\frac{p_S(\mathbf v_1)}{p_S(\mathbf v_2)}
\end{equation}
can be approximated to additive error $2^{-b}$ in time
polynomial in the description length of $S$ and $b$.
\end{corollary}

\begin{proof}
By the factorization in \cref{lem:shallow marg},
\begin{equation}
    \log p^*_S(\mathbf v)
    =
    -H'(\mathbf v)
    +\sum_{i=1}^m
    \log\left(
        e^{-H_i(\mathbf v,-1)}
        +e^{-H_i(\mathbf v,+1)}
    \right).
\end{equation}
To avoid forming potentially large or small exponentials,
set
\begin{equation}
    L_i=\min\{H_i(\mathbf v,-1),H_i(\mathbf v,+1)\},
    \qquad
    D_i=|H_i(\mathbf v,+1)-H_i(\mathbf v,-1)|.
\end{equation}
Then
\begin{equation}
    \log p^*_S(\mathbf v)
    =
    -H'(\mathbf v)-\sum_{i=1}^m L_i
    +\sum_{i=1}^m\log(1+e^{-D_i}).
\end{equation}
The rational quantities $H'$, $L_i$ and $D_i$ are
polynomial-time computable.

Each remaining correction lies in $[0,\log 2]$ and
satisfies
\begin{equation}
    0\leq\log(1+e^{-D_i})\leq e^{-D_i}.
\end{equation}
If $D_i$ is sufficiently large, the correction can
therefore be replaced by zero at the requested accuracy.
Otherwise, $D_i$ is bounded by the required number
of precision bits, and standard finite-precision
evaluation of the exponential and logarithm takes
polynomial time.

Approximating each correction to additive error at most
$2^{-b}/\max\{1,m\}$ gives total error at most $2^{-b}$.
This requires only $b+\mathcal O(\log(m+1))$ precision
bits per correction and proves the first claim.

For the second claim, the global partition function
cancels:
\begin{equation}
    \log\frac{p_S(\mathbf v_1)}{p_S(\mathbf v_2)}
    =
    \log p^*_S(\mathbf v_1)-\log p^*_S(\mathbf v_2).
\end{equation}
Evaluating each term to additive error $2^{-(b+1)}$
gives the desired error bound for their difference.
\end{proof}

\begin{theorem}[complexity separation for unnormalized marginals]\label{thm:separation 2}
Any (v,p)-maximal and strongly visible closed spin ML model can efficiently approximate
the $\SharpPclass$-hard reference partition-function $\mathcal{Z}$ to arbitrary relative precision in its unnormalized
visible spin marginals, up to a common positive factor independent of the visible configuration.
Shallow models cannot admit such polynomial-size approximations unless the polynomial
hierarchy collapses to $\MAclass$.
\end{theorem}

\begin{proof}
 By \cref{lem:sharp P}, any (v,p)-maximal and strongly
visible closed model can efficiently approximate
$\mathcal Z$ to arbitrary relative precision in its
unnormalized visible spin marginals, up to a common
positive factor independent of the visible configuration.

We show that polynomial-size shallow approximations with
this property would imply
$\SharpPclass\subseteq\FPclass/\mathrm{poly}$.
Let $0<\rho_n<1$ be a sufficient relative accuracy from
\cref{cor:part fun approx mult}, with
$\log(1/\rho_n)=\mathrm{poly}(n)$.

Suppose that a shallow model admits the stated
approximations. For each $n$, choose a system $T_n$
of size $\mathrm{poly}(n)$ satisfying
\begin{equation}
    p^*_{T_n}(\mathbf m)
    =
    \alpha_n Z_n(\mathbf m)(1+r_{\mathbf m}),
    \qquad
    |r_{\mathbf m}|\leq\rho_n/8,
\end{equation}
where $\alpha_n>0$ is independent of $\mathbf m$.
Use the description of $T_n$ as advice. Its length is
polynomial under our bit-complexity assumptions.

To eliminate the unknown factor $\alpha_n$, let
$\mathbf m_0=-\mathbf 1_K$, where $K=\binom n2$,
encode the edgeless graph. Since $Z_n(\mathbf m_0)=2^n$,
define
\begin{equation}
    \widehat Z_n(\mathbf m)
    =
    2^n\frac{p^*_{T_n}(\mathbf m)}
              {p^*_{T_n}(\mathbf m_0)}
    =
    Z_n(\mathbf m)
    \frac{1+r_{\mathbf m}}{1+r_{\mathbf m_0}}.
\end{equation}
Consequently,
\begin{equation}
    \left|
        \frac{\widehat Z_n(\mathbf m)}{Z_n(\mathbf m)}-1
    \right|
    \leq
    \frac{2(\rho_n/8)}{1-\rho_n/8}
    \leq \frac{\rho_n}{2}.
\end{equation}

By \cref{cor:shallow log eval}, the logarithms of the
unnormalized marginals of $T_n$ are polynomial-time
computable to any requested additive precision.
We therefore evaluate $\widehat Z_n$ through
\begin{equation}
    \log\widehat Z_n(\mathbf m)
    =
    n\log 2
    +\log p^*_{T_n}(\mathbf m)
    -\log p^*_{T_n}(\mathbf m_0).
\end{equation}
Since
\begin{equation}
    2^n e^{-K}\leq Z_n(\mathbf m)\leq 2^n e^K
\end{equation}
and $\log(1/\rho_n)=\mathrm{poly}(n)$, polynomially many
precision bits suffice to obtain a numerical value
$\widetilde Z_n(\mathbf m)$ with relative error at most
$\rho_n/4$ with respect to $\widehat Z_n(\mathbf m)$.
Combining the representation and numerical errors gives
\begin{equation}
    \left|
        \frac{\widetilde Z_n(\mathbf m)}{Z_n(\mathbf m)}-1
    \right|
    \leq
    \frac{\rho_n}{2}
    +\frac{\rho_n}{4}\left(1+\frac{\rho_n}{2}\right)
    <\rho_n.
\end{equation}

Thus polynomial advice would suffice to approximate
$\mathcal Z$ to the relative accuracy required by
\cref{cor:part fun approx mult}. This implies
$\SharpPclass\subseteq\FPclass/\mathrm{poly}$,
which collapses the polynomial hierarchy to
$\MAclass$~\cite{Ba90}.
Unless this collapse occurs, no shallow model admits
the stated polynomial-size approximations.
\end{proof}

\subsection{Separation 3: Normalized marginals}\label{ssec: separation 3}
The separation for unnormalized marginals relies on
their efficient evaluation for
shallow models. For normalized marginals, the same
argument does not apply directly: the global partition
function of a shallow system can itself be computationally
hard to evaluate\footnote{This is e.g.\ the case for RBMs}. We therefore consider a target
distribution whose hardness remains accessible through
probability ratios, in which this normalization cancels.
We then prove a corresponding complexity separation between (v,p)-maximal and shallow models in terms of the approximation reach of their normalized visible marginals,
\begin{equation}
    p_S(\mathbf{v}) = \frac{1}{Z_S}\sum_{\mathbf{h}}e^{-H_S(\mathbf{v},\mathbf{h})}.
\end{equation}

Define the probability distribution on graph encodings (cf. \cref{ssec:hard reference model})
\begin{equation}\label{eq:normalized reference distribution}
    q_n(\mathbf m)
    =
    \frac{Z_n(\mathbf m)}
         {\sum_{\mathbf m'}Z_n(\mathbf m')},
\end{equation}
where the sum runs over all encodings of graphs on
$n$ spins. Further define the function family $\mathcal{Q} = \{q_n \mid n \in \mathbb{N} \}$.

For the edgeless graph, encoded by
$\mathbf m_0=-\mathbf 1_K$ with $K=\binom n2$,
we have $Z_n(\mathbf m_0)=2^n$. Consequently,
\begin{equation}\label{eq:reference distribution calibration}
    2^n\frac{q_n(\mathbf m)}{q_n(\mathbf m_0)}
    =
    Z_n(\mathbf m).
\end{equation}
Thus, approximating these probability ratios to
sufficient relative precision is $\SharpPclass$-hard by \cref{cor:part fun approx mult}.
The same holds for evaluating $q_n$ to sufficient
relative precision, since two such evaluations
determine the ratio.

We first show that the positive approximation result
for unnormalized marginals survives normalization.

\begin{lemma}[efficient probability distribution approximation ]\label{lem:part fun ratio model}
Any (v,p)-maximal and strongly visible closed spin ML
model can efficiently approximate $q_n$ in its normalized
visible spin marginals to arbitrary uniform relative
precision.

More precisely, for every $n$ and
$0<\varepsilon<1$, there exists a system
$T_{n,\varepsilon}$ of size $\mathrm{poly}(n)$ such that
\begin{equation}\label{eq:normalized reference approximation}
    \max_{\mathbf m}
    \left|
        \frac{p_{T_{n,\varepsilon}}(\mathbf m)}
             {q_n(\mathbf m)}
        -1
    \right|
    \leq\varepsilon.
\end{equation}
The same system approximates the reference
partition-function ratios with relative error at most
$\varepsilon$:
\begin{equation}\label{eq:ratio approx}
    \max_{\mathbf m_1,\mathbf m_2}
    \left|
        \frac{
            p_{T_{n,\varepsilon}}(\mathbf m_1)/
            p_{T_{n,\varepsilon}}(\mathbf m_2)
        }{
            C_n(\mathbf m_1,\mathbf m_2)
        }
        -1
    \right|
    \leq\varepsilon.
\end{equation}
The polynomial size bound is independent of
$\varepsilon$.
\end{lemma}

\begin{proof}
By \cref{lem:sharp P} and its proof, there exists
$T=T_{n,\varepsilon}$ of size $\mathrm{poly}(n)$ with
\begin{equation}
    p_T^*(\mathbf m)
    =
    \alpha Z_n(\mathbf m)(1+r_{\mathbf m}),
    \qquad
    0\leq r_{\mathbf m}\leq\varepsilon,
\end{equation}
where $\alpha>0$ is independent of $\mathbf m$.
Set
\begin{equation}
    \bar r
    =
    \sum_{\mathbf m}q_n(\mathbf m)r_{\mathbf m}.
\end{equation}
Since $q_n$ is a probability distribution,
$0\leq\bar r\leq\varepsilon$.
Normalizing $p_T^*$ gives
\begin{equation}
    p_T(\mathbf m)
    =
    q_n(\mathbf m)
    \frac{1+r_{\mathbf m}}{1+\bar r}.
\end{equation}
Therefore,
\begin{equation}
    \left|
        \frac{p_T(\mathbf m)}{q_n(\mathbf m)}-1
    \right|
    =
    \frac{|r_{\mathbf m}-\bar r|}{1+\bar r}
    \leq\varepsilon.
\end{equation}

For probability ratios, both normalization factors
cancel:
\begin{equation}
    \frac{p_T(\mathbf m_1)}{p_T(\mathbf m_2)}
    =
    C_n(\mathbf m_1,\mathbf m_2)
    \frac{1+r_{\mathbf m_1}}{1+r_{\mathbf m_2}}.
\end{equation}
Their relative error is at most $\varepsilon$, since
\begin{equation}
    \frac{|r_{\mathbf m_1}-r_{\mathbf m_2}|}
         {1+r_{\mathbf m_2}}
    \leq\varepsilon.
\end{equation}
\end{proof}

For shallow models, probability ratios are
polynomial-time evaluable in logarithmic form
by \cref{cor:shallow log eval}.
This prevents sufficiently accurate relative
approximation of $q_n$, regardless of the complexity
of the shallow model's global partition function.

\begin{theorem}[complexity separation for normalized marginals]
\label{thm:separation 3}
Any (v,p)-maximal and strongly visible closed spin ML
model can efficiently approximate the distributions
$q_n$ in its normalized visible spin marginals to
arbitrary uniform relative precision.
Shallow models cannot admit such polynomial-size
approximations unless the polynomial hierarchy
collapses to $\MAclass$.
\end{theorem}

\begin{proof}
The positive statement follows from
\cref{lem:part fun ratio model}.

Let $0<\rho_n<1$ be a sufficient relative accuracy
for the hardness result in
\cref{cor:part fun approx mult}, with
$\log(1/\rho_n)=\mathrm{poly}(n)$.
Suppose that shallow models admit the stated
polynomial-size approximations.
For each $n$, choose a shallow system $T_n$ such that
\begin{equation}
    p_{T_n}(\mathbf m)
    =
    q_n(\mathbf m)(1+s_{\mathbf m}),
    \qquad
    |s_{\mathbf m}|\leq\rho_n/8.
\end{equation}
Its description has polynomial length and can therefore be used as
advice.

Using the edgeless graph as a reference, define
\begin{equation}
    \widehat Z_n(\mathbf m)
    =
    2^n
    \frac{p_{T_n}(\mathbf m)}
         {p_{T_n}(\mathbf m_0)}
    =
    Z_n(\mathbf m)
    \frac{1+s_{\mathbf m}}{1+s_{\mathbf m_0}}.
\end{equation}
The representation error satisfies
\begin{equation}
    \left|
        \frac{\widehat Z_n(\mathbf m)}
             {Z_n(\mathbf m)}
        -1
    \right|
    \leq
    \frac{2(\rho_n/8)}{1-\rho_n/8}
    \leq\frac{\rho_n}{2}.
\end{equation}

The global partition function of $T_n$ cancels
from this ratio. By \cref{cor:shallow log eval},
we can evaluate
\begin{equation}
    \log\widehat Z_n(\mathbf m)
    =
    n\log 2
    +\log p_{T_n}^*(\mathbf m)
    -\log p_{T_n}^*(\mathbf m_0)
\end{equation}
to any requested additive precision in polynomial
time.
Since
\begin{equation}
    2^n e^{-K}
    \leq Z_n(\mathbf m)
    \leq 2^n e^K,
    \qquad K=\binom n2,
\end{equation}
and $\log(1/\rho_n)=\mathrm{poly}(n)$, we can obtain
a numerical value $\widetilde Z_n(\mathbf m)$
with relative error at most $\rho_n/4$ with respect
to $\widehat Z_n(\mathbf m)$ in polynomial time.
Combining the two errors gives
\begin{equation}
    \left|
        \frac{\widetilde Z_n(\mathbf m)}
             {Z_n(\mathbf m)}
        -1
    \right|
    \leq
    \frac{\rho_n}{2}
    +\frac{\rho_n}{4}
     \left(1+\frac{\rho_n}{2}\right)
    <\rho_n.
\end{equation}

Thus polynomial advice suffices to approximate
$\mathcal Z$ to the relative accuracy required by
\cref{cor:part fun approx mult}.
This implies
$\SharpPclass\subseteq\FPclass/\mathrm{poly}$,
and hence the polynomial hierarchy collapses to
$\MAclass$~\cite{Ba90}.
\end{proof}

A computationally hard global partition function
therefore cannot overcome the restriction on shallow
models: it changes the common normalization, while
the target distribution requires computationally
hard relative probabilities.

\section{The efficient reach of shallow models}\label{sec:RBM separation}
In this section we study the efficient reach of shallow models and characterize it based on their interaction order, their maximal number of spins per local interaction. 
In \cref{ssec:small weights} we introduce shallow $k$-local models and lay out an additional assumption used in this section, namely that of small weights, i.e.\ polynomially bounded rational weights with polynomially bounded common denominator.
In \cref{ssec:upper k local}, we characterize the efficient reach of shallow $k$-local models by linear combinations of polynomial threshold functions (\cref{lem:sim PTF}) of degree $k-1$. This allows us to upper bound the efficient reach of shallow $k$-local functions in terms of small $k$-party communication complexity. This bound is violated by the $k$-party inner product modulo two function, $\mathrm{IP}_k$, implying that it cannot be efficiently simulated by shallow $k$-local models.
In \cref{ssec:lower k local}, we characterize the efficient reach of RBMs and extend these results to shallow $k$-local models. This implies that shallow $k$-local models can efficiently simulate the $k-1$-party inner product modulo $2$, $\mathrm{IP}_{k-1}$. 
The main results of this section are thus:
\begin{informalresult}[characterization of the efficient reach of shallow $k$-local models]
The efficient reach of shallow $k$-local models is characterized by an infinite hierarchy. For each $k$, shallow, $k$-local models cannot efficiently simulate functions with large $k$-party communication complexity (\cref{lem:shallow PP,thm:shallow rand}), but shallow, $k+1$-local models can efficiently simulate some of them, e.g.\ the $k$-party inner product modulo $2$, $\mathrm{IP}_k$ (\cref{lem:IP k sim}). 
\end{informalresult}

\subsection{Shallow $k$-local models with small weights}\label{ssec:small weights}
In many aspects, shallow $k$-local models, for fixed $k$,  can be considered generalizations of RBMs. Since they are shallow, each $k$-body interaction can contain at most $1$ hidden spin, which implies that their effective energies as well as unnormalized visible spin marginals  are efficiently computable (see \cref{sec:consequences}). 
For $k=3$ consider e.g.\ the $3$-body generalization of RBMs, that is, $3$-body RBMs or 3RBMs \cite{Ra10}. For each pair of visible spins $(v_i,v_j)$, and for each hidden spin $h_k$, 3RBMs additionally contain the respective $3$-body interaction (see \cref{fig:3RBM}). They are thus described by the following Hamiltonian
\begin{equation}\label{eq:3RBM}
\begin{aligned}
H(\mathbf v,\mathbf h)
&=\sum_{j=1}^n b_j^{(v)}v_j
+\sum_{r=1}^m b_r^{(h)}h_r\\
&\quad+\sum_{r=1}^m\sum_{j=1}^n W_{rj}h_rv_j
+\sum_{r=1}^m\sum_{i<j}C_{rij}h_rv_iv_j.
\end{aligned}
\end{equation}

\begingroup
\definecolor{TRBMpurple}{HTML}{8866AA}
\definecolor{TRBMink}{HTML}{28353E}
\begin{figure}[!t]
\centering
\begin{tikzpicture}[
 font=\small,text=TRBMink,line join=round,
 vis/.style = {draw=black!80, fill=blue!60, thick, circle, minimum size= 0.65em, inner sep=2pt},
    hid/.style = {draw=black!80, fill=red!60, thick, circle, minimum size= 0.65em, inner sep=2pt}, scale=0.8]
 \foreach \i/\x/\y in {1/-4/0,2/-2/-.55,3/0/-.75,4/2/-.55,5/4/0}
  \coordinate (v\i) at (\x,\y);
 \foreach \j in {1,...,8}
  \coordinate (h\j) at ({1.3*(\j-4.5)},4.2);

 \foreach \j in {1,...,8} {
  \foreach \a in {1,...,4} {
   \pgfmathtruncatemacro{\nextvisible}{\a+1}
   \foreach \b in {\nextvisible,...,5} {
    \path[fill=TRBMpurple,fill opacity=.009,
      draw=TRBMpurple,draw opacity=.10,line width=.35pt]
      (v\a)--(v\b)--(h\j)--cycle;
   }
  }
 }
 \foreach \i in {1,...,5} {
  \node[vis] at (v\i) {};
  \node[below=8pt] at (v\i) {};
 }
 \foreach \j in {1,...,8} {
  \node[hid] at (h\j) {};
  \node[above=8pt] at (h\j) {};
 }
\end{tikzpicture}
\caption{\textbf{A fully connected $3$-body RBM.}
Visible spins in blue hidden spins in red.
Each translucent triangle denotes a three-body interaction involving
one hidden spin and two distinct visible spins. }
\label{fig:3RBM}
\end{figure}
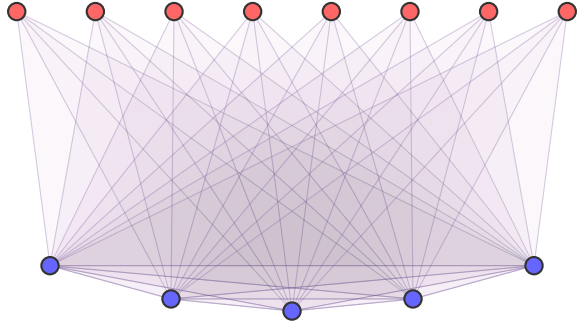
\endgroup

Note, however, that in contrast to RBMs, for shallow $k$-local models, visible spins are not generally conditionally independent given a hidden configuration. Hence, certain advantages of RBMs do not directly extend to all shallow $k$-local models, e.g.\ training via contrastive divergence can no longer rely on the efficient parallel sampling of visible spins \cite{Hi02}.

Throughout this section we consider pseudo-Boolean target functions instead of generic target systems, since this is more appropriate in the context of threshold functions and communication complexity. When stating that a shallow $k$-local system $S$ efficiently simulates a function $f$, this is always to be understood in the context of a shallow $k$-local model $\mathcal{M}$ with $S \in \mathcal{M}$ and such that $f$ is defined for inputs of arbitrary length $n$, so strictly speaking $f$ is a family of functions (cf.\ \cref{sec:consequences}).  
Importantly, we then assume that the interaction order $k$ does not depend on the target size $n$.

Finally, throughout this section, we additionally require that all spin ML models have \emph{small weights}, that is, intuitively, rational weights of polynomially bounded magnitude and precision.
\begin{definition}[small weights]\label{def:poly bounded weights}
    We say that a spin model $\mathcal{M}$ has
\begin{itemize}
\item 
 \emph{polynomially bounded weights} if all its $n$-spin systems have weights (i.e.\ local energy values, cf.\ \cref{def:spin system}) of magnitude at most $\mathrm{poly}(n)$;
    \item 
    \emph{weights of polynomial precision} 
    if for any of its $n$-spin systems $S$ there exists an integer $q_S= \mathrm{poly}(n)$ such that all weights of $S$ can be expressed as $p/q_S$, for some integer $p$; 
    \item 
    \emph{small weights}  if its weights are polynomially bounded and of polynomial precision.
    \end{itemize}
\end{definition}
Note that this assumption of small weights is stronger than the assumption of weights of polynomial bit complexity used in \cref{sec:consequences}.

In the context of machine learning models, the assumption of bounded weights is reasonable, since there are several results that link models with bounded weights to better performance \cite{Be09,Ba98}.
Similar constraints are also common in other computational models, such as circuits of threshold gates \cite{Ka20, Go92}, as well as in physics contexts, where the weights of a spin model are usually considered constant and not to scale with the system size.

The additional polynomial-precision assumption can be understood
as a finite-resolution restriction: all weights then lie
on a common grid with inverse-polynomial spacing.
Such a grid suffices to approximate the Hamiltonian of a
polynomial-size system uniformly to inverse-polynomial accuracy.

The property of small weights does not depend on the chosen representation of the spin system:
If, when specifying weights as values of the local energy functions of $S$ (see \cref{def:spin system}) all such weights of $S$ are polynomially bounded and/or of polynomial precision then, by the Fourier decomposition of pseudo-Boolean functions (\cref{eq:fourier coeff}), also all Fourier coefficients of $S$, that is, its weights w.r.t.\ the Fourier basis, are polynomially bounded and/or of polynomial precision.

We conjecture that under mild additional assumptions, the separation results of \cref{sec:consequences}  (\cref{thm:separation 1}, \cref{thm:separation 2} and \cref{thm:separation 3}) are compatible with the constraint of small weights, specifically with the constraint that targets with small weights must be simulated by a source with small weights, i.e.\ weights must remain small under simulation. More precisely, we conjecture this to be the case, whenever for a (v,p)-maximal model, it holds for all simulations that characterize (v,p)-maximality (i.e.\ visible closure and universality, through closure and functional completeness, see \cref{def:visible closed,thm:main}, and also \cite{Re24}).

\subsection{Upper bounds on the efficient reach of shallow $k$-local systems}\label{ssec:upper k local}
In this subsection we characterize those target functions $f$ that can be efficiently simulated by shallow $k$-local models. 

We first prove that efficiently simulating a pseudo-Boolean function $f$ with a shallow source $S$ essentially is equivalent to representing $f$ as a linear combination of polynomially many polynomial threshold functions (PTFs). 

\begin{lemma}[PTF representation]\label{lem:sim PTF}
If a target family $f$ can be efficiently simulated
by shallow $k$-local systems, then it admits a PTF
representation
\begin{equation}\label{eq:PTF representation main 2}
    f(\mathbf v)
    =
    \sum_{i=1}^{M}w_i\sgn(P_i(\mathbf v)),
    \qquad
    M=\mathrm{poly}(n),
\end{equation}
where each polynomial $P_i$ has degree at most $k-1$.\footnote{Throughout, we use the convention $\sgn(0)=+1$.}
If the simulation weights and energy shifts are
polynomially bounded or have polynomial precision,
the PTF representation can be chosen with the corresponding
property, for all its parameters.

Conversely, if $f$ admits such a PTF representation with
small parameters\footnote{Here we mean that all parameters internal to the PTFs as well as the linear coefficients $w_i$ jointly satisfy \cref{def:poly bounded weights}.}, then it can be efficiently simulated
by a shallow $k$-local system with small weights and
energy shifts.
\end{lemma}
\begin{proof}
    See \cref{ssec:proof 1} in the appendix.
\end{proof}

It follows from \cref{lem:sim PTF} that any function that can be efficiently simulated by an RBM can be represented by a linear combination of linear threshold functions (LTFs) (see \cref{ssec:bool basics}).

Next, we prove that under the assumption of small weights, \cref{lem:sim PTF} implies that functions $f$ that can be efficiently simulated by $k$-local models necessarily have large correlation $\langle \sgn(P), f \rangle$ with some PTF. 
By definition of $\langle , \rangle$, the inner product of pseudo-Boolean functions (see \cref{ssec:bool basics}) this means that the PTF $\sgn(P)$ provides a reasonable approximation of $f$.
\begin{lemma}[simulation implies correlation]\label{lem:shallow Corr}
    If a target function $f$ with $\langle f,f \rangle = \frac{1}{\mathrm{poly}(n)}$ can be efficiently simulated by a shallow, $k$-local system $S$ with polynomially bounded weights and energy shift $\Gamma$, then $f$ has large (inverse polynomial) correlation with some degree $k-1$ PTF $\sgn(P)$, i.e.\
    \begin{equation}
        \langle f, \sgn(P)\rangle = \frac{1}{\mathrm{poly}(n)}.
    \end{equation} 
\end{lemma}
\begin{proof}
    By \cref{lem:sim PTF}, $f$ can be represented by a linear combination of PTFs, 
    \begin{equation}\label{eq:PTF representation corr}
        f(\mathbf{v}) = \sum_I w_I \sgn(P_I(\mathbf{v})).
    \end{equation}
    W.l.o.g.\ the polynomials $P_I$ never vanish
on Boolean inputs and all coefficients $w_i$ are positive. 
Indeed, for each $I$, choose $0<\eta_I\leq1$ smaller than every
nonzero value $|P_I(\mathbf v)|$. Replacing $P_I$ by
$P_I+\eta_I$ preserves its sign, degree and polynomially
bounded coefficients, and removes all zeros. We may then
absorb the signs of the nonzero $w_I$ into the polynomials.
Thus
    \begin{equation}
    \begin{split}
        \frac{1}{\mathrm{poly}(n)} = \langle f,f \rangle = \frac{1}{2^n} \sum_{\mathbf{v}}  (f(\mathbf{v}))^2   
         =\sum_I w_I \frac{1}{2^n}\sum_{\mathbf{v}} f(\mathbf{v})\sgn(P_I(\mathbf{v})) = \sum_I w_I \langle f,\sgn(P_I)\rangle.
    \end{split}
    \end{equation}
    Let $\sgn(P)$ be the PTF of \cref{eq:PTF representation corr} that has maximal correlation with $f$, i.e.\ 
    \begin{equation}
        \langle f, \sgn(P)\rangle = \max_I \langle f, \sgn(P_I) \rangle.
    \end{equation}
    Since all coefficients satisfy $w_I\geq 0$, 
    \begin{equation}
         \frac{1}{\mathrm{poly}(n)} \leq \langle f,\sgn(P) \rangle\cdot \left(\sum_I w_I\right).
    \end{equation}
    Dividing by $\sum_I w_I=\mathrm{poly}(n)$ (see \cref{lem:sim PTF}) proves the claim.
\end{proof}
Note that by \cref{lem:sim PTF} $P$ has polynomially bounded coefficients. If,  additionally, the weights of $S$ are of polynomial precision then so are the coefficients of $P$. 

\cref{lem:shallow Corr} assumes that $\langle f, f \rangle = \frac{1}{\mathrm{poly}(n)}$. This is e.g.\ the case if $\vert f (\mathbf{v}) \vert = \frac{1}{\mathrm{poly}(n)}$ for all configurations $\mathbf{v}$ and thus in particular whenever $f$ is Boolean valued or has sufficient margin around $0$.

In the following we prove that even if this is not the case, if $f$ can be efficiently simulated with small weights  
then $\sgn(f)$ has large correlation with some PTF.
\begin{lemma}[shallow simulation sign correlation]\label{cor:shallow Corr sign}
    If a target function $f$ can be efficiently simulated by a shallow $k$-local source $S$ with small weights and energy shift $\Gamma$ then $\sgn(f)$ has large (inverse polynomial) correlation with some degree $k-1$ PTF $\sgn(P)$
    \begin{equation}
        \langle \sgn(f), \sgn(P)\rangle = \frac{1}{\mathrm{poly}(n)}.
    \end{equation} 
\end{lemma}
\begin{proof}
    We prove that if $f$ can be efficiently simulated with small weights then we can construct a scaled and shifted function $f'$ and an efficient simulation of $f'$ such that  $\sgn(f)=\sgn(f')$ and $f'$ has margin around $0$. Then, similar to \cref{lem:shallow Corr}, we  prove that $\sgn(f')$ and hence also $\sgn(f)$ has large correlation with some PTF.

    Specifically, if some shallow, $k$-local system $S$ efficiently simulates $f$ with small weights and energy shift, then all target values $f(\mathbf{v})$ can be expressed as a sum of at most polynomially many fractions $\frac{p}{q_S}$ 
    for polynomially bounded integers $p,q_S$ . This implies that non-zero target values $f(\mathbf{v})\neq 0$ satisfy 
    \begin{equation}
        \vert f (\mathbf{v}) \vert \geq \frac{1}{q_S} =  \frac{1}{\mathrm{poly}(n)}.
    \end{equation}
    Then $f'=2q_Sf +1$ preserves the sign while having absolute value at least $1$.
    
    Note that scaling and shifting $S$ appropriately yields a shallow $k$-local system $S'$ that efficiently simulates $f'$ with small weights and energy shift and hence, by \cref{lem:sim PTF}, $f'$ can be expressed as a linear combination of PTFs (see \cref{eq:PTF representation corr}).

    Now since all target values satisfy $\vert f'(\mathbf{v}) \vert \geq 1$ and $\sgn(f)=\sgn(f')$,
    \begin{equation}
        \begin{aligned}
            1 &\leq \frac{1}{2^n}\sum_{\mathbf v}|f'(\mathbf v)|\\
            &=\frac{1}{2^n}\sum_{\mathbf v}
            \sgn(f(\mathbf v))f'(\mathbf v)\\
            &=\sum_I w_I\frac{1}{2^n}\sum_{\mathbf v}
            \sgn(f(\mathbf v))\sgn(P_I(\mathbf v))\\
            &=\sum_I w_I\langle\sgn(f),\sgn(P_I)\rangle.
        \end{aligned}
    \end{equation}
Now similar to \cref{lem:shallow Corr},
let $\sgn(P)$ be the PTF that has maximal correlation with $\sgn(f)$, then 
    \begin{equation}
        1 \leq \langle \sgn(f'),\sgn(P) \rangle\cdot \left(\sum_I w_I\right)
    \end{equation}
    and dividing by $\sum_I w_I=\mathrm{poly}(n)$ proves the claim.
\end{proof}

\begin{theorem}[shallow simulation communication complexity]\label{lem:shallow PP}
     Whenever a target function $f$ can be efficiently simulated by a shallow, $k$-local system with small weights and energy shift $\Gamma$
     then $\sgn(f)$ has an efficient, $\mathcal{O}(\log(n))$, $k$-party NOF communication protocol with success rate at least $\frac{1}{2} + \frac{1}{\mathrm{poly}(n)}$.
\end{theorem}
\begin{proof}
We first argue that proving the claim amounts to proving that any degree $k-1$ PTF with small parameters has a $\mathcal{O}(\log(n))$ $k$-party communication protocol with success rate $1$. 
For any Boolean-valued target function $f'(\mathbf{v})$, the success rate $\mathrm{S}(\Pi)$ of a communication protocol that on input $\mathbf{v}$ returns $\Pi(\mathbf{v})$ can be expressed as
\begin{equation}
    \mathrm{S}(\Pi) = \frac{1}{2}\left (1+  \langle f', \Pi \rangle\right).
\end{equation}
Since by \cref{cor:shallow Corr sign}, $\sgn(f)$ has large correlation with some PTF $\sgn(P)$, defining the output of the protocol for $\sgn(f)$ to be  $\Pi=\sgn(P)$ yields an efficient $\mathcal{O}(\log(n))$ $k$-party communication protocol with success rate at least $\frac{1}{2} + \frac{1}{\mathrm{poly}(n)}$.

We now prove that any PTF $\sgn(P_I)$ has a  $\mathcal{O}(\log(n))$ $k$-party communication protocol.
Recall that in the $k$-party NOF communication setting the input $\mathbf{v}=(\mathbf{v}_1, \ldots, \mathbf{v}_k)$ is split amongst the $k$ players such that player $j$ has access to all blocks $\mathbf{v}_i$ for $i\neq j$.
Since each polynomial has degree $k-1$ it cannot contain monomials that depend on inputs from all $k$ blocks. We can thus decompose $P_I$ into polynomials $Q_j$, 
\begin{equation}
    P_I(\mathbf{v}) = \sum_{j=1}^k Q_j(\mathbf{v}),
\end{equation}
such that $Q_j$ does not depend on $\mathbf{v}_j$. This implies that $Q_j$ can be computed locally  by player $j$. Thus, the protocol for $\sgn(P_I)$ can be designed as follows: Each player $j$ for $j\leq k-1$ locally computes $Q_j(\mathbf{v})$ and communicates the value to player $k$. Player $k$ then locally computes $Q_k(\mathbf{v})$ and, using the communicated values, $P_I(\mathbf{v})$ to output $\sgn(P_I(\mathbf{v}))$.
By construction this protocol has success rate $1$.

To finish the proof, we argue that each value $Q_j(\mathbf{v})$ can be communicated in $\mathcal{O}(\log(n))$ bits, so the overall communication requires $(k-1)\cdot \mathcal{O}(\log(n))$ bits, i.e.\ the protocol has communication complexity $\mathcal{O}(\log(n))$. 
Recall that $P_I$ and hence also each $Q_j$ is a linear combination of $\mathrm{poly}(n)$ monomials, with polynomially bounded parameters of polynomial precision $2^kq_S$. So all parameters of these linear combinations are of the form $\frac{p}{2^kq_S}$ for an integer $p$ with $\vert p \vert = \mathrm{poly}(n)$. Therefore, also the value of each $Q_j$ can be expressed as $\frac{p_j}{2^kq_S}$ with $\vert p_j \vert = \mathrm{poly}(n)$ and thus, communicating $Q_j$ amounts to communicating $p_j$, which requires $\mathcal{O}(\log(n))$ bits.
\end{proof}

\begin{theorem}[shallow simulation randomized communication complexity]\label{thm:shallow rand}
    Whenever a target function $f$ can be efficiently simulated by a shallow, $k$-local system with small weights and energy shift $\Gamma$ then $\sgn(f)$ has an efficient, $\mathcal{O}(\log(n))$, randomized $k$-party NOF communication protocol that for any input $\mathbf{v}$ has success probability at least $\frac{1}{2} + \frac{1}{\mathrm{poly}(n)}$.
\end{theorem}
Note that this is strictly stronger than \cref{lem:shallow PP} since any such randomized communication protocol $\Pi$ would immediately yield success probability $\geq \frac{1}{2}+\frac{1}{\mathrm{poly}(n)}$. Since a randomized communication protocol is a probability distribution over deterministic protocols this implies that $\Pi$ must contain at least one deterministic protocol that has success rate over uniformly distributed inputs $\geq \frac{1}{2}+\frac{1}{\mathrm{poly}(n)}$.
\begin{proof}
    First, w.l.o.g.\ assume that for all inputs $\mathbf{v}$, $\vert f(\mathbf{v}) \vert \geq 1$, since else, following the proof of \cref{cor:shallow Corr sign}, we construct a scaled and shifted function $f'$, that satisfies $\vert f'(\mathbf{v}) \vert \geq 1$ for all inputs, can be efficiently simulated with small weights and satisfies $\sgn(f) = \sgn(f')$. Thus we can equivalently construct the protocol for $\sgn(f')$.
    
    Recall that by \cref{lem:sim PTF}, $f$ can be represented by a linear combination of degree $k-1$ PTFs, i.e.\ 
    \begin{equation}
        f(\mathbf{v}) = \sum_I w_I \sgn(P_I(\mathbf{v})).
    \end{equation}
    Similar to \cref{lem:shallow Corr}, w.l.o.g.\ all parameters $w_I$ are positive since else we can simply absorb signs by redefining the respective polynomial $P_I$.
    Following the proof of \cref{lem:shallow PP}, each individual PTF $\sgn(P_I)$ can be exactly computed by a protocol with communication complexity $\mathcal{O}(\log(n))$.
    
    The randomized protocol $\Pi$ for $\sgn(f)$ is constructed as follows:
    First, define $W = \sum_I  w_I $. The players start by using the shared randomness to sample an index $I$ with probability $p_I=\frac{ w_I }{W}$.
    They then follow the deterministic $\mathcal{O}(\log(n))$ protocol from \cref{lem:shallow PP} to return $\sgn(P_I(\mathbf{v}))$.

    To finish the proof, fix an input $\mathbf{v}$.
Since $\Pi(\mathbf{v})$ and $\sgn(f(\mathbf{v}))$ take values
in $\{-1,+1\}$, the success probability satisfies
\begin{equation}
    \Pr[\Pi(\mathbf{v})=\sgn(f(\mathbf{v}))]
    =
    \frac{1+\mathbb{E}[
        \Pi(\mathbf{v})\sgn(f(\mathbf{v}))
    ]}{2},
\end{equation}
where the expectation is over the shared randomness of the
protocol.
This expectation can be expressed as
    \begin{equation}
        \begin{split}
            \mathbb{E}[\Pi(\mathbf{v})\cdot \sgn(f(\mathbf{v}))] = \sgn(f(\mathbf{v}))\cdot \left(\sum_I \frac{ w_I }{W}\sgn(P_I(\mathbf{v})) \right)\\
            = \frac{1}{W}\sgn(f(\mathbf{v}))\cdot f(\mathbf{v}) = \frac{\vert f (\mathbf{v}) \vert }{W}\geq \frac{1}{W}= \frac{1}{\mathrm{poly}(n)}.
        \end{split}
    \end{equation} 
\end{proof}

Through this connection to communication complexity, the efficient reach of shallow, $k$-local models is in turn connected to various measures of communication complexity, including
discrepancy \cite{Ba92}, margin complexity
\cite{LiSh09margin}, approximate $\gamma_2$ factorization norms
\cite{LiSh09factor}, and sign rank \cite{PaSi86}.

\begin{corollary}[non-existence of shallow, $k$-local $\mathrm{IP}_k$ simulation]\label{cor:rbm ip}
    Shallow $k$-local systems with small weights cannot efficiently simulate the $k$-party inner product modulo $2$, $\mathrm{IP}_{k}$ (see \cref{ssec:bool basics}).
\end{corollary}
\begin{proof}
Note that $\sgn(\mathrm{IP}_k)=\mathrm{IP}_k$.
By \cref{lem:shallow PP}, the existence of a shallow, $k$-local system that simulates $\mathrm{IP}_{k}$ would induce a $\mathcal{O}(\log(n))$ communication protocol for $\mathrm{IP}_{k}$ with success rate at least $\frac{1}{2}+\frac{1}{\mathrm{poly}(n)}$, but by \cite[Theorem 2]{Ba92} any such protocol requires the communication of at least $\Omega(n)$ bits.
\end{proof}

\subsection{Lower bounds and separation hierarchy of shallow $k$-local systems}\label{ssec:lower k local}
In this subsection we lower bound the efficient reach of shallow $k$-local models. 
Note that several lower bounds can be directly derived from the characterization of universal spin models (\cite{Re24,Re25}). It e.g.\ implies that shallow $k$-local models can efficiently simulate all functions that have a sparse representation in either the indicator or the Fourier basis (\cref{ssec:bool basics}), i.e.\ that have at most polynomially many, non-zero coefficients.

We first prove that RBMs can efficiently simulate arbitrary, small ridge functions, that is, functions $f(\mathbf{v})=g(\mathbf{A}\mathbf{v})$ where $\mathbf{A}$ has small components (see \cref{ssec:bool basics}).

\begin{lemma}[RBM ridge function simulation]\label{lem:rbm ridge}
RBMs can efficiently simulate all small ridge functions, for arbitrary cut-offs. 

Conversely, if a target can be efficiently simulated by RBMs with small weights and energy shift then it can be expressed as a sum of polynomially many small ridge functions.
\end{lemma}
\begin{proof}
    See \cref{ssec:proof 2} in the appendix.
\end{proof}

In particular, \cref{lem:rbm ridge} implies that the cut-off
of an efficient RBM simulation with small weights and energy
shift can be increased to any prescribed polynomially bounded
value, while retaining small weights and energy shift (cf. \cref{ssec: supp cut-off}).

The class of ridge functions contains many interesting examples. 
Consider a ridge function $f(\mathbf{v})=g(\mathbf{A}\mathbf{v})$ for $\mathbf{A}=\mathbf{1}$ then, since 
\begin{equation}\label{eq:ridge hamming weight}
 \mathbf{1} \mathbf{v} = 2\mathrm{d}_H(\mathbf{v})-n,
\end{equation}
$f$ only depends on the Hamming weight $\mathrm{d}_H(\mathbf{v})$ of the input configuration, the number of bits $i$ with $v_i=+1$.
Similarly, consider a ridge function with $\mathbf{A}= -\mathbf{w} \in \mathbf{2}^n$, then 
\begin{equation}
    -\mathbf{w}\mathbf{v} = 2\mathrm{d}_H(\mathbf{v},\mathbf{w})-n,
\end{equation}
where $\mathrm{d}_H(\mathbf{v},\mathbf{w})$ is the Hamming distance between $\mathbf{v}$ and $\mathbf{w}$, the number of bits $i$ with $v_i\neq w_i$.
This shows that ridge functions contain all distance functions, that is, functions that only depend on the (Hamming) distance to some fixed configuration.
Important examples of distance functions include indicator functions, that is, 
\begin{equation}
    \delta_{\mathbf{s}}(\mathbf{v}) = \begin{cases}
        1 \ &\text{if} \ \mathbf{v}= \mathbf{s} \\ 
        0 \ &\text{else},
    \end{cases}
\end{equation}
for any $\mathbf{s}\in \mathbf{2}^n$. Indicator functions provide the canonical basis for the space of pseudo-Boolean functions. 

Other important examples of ridge functions are parity functions, that is, functions that only depend on the parity of a subconfiguration of $\mathbf{v}$. Consider $S\subseteq [n]$ and define
\begin{equation}
    \pi_S(\mathbf{v}) = \prod_{i\in S}v_i.
\end{equation}
Note that $\pi_S(\mathbf{v})$ yields $+1$ if the subconfiguration $\mathbf{v}\vert_S$ has even parity, i.e.\ an even number of $(-1)s$ and $-1$ else.
Parity functions provide a basis of the space of pseudo-Boolean functions. The resulting decomposition  
\begin{equation}
    f(\mathbf{v}) = \sum_{S \subseteq [n]}\widehat{f}(S)\pi_S(\mathbf{v})
\end{equation}
is called the Fourier decomposition of $f$ and the coefficients $\widehat{f}(S)$ are called Fourier coefficients of $f$.

To see that $\pi_S$ is a ridge function, consider the weight vector
$\mathbf{A}_S \in \{0,1\}^n$ such that $(A_S)_i=+1$ if $i\in S$ and $(A_S)_i=0$ if $i\notin S$. Denote the number of $(-1)s$ in the subconfiguration $\mathbf{v}\vert_S$ as $n_-(\mathbf{v}\vert_S)$ then 
\begin{equation}
    \mathbf{A}_S\mathbf{v} = \vert S \vert - 2 n_-(\mathbf{v}\vert_S).
\end{equation}
thus, choosing $g(k)= (-1)^{\frac{\vert S \vert -k}{2}}$ yields 
\begin{equation}
    g(\mathbf{A}_S\mathbf{v}) = (-1)^{n_-(\mathbf{v}\vert_S)} = \pi_S(\mathbf{v}).
\end{equation}

Finally, ridge functions also include linear threshold functions (LTFs)
\begin{equation}
    f(\mathbf{v}) = \sgn(\mathbf{A}\mathbf{v}+b),
\end{equation}
where we use the convention $\sgn(0)=+1$.
Similar to circuits of logical gates for ``and'', ``or'' and ``negation'',  circuits of LTFs provide a model of computation. The minimal number of LTFs or circuit depth necessary to represent a given function $f$ then defines a complexity measure for $f$. 

Note that ridge functions, by construction, depend on $\mathbf{v}$ linearly. Thus, ridge functions do not necessarily include PTFs (cf. \cref{lem:sim PTF}).
We now extend \cref{lem:rbm ridge} to functions with higher-order $\mathbf{v}$ dependency. This, however, requires spin systems with local interactions of higher Fourier degree.  
Specifically, in the following, we prove that shallow $k$-local systems can efficiently simulate functions that can be decomposed as an RBM acting on an ``effective'' configuration of degree $k-1$ polynomials in $\mathbf{v}$.
\begin{lemma}[shallow simulation RBM decomposition]\label{lem:RBM decomp}
    Let $f$ be any target function such that 
    \begin{equation}\label{eq:RBM decomp assumption}
        f(\mathbf{v})= g\left(P_1(\mathbf{v}), \ldots, P_m(\mathbf{v}) \right),
    \end{equation}
    where $m=\mathrm{poly}(n)$ and $P_I$ are degree $\leq k-1$ polynomials (each involving at most $\mathrm{poly}(n)$ monomials) with values in $\mathbf{2}$ and $g$ can be efficiently simulated by an RBM then $f$ can be efficiently simulated by a shallow $k$-local system.
\end{lemma}
\begin{proof}
By assumption there exists an RBM $R$ such that
\begin{equation}
    g(\mathbf{u}) = \min_{\mathbf{h}}H_R(\mathbf{u},\mathbf{h}) - \Gamma = \min_{\mathbf{h}}\bigl(\sum_j \sum_i W_{i,j}u_ih_j + \sum_ib_i u_i + \sum_jc_jh_j \bigr) - \Gamma.
\end{equation}
Inserting 
\begin{equation}
    u_i = P_i(\mathbf{v}) = \sum_{I_i}\lambda_{I_i}\prod_{j\in I_i}v_j
\end{equation}
and regrouping terms yields the minimum over hidden configurations $\mathbf{h}$ of a degree $\leq k$ polynomial in $v_i,h_j$, such that each monomial involves at most one hidden spin $h_j$. This is precisely the effective energy of a shallow $k$-local spin system $S$ and by \cref{eq:RBM decomp assumption}, $S$ simulates $f$. 

The cut-off of $S \simul f$ is at least as large as the cut-off of $R \simul g$: By definition the cut-off of $R \simul g$ is the minimal energy gap $H_R(\mathbf{u},h)-H_R(\mathbf{u}, h_{\min}(\mathbf{u}))$, minimized over all $\mathbf{u}$ and all incorrect hidden configurations $\mathbf{h}\neq \mathbf{h}(\mathbf{u})$. In contrast, the cut-off of $S \simul f$ is obtained similarly, but the minimizing does not involve all configurations $\mathbf{u}$, but only those that are reached by $(P_1(\mathbf{v}), \ldots, P_m(\mathbf{v}))$. Decreasing the scope of the minimization can only increase the cut-off.

Since $m=\mathrm{poly}(n)$ and the RBM simulation $R \simul g$ is efficient, also the simulation  $S \simul f$ is efficient.

Finally, note that the weights of $S$ are products of weights of $R$ and coefficients of $P_i$. If both are polynomially bounded or have polynomial precision then the same applies to the weights of $S$. 
\end{proof}

\begin{theorem}[shallow, $k$-local $\mathrm{IP}_{k-1}$ simulation]\label{lem:IP k sim}
    The $(k-1)$-party inner product modulo $2$ function $\mathrm{IP}_{k-1}$ can be efficiently simulated by a shallow, $k$-local system with small weights. 
\end{theorem}
Since by \cref{cor:rbm ip}, $\mathrm{IP}_{k-1}$ cannot be efficiently simulated by shallow $(k-1)$-local systems, for each $k$ $\mathrm{IP}_{k-1}$ certifies the separation between shallow $(k-1)$-local and shallow $k$-local systems in terms of their efficient reach (\cref{fig:efficient-reach}).
\begin{proof}
    Let $\mathbf{v} = (\mathbf{v}^{(1)}\ldots, \mathbf{v}^{(k-1)})$. W.l.o.g.\ we consider states $v_i \in \{0,1\}$ (see \cref{ssec:bool basics}), then
    \begin{equation}
        \mathrm{IP}_{k-1}(\mathbf{v}^{(1)}, \ldots, \mathbf{v}^{(k-1)}) = (-1)^{\sum_{i=1}^n \prod_{l=1}^{k-1} v^{(l)}_i \mod 2}.
    \end{equation}
    
    Define 
    \begin{equation}
        P_i(\mathbf{v}) = \prod_{l=1}^{k-1} v^{(l)}_i,
    \end{equation}
    and $g(\mathbf{w})= (-1)^{\sum_i w_i \mod 2}$. Then 
    \begin{equation}
        \mathrm{IP}_{k-1}(\mathbf{v}) = g\bigl(P_1(\mathbf{v}), \ldots, P_n(\mathbf{v}) \bigr).
    \end{equation}
    Note that $P_i$ has degree $k-1$. By \cref{lem:RBM decomp} constructing the required simulation of $\mathrm{IP}_{k-1}$ thus amounts to proving that $g$ can be efficiently simulated by an RBM with polynomially bounded weights of polynomial precision. Finally, since $g$ is a distance function, this follows from \cref{lem:rbm ridge}.
\end{proof}
\begin{example}[3RBM simulation of $\mathrm{IP}_2$]\label{rem:3RBM IP}
We work with states from $\{0,1\}$.
Let $\mathbf{x},\mathbf{y}\in\{0,1\}^n$ and write
$t=\sum_{i=1}^n x_i y_i$, so that
$\mathrm{IP}_2(\mathbf{x},\mathbf{y})=(-1)^t$.
Following \cref{lem:IP k sim}, we apply the ridge-function
construction to parity and substitute $u_i=x_i y_i$.

For a desired cut-off $\Delta>0$, choose
\begin{equation}
 D=\max\{1,\lceil\Delta\rceil\},\qquad
 L=4n+D,\qquad B=Ln+D,
 \qquad
 a_j=2L+8(j-1)(-1)^{j+1}.
\end{equation}
Introduce $2n$ hidden spins
$\mathbf{h},\mathbf{s}\in\{0,1\}^n$ and define
\begin{equation}
\begin{split}
 H(\mathbf{x},\mathbf{y},\mathbf{h},\mathbf{s})
 ={}&(B-2)t
 -\sum_{j=1}^n a_jh_j\left(t-j+\tfrac12\right)\\
 &-\sum_{i=1}^n s_i
 \left[B-Lt-(B+D)(1-x_i y_i)\right].
\end{split}
\end{equation}
Expanding $t$ shows that every interaction contains at most
three spins and at most one hidden spin.
The visible--visible terms can be replaced by three-body
interactions with one additional hidden spin, pinned to $+1$
by a sufficiently strong field.
With this additional modification, $H$ is a
3-body RBM Hamiltonian.

To verify the simulation, write $[z]_+=\max\{0,z\}$.
For integer $t\in\{0,\ldots,n\}$, the parity specialization
of the ridge construction gives
\begin{equation}
 \sum_{j=1}^n a_j[t-j+\tfrac12]_+
 =Lt^2-(-1)^t+1-2t.
\end{equation}
Moreover, $B-Lt-(B+D)(1-x_i y_i)$ is positive precisely
when $x_i y_i=1$, so its positive parts sum to $Bt-Lt^2$.
Minimizing each hidden spin independently therefore yields
\begin{equation}
 \min_{\mathbf{h},\mathbf{s}}
 H(\mathbf{x},\mathbf{y},\mathbf{h},\mathbf{s})
 =(-1)^t-1
 =\mathrm{IP}_2(\mathbf{x},\mathbf{y})-1.
\end{equation}
The simulation has energy shift $\Gamma=-1$.
Every hidden-spin coefficient has absolute value at least
$D$: this follows from $a_j\geq2D$, the half-integer
breakpoints, and the two cases $x_i y_i=0,1$.
Hence the minimizing hidden configuration is unique and
the cut-off is at least $\Delta$.

The construction uses $2n$ visible spins, $2n+1$ hidden spins,
and $\mathcal{O}(n^2)$ interactions. Its weights are integers
of magnitude $\mathcal{O}(n^2+nD)$, and thus are small whenever
$\Delta$ is polynomially bounded.
    
\end{example}

\section{Conclusion and outlook}\label{sec:conclusion and outlook}

Let us summarize our results (\cref{ssec:summary}) and look beyond (\cref{ssec:outlook}). 

\subsection{Summary of results}\label{ssec:summary}

\paragraph{Characterization of maximal models.} 
One of the central notions of this work is that of visible closure (\cref{def:visible closed}; \cref{fig:vis subset closed}), which can be thought of as a more nuanced notion of depth. 
More precisely, visible closure captures the compositional properties of deep models. It allows visible spins to become hidden (up to simulation), and  thereby preserves the visible-hidden split when composing simulations  (\cref{lem:visible closed}). 
Note that when visible spins become hidden this effectively introduces new hidden-hidden interactions. 
We thus conjecture that most (non-trivial) deep models are visible closed, and that the converse also holds. The following separation results also support this conjecture. 
The paradigmatic deep models, DBMs, are visible closed in a straightforward manner (\cref{fig:DBM part}); technically, they are strictly visible closed. 

Now, visible closure characterizes what a spin ML model needs to go from universality to maximality: a universal spin ML model is maximal if and only if it is visible closed. 
More generally, for spin ML models any two of the three properties ``universal'', ``maximal'' and ``visible closed'' imply the third (\cref{thm:main}).  
This can be interpreted as follows: if a spin ML model has a sufficient amount of depth (visible closure) then any statements about its unconstrained representation capabilities (universality) also apply to its efficient representation capabilities (maximality), and vice versa. 

Overall, since DBMs are universal and visible closed, it follows that they are  maximal (\cref{cor:DBM max}).

\paragraph{Separation between shallow and maximal models.} 
Could it happen that a shallow model is, for some reason, visible closed and thus powerful enough to be maximal? The answer is no (under reasonable assumptions), because there are some functions that all maximal models can efficiently represent but shallow models can't (\cref{tab:complexity separation}). 

Specifically, maximal models can efficiently simulate certain $\NPclass$-hard functions (\cref{thm:separation 1}), approximate certain $\SharpPclass$-hard functions in their unnormalized visible spin marginals (\cref{thm:separation 2}) and approximate the corresponding $\SharpPclass$-hard probability distribution in their normalized visible spin marginal (\cref{thm:separation 3}). In contrast, shallow models can do none of the above, under reasonable assumptions on computational complexity. 
This provides an explicit separation between RBMs and DBMs. 
In particular, RBMs are not maximal. 

So, these separations prove that maximal models are necessarily deep. 
However, this statement is contingent on certain computational complexity assumptions, in contrast to the results involving visible closure, which stand on their own (cf.\ conjecture above).

\paragraph{Separation hierarchy for shallow models.}

Is there anything else we can say concerning the gap (in efficient representation capabilities) between RBMs and DBMs? Namely, is there a sharp transition, or are there some internal levels that bridge between the two? 
The latter is the answer, if we consider RBMs with increasing interaction order. These models are not often considered in the literature. In this work, we find that RBMs with $(k+1)$-body interactions can efficiently represent more functions than RBMs with $k$-body interactions, for any $k$ (\cref{fig:polynomial-reach}). They thus build an infinite separation hierarchy which does not saturate at maximality but underneath. 
Recall that, without the efficiency constraint, all models considered here are equally powerful, namely universal!  

The separation between $(k+1)$- and $k$-body RBMs is witnessed by some functions with large $k$-party communication complexity, such as the $k$-party inner product modulo 2 (\cref{lem:IP k sim,thm:shallow rand,cor:rbm ip}). It  relies on the assumption of small weights (\cref{def:poly bounded weights}). 

In summary, while the separation between sub-maximal and maximal models is controlled by depth, within the space of shallow models the efficient simulation capabilities are determined by the interaction order. It follows that $3$-body RBMs can efficiently simulate certain targets that conventional RBMs can't (\cref{rem:3RBM IP}).

\subsection{Outlook}\label{ssec:outlook}

\paragraph{How much do our results depend on the assumptions?}
This question must be considered for four separate themes. 
First, we have assumed small bit complexity of all weights and simulation parameters to prove the separations between deep and shallow models in \cref{sec:consequences}. 
And we have assumed small weights (\cref{ssec:small weights}) to prove the separation hierarchy of \cref{sec:RBM separation}. 
To what extent are these assumptions justified? 
Do they apply to common examples of simulations/approximations?

Second, we have used common computational complexity assumptions such as the non-collapse of the polynomial hierarchy (\cref{ssec:comp compl assumptions}). 
This is widely believed to be true, and can be considered a generalization of the standard assumption that $\Pclass \neq \NPclass$. 
More generally, most statements in computational complexity are in fact conditional on similar conjectures separating complexity classes. 

Third, do the separation results for simulations extend to approximations?
Several of our separation results are statements about existence and non-existence of simulations (specifically \cref{thm:separation 1} and the entirety of \cref{sec:RBM separation}). 
The translation of separations from simulations to approximations can fail for two reasons. 
The first one is that if a certain simulation does not exist, the corresponding approximation may still exist, as far as we know. 
We partially address this question in \cref{rem:sim approx}, showing that approximations lead to simulations in some cases. 
The second reason has to do with the cut-off. 
We have argued that if simulations with sufficiently large cut-off exist, then the corresponding approximations with small error exist too. 
Note that \cref{thm:separation 1} only proves existence of simulations for a fixed cut-off of $\Delta=1$. 

Finally, what is the relevance of our asymptotic analysis in practice? 
Our notion of efficiency requires polynomial overhead, but we ignore which specific polynomial is in play. But, in practice, this may matter, and models within one level of the landscape of \cref{fig:polynomial-reach} may strongly differ in the polynomial overhead they require to represent a target. 
We study some of these differences with our learning experiments in \cref{sec:ip-learning-methods}. 
We consider physics-motivated targets with high order interactions which we know, from the theory developed in this paper, admit an efficient representation by both RBMs and $3$-body RBMs. Yet, the numerical experiments show an advantage for $3$-body RBMs. 
It would be interesting to develop a finer analysis that quantifies over such polynomial budgets.

\paragraph{Can one quantify the trade-offs of resources for efficient reach?}
Our results suggest that representational capabilities depend on the
interplay between several resources: 
the overhead (i.e.\ number of weights), depth, interaction order, and the magnitude and precision of weights. But we would like to disentangle their role.
On the one hand, we have seen that increasing depth can compensate for restrictions on interaction order: deep models with pair interactions (e.g.\ DBMs) have maximal efficient reach, while shallow models exhibit a strict hierarchy as interaction order increases. 
On the other hand, certain targets that require an exponential overhead in shallow models or models of small interaction order can be efficiently represented by deep ones or such with higher-order interactions. 
Yet, the role of weights is less clear. Especially, our small weight assumption raises the question of whether larger weights or greater numerical precision can overcome some of the limitations of shallow models.
Can these observations be developed into a quantitative account of
which resources can compensate for one another, and at what cost?

\paragraph{What do our separations imply for learning?}
One very important aspect we have not considered throughout the main text is learning (but see \cref{sec:ip-learning-methods}). Do our separations for the existence of efficient representations transfer to separations for efficiently learning them? 
Broadly, this transfer may fail for two reasons. 

First, can the weights of efficient representations be learned from finite training data? 
We would need that a representation that witnesses a separation, e.g.\ between RBMs and DBMs, can be learnt, so that the separations are accessible through training.  
We provide some insights into this question in \cref{sec:ip-learning-methods}, where we numerically study the learning of RBMs and $3$-body RBMs for the inner product modulo $2$.
For the tested system sizes, parameter budgets and training procedure, we observe a learning advantage for $3$-body RBMs. This suggests that at least some of our separations are accessible through training.

A second issue is: does efficient representability imply efficient learnability? Even if efficient representations can be learned, the learning itself may be inefficient, because of the following. 
Efficient representations have at most polynomially many weights, which is necessary for efficient learning. Yet, learning typically introduces other bottlenecks, e.g.\ the amount of training data required or the computational cost of sampling from the model or minimizing the objective. For which models are these aspects also guaranteed to be efficient?

On a related note, we would like to understand what universal approximation and our efficiency results imply for practical learning, specifically where models must generalize beyond the training data. In particular, when does greater (efficient) representational power help capture underlying structure, and when does it instead facilitate overfitting?

\paragraph{Where do practical learning targets lie within our hierarchy?}
Our results leave open which levels of the efficient reach landscape (see \cref{fig:polynomial-reach}) are occupied in practice. 
Conversely, given a target, or a data sample thereof, we do not know how to determine where it sits in the landscape.   
For which targets do deep models, or models with higher-order interactions, work actually better? This question is also studied in \cite{Mo17}. 
In \cref{sec:ip-learning-methods} we also partially address this question for RBMs and $3$-body RBMs for certain physics-motivated targets.  

\paragraph{Do other notions of universality show similar behavior under efficiency constraints?} 
This investigation started from the realization that universal spin models are intimately related to the universal approximation theorem in some machine learning models. 
But there are other notions of universality, which obey a common structure \cite{Go24,Go24b}. 
For spin models and machine learning models, the additional requirement of efficiency has given rise to a split in their capabilities: for spin models universality implies maximality, whereas for
machine learning models, universal approximation is compatible with
strictly different efficient reaches. 
What would happen if we applied the efficiency requirement to the other notions of universality? Would it also give rise to a separation of capabilities? 
A first, much more concrete step in this direction would be to study other types of machine learning models such as feedforward neural networks. For the latter it is known that one hidden layer suffices for universal approximation, but depth can yield substantial representational advantages \cite{Te16}. 
More generally, which structural features determine whether efficiency constraints preserve a universality classification or refine it into a hierarchy?
The framework for universality \cite{Go24} may allow one to study these questions in a broad manner.

\subsection{Acknowledgments and disclosures}
\paragraph{Acknowledgments.}
This research was funded in part by the Start Prize of the Austrian Science Fund 10.55776/Y1261. For open access purposes, the authors have applied a CC BY public copyright license to any author accepted manuscript version arising from this submission.

\paragraph{Use of artificial intelligence.}
The authors used OpenAI's GPT-6 and GPT-5.6 as research
assistants to support literature searches, the development
and checking of mathematical arguments, and manuscript
preparation. The authors reviewed and revised the
AI-assisted material and take full responsibility for
the content of this work.

\appendix

\section{Complexity theory toolbox}\label{sec:comp compl}
In this section we provide an introduction to computational complexity theory (\cref{ssec:comp compl basics}), detailed proofs of the complexity of the hard reference model $\mathcal{M}_{\mathrm{Ising}}$ introduced in \cref{ssec:hard reference model} (\cref{ssec: complexity hard reference model}), as well as an introduction to the theory of pseudo-Boolean functions (\cref{ssec:bool basics}) and to communication complexity (\cref{ssec:comm compl}).

\subsection{Computational complexity basics}\label{ssec:comp compl basics}
Computational complexity theory studies the computational cost (e.g.\ runtime or memory/space) of functions \cite{Ab09, Mo11}. One usually considers functions that are defined for inputs of arbitrary length, e.g.\ $f\colon \bigcup_{n\geq 0} \mathbf{2}^n \to  \bigcup_{n\geq 0} \mathbf{2}^n$, or equivalently, that map strings over some input alphabet to strings over some output alphabet, such that both inputs and outputs can have arbitrary length. 
Functions $f$ that for all inputs yield an output in $\mathbf{2}$ play a special role in computational complexity. Such functions are completely defined by their set of \emph{yes-instances}, the set of inputs $x$ with $f(x)=+1$. They are called \emph{decision problems} since computing $f(x)$ amounts to deciding whether $x$ is a yes-instance. In practice, the set of yes-instances is usually defined by some property and computing $f(x)$ amounts to deciding whether $x$ has that property.  

Instead of considering the computational cost of each input, among other things computational complexity focuses on the asymptotic scaling of the worst-case runtime with the input size,  that is, 
\begin{enumerate}
    \item The \emph{computational complexity} of an algorithm is the asymptotic scaling $n \to \infty$ of the worst case computational cost for inputs of size/length $n$.
    \item The \emph{computational complexity} of a function is the computational complexity of the best algorithm for it.
\end{enumerate}
Note that above, the concepts \emph{algorithm}, as well as \emph{computational cost}, i.e.\ \emph{runtime} and \emph{space}, are defined w.r.t.\ a model of computation, usually that of Turing machines.

Grouping functions that asymptotically have similar computational cost defines so-called \emph{complexity classes}. In this work we only consider the time-complexity of functions (see \cref{sec:consequences}). 

\begin{definition}[useful complexity classes]\label{def:complexity classes}
The following is a list of the  most relevant complexity classes (for this work).
\begin{enumerate}
    \item $\Pclass$ $(\FPclass)$: the class of decision problems (functions) that can be computed with a runtime that increases at most polynomially with the input length\footnote{Recall that decision problems are simply functions with output in $\mathbf{2}$.}.
    \item $\NPclass$: the class of decision problems for which yes-instances can be verified in $\Pclass$, that is, each decision problem $f\in \NPclass$ has a polynomial-time verifier, namely an algorithm $R\in \Pclass$ such that $x$ is a yes-instance of $f$ iff there exists a certificate $y$ of length at most $\mathrm{poly}(n)$ such that $R(x,y)=1$. 
    \item $\CoNPclass$: the class of decision problems for which no-instances can be verified in $\Pclass$, $f$ is in $\CoNPclass$ if swapping yes- and no-instances yields a problem in $\NPclass$.
    \item $\CoAclass$, for any complexity class $\Aclass$ of decision problems: the class of decision problems for which swapping yes- and no-instances yields a problem in $\Aclass$.
    \item $\SharpPclass$: the class of counting problems (functions with outputs in $\mathbb{N}$) that count the number of certificates of $\NPclass$-problems. That is, for each $f \in \SharpPclass$ there exists some related $\NPclass$-problem $g$ and verifier $R$ such that $f(x)$ is the number of certificates $y$ of $x$ w.r.t.\ $R$.
    \item $\Pclass/\mathrm{poly}$ $(\FPclass/\mathrm{poly})$: the class of decision problems (functions) that can be computed in $\Pclass$ \textup{($\FPclass$)}, when provided with advice strings $a_n$ of polynomial length. For $f\in \Pclass/\mathrm{poly}$ ($f \in \FPclass/\mathrm{poly}$) there exists a sequence of advice strings $(a_n)_{n\in \mathbb{N}}$, such that $a_n$ is of length $\mathrm{poly}(n)$ and given any input of length $n$ and the corresponding advice string $a_n$, $f(x)$ can be computed in $\Pclass$ \textup{($\FPclass$)}. Importantly, the advice strings $a_n$ need not be computable in any way. 
    The advice strings can e.g.\ encode Boolean circuits of polynomial size. Hence, $\Pclass/\mathrm{poly}$ can equivalently be defined as the complexity class of functions that can be computed by polynomially sized Boolean circuits.
    \item $\Aclass/\mathrm{poly}$ for any complexity class $\Aclass$: the class of decision problems or functions that can be computed in $\Aclass$, when provided with advice strings $a_n$ of polynomial length (cf. $\Pclass/\mathrm{poly}$). 
    \item $\Aclass^{\Bclass}$, for any two complexity classes, $\Aclass, \Bclass$: the class of decision problems or functions that can be computed with complexity $\Aclass$, assuming the algorithm has access to an oracle for any function from $\Bclass$ (see \cref{def:oracle}).
\end{enumerate}
\end{definition}

In computational complexity theory, the computational hardness of functions is studied in a relational way; $f$ is at least as hard as $g$, $(f\geq g)$ if an efficient algorithm for $f$ would allow one to construct an efficient algorithm for $g$, where `efficient' here usually means within some fixed complexity class $\Aclass$. The process of constructing the $g$-algorithm by using the $f$-algorithm as subroutine is called \emph{reducing} $g$ to $f$ or constructing a \emph{reduction} from $g$ to $f$.
There exist several notions of reductions, that mainly differ in  how exactly the $g$-algorithm is allowed to use the $f$-subroutine.
In this work, we mainly consider the most general such notion, that of an oracle reduction or Turing reduction.
\begin{definition}[reduction]\label{def:oracle}
    A \emph{polynomial time oracle reduction} from $g$ to $f$ is a polynomial-time algorithm for $g$ that has access to an oracle for $f$, i.e.\ a hypothetical subroutine that on any input $x$ can compute $f(x)$ in $\mathcal{O}(1)$, i.e.\ instantly.
\end{definition}
Note that oracle reductions might use polynomially many oracle calls and  might post-process the oracle answers.
For stricter notions of reductions this is prohibited:  \emph{many-to-one reductions} (or Karp reductions) are oracle reductions that  use a single oracle call and return the oracle answer as final output, without post-processing it. Many-to-one reductions are commonly used for  decision problems, only.
\begin{definition}[hardness and completeness]\label{def:completeness}
    A function $f$ is called \emph{hard} w.r.t.\ complexity class $\Aclass$ (or $\Aclass$-hard) if every problem in $\Aclass$ can be reduced to $f$. It is called \emph{complete}  w.r.t.\ complexity class $\Aclass$ (or $\Aclass$-complete) if it is $\Aclass$-hard and $f \in \Aclass$.
\end{definition}
Note that, depending on the complexity class $\Aclass$, it might be necessary to impose  restrictions on the reduction to ensure that, relative to $\Aclass$, the overhead introduced by the reduction is negligible. 
In this work, hardness results for decision problems (e.g.\ \cref{lem:Ising gse}) are proven w.r.t.\ many-to-one reductions, while hardness results for function problems (including $\SharpPclass$-hardness) are proven w.r.t.\ oracle reductions.

Finally, separation results between shallow and (v,p)-maximal models in \cref{sec:consequences} are derived based on the assumption that there is no collapse of the \emph{polynomial hierarchy} \cite{St76}.   The polynomial hierarchy is an infinite  hierarchy of complexity classes that generalize the classes $\Pclass$, $\NPclass$ and $\CoNPclass$. That there is no collapse of the polynomial hierarchy can be seen as a generalization of the widely believed conjecture that $\Pclass\neq \NPclass$.

More precisely, the $0$th level of the polynomial hierarchy is the complexity class $\Pclass$. Define 
\begin{equation}
    \Delta_0 = \Sigma_0 = \Pi_0 = \Pclass.
\end{equation}
Further levels of the polynomial hierarchy are defined iteratively, by oracle complexity classes 
\begin{equation}
    \Delta_{k+1} = \Pclass^{\Sigma_k},  \qquad \Sigma_{k+1} = \NPclass^{\Sigma_k}, \qquad \Pi_{k+1}=\CoNPclass^{\Sigma_k}.
\end{equation}
The polynomial hierarchy then is the union of all these classes, for all $k$. 

Note that due to relations amongst the above complexity classes, the polynomial hierarchy can be defined as 
\begin{equation}
    \PH = \bigcup_{k \geq 0} \Sigma_k = \bigcup_{k\geq 0} \Pi_k. 
\end{equation}
A collapse of the polynomial hierarchy to level $k$ then means that for any $k' > k$, 
\begin{equation}
    \Sigma_{k'}, \Pi_{k'}, \Delta_{k'} \subseteq \Sigma_k,
\end{equation}
and thus $\PH=\Sigma_k$.
Note that a collapse to the $0$th level  is equivalent to $\Pclass=\NPclass$.

\subsection{Computational complexity of the reference model}\label{ssec: complexity hard reference model}
In this subsection we prove that, for $\mathcal{M}_{\mathrm{Ising}}$, computing the ground-state energies (i.e.\ the function $\mathcal{G}(\mathbf{m})$) is $\NPclass$-hard, and computing partition functions and partition function ratios both is $\SharpPclass$-hard. We further prove that approximating partition functions  and partition function ratios to sufficient precision is  $\SharpPclass$-hard.

\begin{lemma}[complexity of $\mathcal{G}(\mathbf{m})$]\label{lem:Ising gse}
    Computing ground-state energies $\mathcal{G}(\mathbf{m})$ of $\mathcal{M}_{\mathrm{Ising}}$ is $\NPclass$-hard.
\end{lemma}
\begin{proof}
    This is a well established result (see e.g.\ \cite{Ba82,Is00}). We nevertheless provide an explicit proof. Specifically, we prove that an oracle for ground-state energies $\mathcal{G}(\mathbf{m})$ of $\mathcal{M}_{\mathrm{Ising}}$ would yield a polytime algorithm for the maximum graph cut decision problem $\textsc{MaxCut}$, which is $\NPclass$-complete \cite{Ga76}. 
    
    Recall that given a graph $G=(V,E)$\footnote{Recall that $G\equiv \mathbf{m}$.}, a cut is a partition of vertices $V$ into two disjoint sets $V=V_+\cup V_-$. The size of a cut  is the number of edges $\vert E_{+-} \vert$ that connect a vertex from $V_+$ with one from $V_-$. Given a graph $G$ and an integer $k$, $\textsc{MaxCut}$ then asks if the graph $G$ has a cut of size at least $k$.

    Consider now the Ising system $S_G$ (\cref{eq:Ising ham reference model}) corresponding to graph $G$. Any configuration $\mathbf{v}$ specifies a graph cut of $G$, into $V_+(\mathbf{v})$, spins in state $+1$, and $V_-(\mathbf{v})$, spins in state $-1$. The energy of $\mathbf{v}$  can be expressed as 
    \begin{equation}
        H_{S_{G}}(\mathbf{v}) = \vert E \vert -2 \vert E_{+-}(\mathbf{v}) \vert. 
    \end{equation}
    Thus, the maximal cut size of $G$ can be computed from the minimal energy of $S_{G}$, which proves that computing minimal energies of $\mathcal{M}_{\mathrm{Ising}}$ is $\NPclass$-hard.
    
    Note that considering the two respective decision problems the reduction is a many-to-one reduction (cf.\ \cref{def:oracle}). This implies that the ground-state energy decision problem of $\mathcal{M}_{\mathrm{Ising}}$ is $\NPclass$-complete under many-to-one reductions. 
\end{proof}

\begin{lemma}[complexity of approximating $\mathcal{Z}(\mathbf{m})$]\label{lem:part fun approx hard}
    Approximating partition functions of $\mathcal{M}_{\mathrm{Ising}}$, $\mathcal{Z}(\mathbf{m})$, to sufficient, additive precision is as hard as computing them exactly. 

    Specifically, there exists an efficiently computable accuracy
$\delta_n=2^{-\mathrm{poly}(n)}$ such that, for every
$n$-spin system $S_{\mathbf m}\in\mathcal{M}_{\mathrm{Ising}}$,
a rational approximation $z$ satisfying
$|z-Z_n(\mathbf m)|\leq\delta_n$ allows polynomial-time
recovery of its energy degeneracy vector $\mathbf a(\mathbf m)$,
and hence its exact partition function in the symbolic form
$e^{-K}\sum_{l=0}^{2K}a_l(\mathbf m)e^l$,
where $K=\binom n2$.
\end{lemma}
The hardness of approximating Ising partition
functions has been studied extensively. Jerrum and Sinclair
\cite{Je93} established that under reasonable assumptions, the general antiferromagnetic Ising partition function
admits no fully polynomial randomized approximation scheme; Goldberg, Jerrum and Paterson
\cite{GoJePa03} obtained similar results for fixed
antiferromagnetic couplings.
Here we require a different statement: for our specific reference
model, approximation to sufficiently high additive precision is
$\SharpPclass$-hard.
\begin{proof}

Denote the restriction of $\mathcal Z$ to $n$-spin systems by $Z_n$. Let $G$, encoded by $\mathbf m$, be an input graph, and suppose that we are given a rational approximation $z$ satisfying
\begin{equation}
\left|Z_n(\mathbf m)-z\right|\leq\delta_n.   
\end{equation}
At $\beta=1$, the partition function is uniquely determined by the integer vector $\mathbf a(\mathbf m)$ that records the degeneracy of each energy level. A quantitative version of the transcendence of $e$ implies that, for a sufficiently small inverse-exponential error $\delta_n$, there is a unique bounded integer vector $\mathbf a$ whose corresponding partition-function value is compatible with $z$. 
Recovering the exact partition function amounts to recovering $\mathbf a(\mathbf m)$ from $z$. We formulate this as a bounded integer-relation problem and, following the standard idea of reducing closest-vector problems to shortest-vector problems, embed it into a higher-dimensional integer lattice (see \cite{Ka87} and \cite[Chapter 18]{Ga12}). The lattice is constructed using sufficiently accurate rational approximations to the powers of $e$ and to $z$; this does not change the partition function itself, which remains defined at $\beta=1$. We choose the scaling of the lattice such that a vector encoding the correct degeneracy vector is short, whereas every vector encoding an incorrect bounded vector, or having coefficients outside the prescribed bound, is longer than the approximation guarantee of the LLL algorithm. Applying LLL \cite{Le82} therefore produces a vector from which $\mathbf a(\mathbf m)$, and hence the exact partition function $Z_n(\mathbf m)$, can be recovered in polynomial time.

\paragraph{Encoding $Z_n(\mathbf{m})$ by a bounded integer vector:}

Recall that, at inverse temperature $\beta=1$,
\begin{equation}
    Z_n(\mathbf{m})
    =
    \sum_{\mathbf{v}}
    e^{-H_{S_{\mathbf{m}}}(\mathbf{v})}.
\end{equation}
Every system $S_{\mathbf{m}}$ has at most $K$ interactions, each contributing either $-1$ or $+1$. Consequently, its Hamiltonian takes integer values in the interval $[-K,K]$, where $K \coloneqq \binom{n}{2}$. Define its degeneracy vector $\mathbf{a}(\mathbf{m})$ by  components $a_l(\mathbf{m})$ being the number of configurations of $S_{\mathbf{m}}$ with energy $K-l$, for $0\leq l\leq 2K$, then 
\begin{equation}
  Z_n(\mathbf{m}) = F(\mathbf{a}(\mathbf{m}))\coloneqq 
    e^{-K}
    \sum_{l=0}^{2K}
    a_l(\mathbf{m})e^l.
\end{equation}
Thus, recovering the exact partition function $Z_n(\mathbf{m})$ from an approximation $z$ amounts to recovering the bounded integer vector $\mathbf{a}(\mathbf{m})$. Note that $\mathbf{a}(\mathbf{m})$ has $2K+1$ non-negative integer components that sum to $2^n$.

We first observe that the map $F$ is injective on integer vectors. Suppose that for some $\mathbf{a}\in\mathbb{Z}^{2K+1}$,
\begin{equation}
    0
    =
    F(\mathbf{a})
    =
    e^{-K}
    \sum_{l=0}^{2K}a_l e^l.
\end{equation}
Since $e^{-K}\neq 0$, the number $e$ is then a root of the integer polynomial
\begin{equation}
    P(x)
    =
    \sum_{l=0}^{2K}a_l x^l.
\end{equation}
Because $e$ is transcendental \cite{Ci74}, this is possible only if $P$ is the zero polynomial. Therefore,
\begin{equation}
    \mathbf{a}=\mathbf{0}.
\end{equation}
It follows that distinct integer vectors give distinct partition-function values.

\paragraph{For sufficiently small $\delta_n$, the degeneracy vector is uniquely determined by $z$:}

For integer vectors with arbitrarily large coefficients, the difference
$\vert F(\mathbf{a})-F(\mathbf{b})\vert$ can be arbitrarily small.
The degeneracy vectors relevant here, however, have bounded components of size at most $2^n$. 
We now use a transcendence measure of $e$ to obtain a lower bound on
the distance between distinct values $F(\mathbf{a})$ and
$F(\mathbf{b})$ when the components of $\mathbf{a}$ and $\mathbf{b}$
are bounded.

Let $\mathbf{a},\mathbf{b}\in\mathbb{Z}^{2K+1}$ be distinct vectors
whose components have magnitude at most $B$. By linearity of $F$,
\begin{equation}
    \left\vert F(\mathbf{a})-F(\mathbf{b})\right\vert
    =
    \left\vert F(\mathbf{a}-\mathbf{b})\right\vert
    =
    e^{-K}\left\vert P_{\mathbf{a},\mathbf{b}}(e)\right\vert,
\end{equation}
where
\begin{equation}
    P_{\mathbf{a},\mathbf{b}}(x)
    \coloneqq
    \sum_{l=0}^{2K}(a_l-b_l)x^l
\end{equation}
is a nonzero integer polynomial of degree at most $2K$ and with
coefficients of magnitude at most $2B$.

The transcendence of $e$ can be quantified by a transcendence measure.
In particular, \cite[Theorem~1]{Ci74} implies that there exists an
effectively computable constant $C>0$ such that every nonzero integer
polynomial $q$ of degree at most $d$ and coefficient height at most
$h$ satisfies
\begin{equation}
    \left\vert q(e)\right\vert
    >
    \exp\left(
        -Cd^2\bigl(d+\log h\bigr)
    \right).
\end{equation}
For $H\geq2$, define
\begin{equation}
    \epsilon(K,H)
    \coloneqq
    \exp\left(
        -K-C(2K)^2\bigl(2K+\log H\bigr)
    \right).
\end{equation}
It follows that
\begin{equation}\label{eq:polynomial bound}
    \left\vert F(\mathbf{a})-F(\mathbf{b})\right\vert
    \geq
    \epsilon(K,2B).
\end{equation}

In particular, two distinct degeneracy vectors have components bounded
by $2^n$ and hence satisfy
\begin{equation}
    \left\vert F(\mathbf{a})-F(\mathbf{b})\right\vert
    \geq
    \epsilon\left(K,2^{n+1}\right).
\end{equation}
If both $F(\mathbf{a})$ and $F(\mathbf{b})$ were within $\delta_n$ of
the same approximation $z$, then the triangle inequality would give
\begin{equation}
    \left\vert F(\mathbf{a})-F(\mathbf{b})\right\vert
    \leq 2\delta_n.
\end{equation}
Therefore, choosing
\begin{equation}
    \delta_n
    <
    \frac{1}{2}\epsilon\left(K,2^{n+1}\right)
\end{equation}
ensures that at most one valid degeneracy vector is compatible with
$z$. The more restrictive choice of $\delta_n$ required by the LLL
construction below automatically satisfies this condition.

\paragraph{Embedding the bounded recovery problem into an SVP instance:}

We now construct an integer lattice $\Lambda(\mathbf{L})$ such that,
given $z$, recovering the correct degeneracy vector
$\mathbf{a}(\mathbf{m})$ amounts to finding a sufficiently short vector
in $\Lambda(\mathbf{L})$.

Set $D\coloneqq 2K+2$.
Let $A$ be a positive integer whose value will be fixed below. For
$0\leq l\leq 2K$, choose integers $q_l$ and $q_z$ satisfying
\begin{equation}
    \left\vert q_l-Ae^{-K+l}\right\vert\leq 1,
    \qquad
    \left\vert q_z-Az\right\vert\leq 1.
\end{equation}
These integers can be computed in polynomial time by approximating
the corresponding real numbers to sufficiently many binary digits;
the required precision is polynomial in the bit length of $A$ and in
$K$.

The lattice $\Lambda(\mathbf{L})$ has rank $D$ and is embedded in
$\mathbb{R}^{D+1}$. It is spanned by the rows of the
integer matrix
\begin{equation}\label{eq:lattice basis}
    \mathbf{L}
    =
    \begin{pmatrix}
        1 & 0 & \cdots & 0 & 0 & q_0 \\
        0 & 1 & \cdots & 0 & 0 & q_1 \\
        \vdots & \vdots & \ddots & \vdots & \vdots & \vdots \\
        0 & 0 & \cdots & 1 & 0 & q_{2K} \\
        0 & 0 & \cdots & 0 & 1 & q_z
    \end{pmatrix}.
\end{equation}
Every lattice vector $\mathbf{v}\in\Lambda(\mathbf{L})$ is an integer
linear combination of the rows of $\mathbf{L}$ and hence has the form
\begin{equation}\label{eq:lattice vector}
    \mathbf{v}(\mathbf{a},y)
    =
    \left(
        \mathbf{a},
        y,
        Q(\mathbf{a},y)
    \right),
\end{equation}
where $\mathbf{a}\in\mathbb{Z}^{2K+1}$, $y\in\mathbb{Z}$, and
\begin{equation}
    Q(\mathbf{a},y)
    \coloneqq
    \sum_{l=0}^{2K}a_lq_l+yq_z.
\end{equation}

By the choice of $q_l$ and $q_z$, we can write
\begin{equation}\label{eq:lattice rounding}
    Q(\mathbf{a},y)
    =
    A\bigl(F(\mathbf{a})+yz\bigr)
    +
    R(\mathbf{a},y),
\end{equation}
where the rounding error satisfies
\begin{equation}\label{eq:rounding error}
    \left\vert R(\mathbf{a},y)\right\vert
    \leq
    \lVert\mathbf{a}\rVert_1+\vert y\vert.
\end{equation}

We now characterize the lengths of vectors in $\Lambda(\mathbf{L})$
and argue that sufficiently short vectors encode
$\mathbf{a}(\mathbf{m})$, i.e.\ given $z$, finding the correct degeneracy vector amounts to finding short vectors in $\Lambda(\mathbf{L})$.
Let $B\geq 2^n$ be an integer bound, whose value will also be fixed
below. We distinguish the following cases.

\begin{enumerate}
\item\label{it:case bounded}\textbf{Bounded components:}
Suppose that 
\begin{equation}
    \lVert\mathbf{a}\rVert_\infty\leq B,
    \qquad
    \vert y\vert\leq B.
\end{equation}
Then
\begin{equation}
    \lVert\mathbf{a}\rVert_1+\vert y\vert
    \leq
    (2K+2)B
    =
    DB.
\end{equation}
We consider two subcases.

\begin{enumerate}
\item\label{it:subcase 1 bounded}
Suppose that
\begin{equation}
    \mathbf{a}+y\mathbf{a}(\mathbf{m})=\mathbf{0}.
\end{equation}
For a nonzero lattice vector, this implies $y\neq0$, since $y=0$
would also imply $\mathbf{a}=\mathbf{0}$.
Whenever this is the case, $\mathbf{a}(\mathbf{m})$ can immediately
be recovered from the components of $\mathbf{v}(\mathbf{a},y)$ as
\begin{equation}
    \mathbf{a}(\mathbf{m})=-\frac{\mathbf{a}}{y}.
\end{equation}

By linearity of $F$,
\begin{equation}
\begin{aligned}
    F(\mathbf{a})+yz
    &=
    F(\mathbf{a})
    +
    yF\bigl(\mathbf{a}(\mathbf{m})\bigr)
    +
    y\left(
        z-F\bigl(\mathbf{a}(\mathbf{m})\bigr)
    \right)
    \\
    &=
    y\left(
        z-F\bigl(\mathbf{a}(\mathbf{m})\bigr)
    \right).
\end{aligned}
\end{equation}
Using
\begin{equation}
    \left\vert
        z-F\bigl(\mathbf{a}(\mathbf{m})\bigr)
    \right\vert
    <
    \delta_n,
\end{equation}
together with \cref{eq:lattice rounding,eq:rounding error}, gives
\begin{equation}
    \left\vert Q(\mathbf{a},y)\right\vert
    \leq
    A\vert y\vert\delta_n+DB
    \leq
    AB\delta_n+DB.
\end{equation}
Consequently,
\begin{equation}
    \lVert\mathbf{v}(\mathbf{a},y)\rVert_2
    \leq
    2DB+AB\delta_n.
\end{equation}

This case contains, in particular, the  choice
\begin{equation}
    \mathbf{a}=\mathbf{a}(\mathbf{m}),
    \qquad
    y=-1.
\end{equation}
Since
\begin{equation}
    \lVert\mathbf{a}(\mathbf{m})\rVert_1=2^n,
\end{equation}
the corresponding lattice vector satisfies the stronger bound
\begin{equation}\label{eq:correct vector bound}
    \left\lVert
        \mathbf{v}\bigl(\mathbf{a}(\mathbf{m}),-1\bigr)
    \right\rVert_2
    \leq
    2^{n+1}+2+A\delta_n.
\end{equation}

\item\label{it:subcase 2 bounded}
Suppose instead that
\begin{equation}
    \mathbf{a}+y\mathbf{a}(\mathbf{m})\neq\mathbf{0}.
\end{equation}
Since
\begin{equation}
    \lVert\mathbf{a}\rVert_\infty\leq B,
    \qquad
    \vert y\vert\leq B,
    \qquad
    \left\lVert\mathbf{a}(\mathbf{m})\right\rVert_\infty
    \leq 2^n
    \leq B,
\end{equation}
the components of
$\mathbf{a}+y\mathbf{a}(\mathbf{m})$ have magnitude at most
$B^2+B$. The transcendence estimate used in
\cref{eq:polynomial bound} therefore gives
\begin{equation}
    \left\vert
        F(\mathbf{a})
        +
        yF\bigl(\mathbf{a}(\mathbf{m})\bigr)
    \right\vert
    \geq
    \epsilon\left(K,B^2+B\right).
\end{equation}
Moreover,
\begin{equation}
    \left\vert
        y\left(
            z-F\bigl(\mathbf{a}(\mathbf{m})\bigr)
        \right)
    \right\vert
    \leq
    B\delta_n.
\end{equation}
It follows that
\begin{equation}
    \left\vert F(\mathbf{a})+yz\right\vert
    \geq
    \epsilon\left(K,B^2+B\right)-B\delta_n.
\end{equation}
Using \cref{eq:lattice rounding,eq:rounding error}, we obtain
\begin{equation}\label{eq:incorrect vector bound}
\begin{aligned}
    \lVert\mathbf{v}(\mathbf{a},y)\rVert_2
    &\geq
    \left\vert Q(\mathbf{a},y)\right\vert
    \\
    &\geq
    A\left(
        \epsilon\left(K,B^2+B\right)-B\delta_n
    \right)
    -
    DB.
\end{aligned}
\end{equation}
\end{enumerate}

\item\label{it:case unbounded}\textbf{Unbounded components:}
Finally, suppose that either
\begin{equation}
    \lVert\mathbf{a}\rVert_\infty>B
\end{equation}
or
\begin{equation}
    \vert y\vert>B.
\end{equation}
Then, directly from \cref{eq:lattice vector},
\begin{equation}\label{eq:unbounded vector bound}
    \lVert\mathbf{v}(\mathbf{a},y)\rVert_2>B.
\end{equation}
\end{enumerate}

\paragraph{An appropriate parameter choice ensures that LLL yields a correct
vector from case~\ref{it:case bounded},
subcase~\ref{it:subcase 1 bounded}:}

Since $\Lambda(\mathbf{L})$ has rank $D=2K+2$, the LLL algorithm
applied to $\mathbf{L}$ returns a nonzero lattice vector
$\mathbf{v}_{\mathrm{LLL}}$ satisfying
\begin{equation}
    \left\lVert\mathbf{v}_{\mathrm{LLL}}\right\rVert_2
    \leq
    2^{(D-1)/2}\lambda_1\bigl(\Lambda(\mathbf{L})\bigr)
    =
    2^{K+\frac{1}{2}}
    \lambda_1\bigl(\Lambda(\mathbf{L})\bigr),
\end{equation}
where $\lambda_1\bigl(\Lambda(\mathbf{L})\bigr)$ denotes the length of
a shortest nonzero lattice vector
\cite[Theorem~17.2.12]{Ga12}.

The vector
$\mathbf{v}\bigl(\mathbf{a}(\mathbf{m}),-1\bigr)$ belongs to the
lattice and, by the preceding estimates, satisfies
\begin{equation}
    \left\lVert
        \mathbf{v}\bigl(\mathbf{a}(\mathbf{m}),-1\bigr)
    \right\rVert_2
    \leq
    2^{n+1}+2+A\delta_n.
\end{equation}
Consequently,
\begin{equation}\label{eq:LLL guarantee}
    \left\lVert\mathbf{v}_{\mathrm{LLL}}\right\rVert_2
    \leq
    2^{K+\frac{1}{2}}
    \left(2^{n+1}+2+A\delta_n\right).
\end{equation}

We now choose the parameters\footnote{The parameters play different
roles: $\delta_n$ bounds the error of the partition-function
approximation, $A$ ensures that differences between partition-function
values dominate the rounding error in the integer lattice, and $B$ is
a threshold used to separate bounded from unbounded lattice
coefficients.}
as follows. Define
\begin{equation}
    \epsilon_B
    \coloneqq
    \epsilon(K,B^2+B),
\end{equation}
and set
\begin{equation}\label{eq:parameter choice}
\begin{aligned}
    B
    &\coloneqq
    2^{K+n+4},
    \\
    A
    &\coloneqq
    2^{
        1+
        \left\lceil
            \log_2\left(
                \frac{(D+3)B}{\epsilon_B}
            \right)
        \right\rceil
    },
    \\
    \delta_n
    &\coloneqq
    \frac{1}{A}.
\end{aligned}
\end{equation}
In particular,
\begin{equation}\label{eq:A condition}
    A\epsilon_B>(D+3)B
\end{equation}
and $A\delta_n=1$.

For $n\geq1$, we have
\begin{equation}
    2^{n+1}+2+A\delta_n
    =
    2^{n+1}+3
    <
    2^{n+2}.
\end{equation}
It follows from \cref{eq:LLL guarantee} that
\begin{equation}\label{eq:B condition}
\begin{aligned}
    \left\lVert\mathbf{v}_{\mathrm{LLL}}\right\rVert_2
    &<
    2^{K+\frac{1}{2}}2^{n+2}
    \\
    &=
    2^{K+n+\frac{5}{2}}
    <
    2^{K+n+4}
    =
    B.
\end{aligned}
\end{equation}
Therefore, case~\ref{it:case unbounded} cannot contain
$\mathbf{v}_{\mathrm{LLL}}$.

Next, every vector in case~\ref{it:case bounded},
subcase~\ref{it:subcase 2 bounded}, satisfies
\begin{equation}
\begin{aligned}
    \left\lVert\mathbf{v}(\mathbf{a},y)\right\rVert_2
    &\geq
    A\left(\epsilon_B-B\delta_n\right)-DB
    \\
    &=
    A\epsilon_B-B-DB
    \\
    &>
    (D+3)B-(D+1)B
    \\
    &=
    2B.
\end{aligned}
\end{equation}
By \cref{eq:B condition}, such a vector also cannot be
$\mathbf{v}_{\mathrm{LLL}}$.

Hence $\mathbf{v}_{\mathrm{LLL}}$ belongs to
case~\ref{it:case bounded},
subcase~\ref{it:subcase 1 bounded}, and therefore satisfies
\begin{equation}
    \mathbf{a}
    +
    y\mathbf{a}(\mathbf{m})
    =
    \mathbf{0}.
\end{equation}
Since $\mathbf{v}_{\mathrm{LLL}}$ is nonzero, this implies $y\neq0$,
and the desired degeneracy vector can be recovered as
\begin{equation}
    \mathbf{a}(\mathbf{m})
    =
    -\frac{\mathbf{a}}{y}.
\end{equation}

\paragraph{The entire algorithm, including LLL, is polynomial-time
computable.}

We now describe the complete algorithm for recovering
$Z_n(\mathbf{m})$ from a sufficiently accurate approximation.

Given $\mathbf{m}$, we first compute $K$, $D$, and the parameters
$A$, $B$, and $\delta_n$ from \cref{eq:parameter choice}. We then
query the approximation oracle to obtain a rational number $z$
satisfying
\begin{equation}
    \left\vert
        Z_n(\mathbf{m})-z
    \right\vert
    <
    \delta_n.
\end{equation}
Using $A$ and $z$, we construct the integer lattice basis
$\mathbf{L}$ from \cref{eq:lattice basis} and apply the LLL algorithm.
By the preceding argument, the resulting vector has the form
\begin{equation}
    \mathbf{v}_{\mathrm{LLL}}
    =
    \mathbf{v}(\mathbf{a},y),
\end{equation}
where
\begin{equation}
    \mathbf{a}
    +
    y\mathbf{a}(\mathbf{m})
    =
    \mathbf{0}
\end{equation}
and $y\neq0$. We can therefore recover the degeneracy vector exactly
as
\begin{equation}
    \mathbf{a}(\mathbf{m})
    =
    -\frac{\mathbf{a}}{y}.
\end{equation}
This determines the partition function exactly through
\begin{equation}
    Z_n(\mathbf{m})
    =
    F\bigl(\mathbf{a}(\mathbf{m})\bigr)
    =
    e^{-K}
    \sum_{l=0}^{2K}
        a_l(\mathbf{m})e^l.
\end{equation}

In order to prove that the entire algorithm and in particular LLL runs in polynomial time, 
it remains to verify that parameters that specify $\mathbf{L}$ have polynomial bit complexity. First, $K=\mathcal{O}(n^2)$ and
\begin{equation}
    \log_2 B
    =
    K+n+4
    =
    \mathcal{O}(n^2).
\end{equation}
The transcendence estimate gives
\begin{equation}
    -\log\epsilon_B
    =
    \mathcal{O}\left(
        K^2\bigl(K+\log B\bigr)
    \right)
    =
    \mathcal{O}(n^6).
\end{equation}
Consequently, the parameter choice in
\cref{eq:parameter choice} satisfies
\begin{equation}
    \log_2 A
    =
    \mathcal{O}(n^6),
    \qquad
    \log_2\left(\frac{1}{\delta_n}\right)
    =
    \mathcal{O}(n^6).
\end{equation}
Thus $A$, $B$, and $\delta_n$ all have descriptions of polynomial
length. Here we use that the transcendence bound is effective, so a
sufficiently large integer value of $A$ can be computed from the
explicit bound.

The integers $q_l$ can be computed by approximating
$Ae^{-K+l}$ to absolute error at most $1/4$ and rounding the result
to the nearest integer. This produces
\begin{equation}
    \left\vert
        q_l-Ae^{-K+l}
    \right\vert
    <
    1.
\end{equation}
Since $z$ is rational, $q_z$ can be obtained by exact rational
arithmetic and rounding $Az$ to the nearest integer.

Moreover,
\begin{equation}
    \left\vert q_l\right\vert
    \leq
    Ae^K+1.
\end{equation}
Since
\begin{equation}
    Z_n(\mathbf{m})
    \leq
    2^ne^K
\end{equation}
and
\begin{equation}
    \left\vert z\right\vert
    \leq
    2^ne^K+\delta_n,
\end{equation}
we also have
\begin{equation}
    \left\vert q_z\right\vert
    \leq
    A\left(2^ne^K+\delta_n\right)+1.
\end{equation}
It follows that the largest entry of $\mathbf{L}$ has bit length
\begin{equation}
    \log_2\left(
        1+
        \left\lVert\mathbf{L}\right\rVert_\infty
    \right)
    =
    \mathcal{O}\left(
        \log A+n+K
    \right)
    =
    \mathrm{poly}(n).
\end{equation}

The LLL algorithm applied to an integer basis runs in time polynomial
in the rank, the ambient dimension, and the bit length of the basis
entries \cite[Section~17.5]{Ga12}. In our construction, the rank is
$D=2K+2=\mathcal{O}(n^2)$, the ambient dimension is $D+1$, and all
basis entries have polynomial bit length. Hence the LLL step, and
therefore the entire recovery algorithm, runs in polynomial time.
\end{proof}

\begin{lemma}[Complexity of computing $\mathcal{Z}(\mathbf{m})$]\label{lem:part fun comp}
     Computing partition functions  $\mathcal{Z}(\mathbf{m})$ of $\mathcal{M}_{\mathrm{Ising}}$ is $\SharpPclass$-hard.
\end{lemma}
\begin{proof}
We remark that this result is also well established in the  literature (see e.g.\ \cite{Je93,Is00}), however often for more general parameter regimes, e.g.\ allowing for an additional coupling $J\neq1$ or variable temperature $\beta\neq1$. We hence provide an explicit proof. 

We will prove that an oracle for partition functions $\mathcal{Z}(\mathbf{m})$ would yield a polytime algorithm for the counting version of $\textsc{MaxCut}$, $\textsc{\#MaxCut}$, which is $\SharpPclass$-complete \cite{Go08}. 

Given a graph $G$ and an integer $k$,
$\textsc{\#MaxCut}$ asks for the number of graph cuts of $G$ of size $k$. Recall from \cref{lem:Ising gse}, that 
\begin{equation}
    H_{S_G}(\mathbf{v}) = \vert E \vert -2 \vert E_{+-}(\mathbf{v}) \vert 
\end{equation}
where $\vert E_{+-}(\mathbf{v}) \vert$ is the size of the graph cut corresponding to configuration $\mathbf{v}$. Note that each graph cut corresponds to precisely two configurations $\mathbf{v}$ and the flipped configuration $\bar{\mathbf{v}}$ that differs from $\mathbf{v}$ on every spin. This implies that the number of graph cuts of a certain size can be computed from the number of configurations with a certain energy, i.e.\ the degeneracy of the respective energy level.

Finally, following the proof of \cref{lem:part fun approx hard}, given the partition function value $\mathcal{Z}(\mathbf{m})$, the full degeneracy vector of $S_{\mathbf{m}}$, $\mathbf{a}(\mathbf{m})$ can be computed in polynomial time. Thus, given $\mathcal{Z}(\mathbf{m})$ also the number of graph cuts of the respective graph $G$ can be computed in polynomial time.
\end{proof}

\begin{lemma}[Complexity of approximating $\mathcal{C}(\mathbf{m}_1, \mathbf{m}_2)$]\label{lem:part fun ratios comp}
    Computing partition function ratios of $\mathcal{M}_{\mathrm{Ising}}$, $\mathcal{C}(\mathbf{m}_1, \mathbf{m}_2)$ or approximating them to sufficient precision is $\SharpPclass$-hard.
\end{lemma}
\begin{proof}
    We prove that both would allow one to compute partition functions of $\mathcal{M}_{\mathrm{Ising}}$, which by \cref{lem:part fun comp} is $\SharpPclass$-hard. 
    
    To this end note that for system size $n$, $K=\binom{n}{2}$, 
    \begin{equation}
        \mathcal{Z}(-\mathbf{1}_K)=2^n,
    \end{equation}
    since $S_{-\mathbf{1}_K}$ is the Ising system of $n$ spins and no edges so $H_{S_{-\mathbf{1}_K}}$ is identically zero.
        
    Thus, $\mathcal{Z}(\mathbf{m})$ can be computed as 
    \begin{equation}
        \mathcal{Z}(\mathbf{m}) = 2^n \mathcal{C}(\mathbf{m}, -\mathbf{1}_K).
    \end{equation}

    Similarly, an approximation $z\in \mathbb{R}$ of $\mathcal{C}(\mathbf{m}, -\mathbf{1}_K)$ of additive errors at most $2^{-n}\delta_n$ (where $\delta_n$ is defined in \cref{lem:part fun approx hard}) satisfies 
    \begin{equation}
        \vert2^nz - \mathcal{Z}(\mathbf{m}) \vert \leq \delta_n, 
    \end{equation}
    so by \cref{lem:part fun approx hard}, it allows one to compute the exact value $\mathcal{Z}(\mathbf{m})$ in polynomial time.
\end{proof}

\begin{corollary}[Complexity of approximating $\mathcal{Z}$ and $\mathcal{C}$ to relative errors]\label{cor:part fun approx mult}
Approximating $\mathcal Z$ and $\mathcal C$ to sufficiently
small relative error $2^{-\mathrm{poly}(n)}$ is
$\SharpPclass$-hard.
\end{corollary}
\begin{proof}
Let $K=\binom n2$ and let $\delta_n$ be the additive accuracy
required by \cref{lem:part fun approx hard}.
Since $Z_n(\mathbf m)\leq 2^ne^K$, relative error at most
$\delta_n/(2^ne^K)$ implies additive error at most $\delta_n$.
The required number of precision bits is polynomial in $n$.

For partition-function ratios, let $\mathbf m_0$ encode
the edgeless graph. Since $Z_n(\mathbf m_0)=2^n$,
multiplying an approximation to $C_n(\mathbf m,\mathbf m_0)$
by $2^n$ gives an approximation to $Z_n(\mathbf m)$ with
the same relative error.
\end{proof}

\subsection{Pseudo-Boolean functions}\label{ssec:bool basics}
The local energy functions and Hamiltonians of spin systems are functions of type $f\colon \mathbf{2}^n \to \mathbb{R}$.
Such functions are called \emph{pseudo-Boolean} functions \cite{Od14}. They form a $2^n$-dimensional real vector space. In \cref{def:spin system} we consider all local energy functions  expanded in the basis of \emph{indicator functions} 
\begin{equation}
    \delta_{\mathbf{s}}(\mathbf{v})= \begin{cases}
        1 \ &\text{if} \ \mathbf{v}=\mathbf{s} \\ 
        0 \ &\text{else}.
    \end{cases}
\end{equation}
The respective coefficients  precisely correspond to the $2^n$ values of $f$ (cf.\ \cref{def:spin system}).

Pseudo-Boolean functions can also be expanded in the basis of \emph{parity functions} 
\begin{equation}
    \pi_S(\mathbf{v}) = \prod_{i \in S} v_i,
\end{equation}
for $S \subseteq [n]=\{1, \ldots, n\}$. This basis is also called \emph{Fourier basis}.
The resulting coefficients are  denoted $\widehat{f}(S)$ and the corresponding decomposition 
\begin{equation}
    f(\mathbf{v}) = \sum_{S}\widehat{f}(S)\pi_S(\mathbf{v})
\end{equation}
is called the \emph{Fourier decomposition} of $f$. 
The cardinality of the largest subset $S$ with non-zero Fourier coefficient $\widehat{f}(S)$ is called the \emph{Fourier degree} of $f$. 
Note that the Fourier decomposition of $f$ is a multilinear polynomial in $v_i$.

The parity functions $\pi_S$ provide an orthonormal basis w.r.t.\ the inner product 
\begin{equation}
    \langle \pi_S, \pi_{S'} \rangle \coloneqq \frac{1}{2^n}\sum_{\mathbf{v}}\pi_S(\mathbf{v})\pi_{S'}(\mathbf{v}) = \delta_{S,S'}.
\end{equation}
This implies that 
\begin{equation}\label{eq:fourier coeff}
    \widehat{f}(S) =  \langle \pi_S, f \rangle.
\end{equation}
Further, the Fourier basis also satisfies
\begin{equation}
\frac{1}{2^n}\sum_S\pi_S(\mathbf{v})\pi_S(\mathbf{w})=\delta_{\mathbf{v},\mathbf{w}}.
\end{equation}

There exist several sub-classes of pseudo-Boolean functions. For this work, the following are the most important.
\begin{definition}[useful pseudo-Boolean functions]
    Let $f\colon \mathbf{2}^n \to \mathbb{R}$ be a pseudo-Boolean function. We call $f$ 
    \begin{enumerate}
        \item a \emph{distance function} if $f(\mathbf{v})=g(\mathrm{d}_H(\mathbf{v},\mathbf{s}))$, where $\mathbf{s}\in \mathbf{2}^n$ is some fixed configuration, $g$ is some fixed, real-valued function and $\mathrm{d}_H(\mathbf{v},\mathbf{s})$ is the Hamming distance between $\mathbf{v}$ and $\mathbf{s}$, that is, the number of components $j$ with $v_j\neq s_j$.
        \item a \emph{ridge function} if $f(\mathbf{v})=g(\mathbf{A}\mathbf{v})$, where $\mathbf{A}$ is a fixed vector with real components and $g$ is a fixed, real-valued function. We call a ridge function small if its weight vector $\mathbf{A}$ has small components and $g(\mathbf{A}\mathbf{v})$ has small values.  
        \item a \emph{linear threshold function} (LTF) if $f(\mathbf{v})=\sgn(L(\mathbf{v}))$, and $L$ is (at most)  linear, i.e.\ $L(\mathbf{v})= \mathbf{W}\mathbf{v}+b$ for a real vector $\mathbf{W}$ and real scalar $b$.
        \item a \emph{polynomial threshold function} (PTF) if $f(\mathbf{v})=\sgn(P(\mathbf{v}))$ and $P(\mathbf{v})$ is a polynomial in the components of $\mathbf{v}$.
    \end{enumerate}
\end{definition}

Note that we might equivalently express pseudo-Boolean functions such as local energy functions of spin systems as functions of type $\{0,1\}^n \to \mathbb{R}$, by identifying $\mathbf{2}=\{-1,+1\}$ with $\{0,1\}$. 
Also functions of type $\{0,1\}^n \to \mathbb{R}$ can be expressed as multilinear polynomials and the degree of this polynomial representation equals the Fourier degree of the transformed function, of type $\mathbf{2}^n \to \mathbb{R}$, given any identification of $\mathbf{2}$ with $\{0,1\}$ \cite{Od14}.
This implies that we can equivalently define pseudo-Boolean functions and hence also spin system (cf. \cref{def:spin system}) in terms of inputs, i.e.\ states from $\{0,1\}$, and notions such as Fourier degree or interaction order of the spin system remain unchanged (see e.g.\ \cref{lem:rbm ridge}).
    
Finally, the separation results on shallow models of \cref{sec:RBM separation} are witnessed by the inner product modulo $2$ function and its $k$-party generalizations. These functions are best defined  for states from $\{0,1\}$. 
\begin{definition}[inner product modulo 2]\label{def:inner prod mod 2}
    Let $k\geq 2$ and $\mathbf{v}=(\mathbf{v}^{(1)}, \ldots, \mathbf{v}^{(k)})$ where each $\mathbf{v}^{(l)}$ is an $n$-component vector with entries in $\{0,1\}$. The $k$-party inner product modulo $2$, $\mathrm{IP}_k$ is defined by 
    \begin{equation}
    \mathrm{IP}_k(\mathbf{v}^{(1)}, \ldots, \mathbf{v}^{(k)}) = (-1)^{\sum_{i=1}^n \prod_{l=1}^k v^{(l)}_i \mod 2}.
        \end{equation}
\end{definition}

\subsection{Communication complexity theory}\label{ssec:comm compl}
While computational complexity theory (see \cref{sec:comp compl}) studies the computational cost to compute a function $f$, \emph{communication complexity} theory studies the communication cost of computing $f$, that is, the minimum amount of communication that is required to compute $f$, assuming the input is split between parties \cite{Ra20, Ab09}.

In the $2$-party scenario the setting is as follows:
Fix some function $f\colon \mathbf{2}^{2n} \to \mathbf{2}$ and split the input $\mathbf{v}=(\mathbf{v}^{(1)}, \mathbf{v}^{(2)})$ between two parties or players, that is, player $1$ only has access to input $\mathbf{v}^{(1)}$ and player $2$ only has access to input $\mathbf{v}^{(2)}$. The players agree on some communication \emph{protocol} $\Pi$, that is, a strategy
that on any given input $\mathbf{v}$, after a finite number of rounds produces some output $\Pi(\mathbf{v}) \in \mathbf{2}$. Their goal is to compute $f(\mathbf{v})$, i.e.\ come up with a protocol with $\Pi(\mathbf{v})=f(\mathbf{v})$ for all inputs $\mathbf{v}$.
At each round of the protocol, each player might perform any sequence of the following two actions:
\begin{enumerate}
    \item communicate any number of bits to the respective other player,
    \item perform any computation on their half of the input and any previously received bits.
\end{enumerate}
The players proceed in rounds until 
either of them returns the final output $\Pi(\mathbf{v})$. 
Note that the protocol must not depend on the input $\mathbf{v}$, the players must agree on one protocol that they apply to any input.

In communication complexity, the players are assumed to have infinite computational resources. The communication complexity of a protocol is the total number of bits that in the worst case, i.e.\ for the hardest input $\mathbf{v}$, is communicated between players. Any computational cost of the individual computations performed by the players is neglected. 
The communication complexity of a function $f$ then is the complexity of the best protocol for $f$. 
Similar to computational complexity, communication complexity then usually considers the asymptotic behavior ($n \to \infty$) of the communication complexity of protocols or functions,  that is, one usually considers protocols and functions that are defined for inputs of arbitrary length (see e.g.\ \cref{def:inner prod mod 2}).
Note that it requires exactly $n$ bits of communication for player $1$ to communicate their entire half of the input $\mathbf{v}^{(1)}$ with player $2$. Then player $2$ has access to the entire input and therefore, for any function $f$ can compute and return $f(\mathbf{v})$ in the next round. Thus, the communication complexity of any function $f$ is at most $\mathcal{O}(n)$. In this work,  we consider communication protocols to be efficient when they require at most a logarithmic amount of communication, i.e.\ have complexity $\mathcal{O}(\log(n))$. 

The number on the forehead (NOF) $k$-party communication scenario generalizes communication complexity to $k>2$ parties and equals the standard $2$-party setting for $k=2$.
Consider a function $f\colon \mathbf{2}^{kn} \to \mathbf{2}$.
In the $k$-party NOF setting, the input is split between $k$ players, $\mathbf{v}=(\mathbf{v}^{(1)}, \ldots, \mathbf{v}^{(k)})$. This is done such that player $i$ has access to all input parts $\mathbf{v}^{(j)}$ except for $\mathbf{v}^{(i)}$. All of the above definitions for the $2$-party setting extend to the $k$-party setting. Usually, we consider $k$ to be fixed, then a trivial protocol that communicates all input parts to the last player then still requires $\mathcal{O}(n)$ bits of communication.

There are several variants of communication complexity that loosen some of the requirements of the communication protocol. 
The above can be considered the exact setting: In order to successfully compute a function $f$, the protocol is required to return $f(\mathbf{v})$ for all inputs $\mathbf{v}$.
In the so-called $\epsilon$-distributional setting, a protocol $\Pi$ for $f$   only is required to return the correct output $f(\mathbf{v})$ for a fraction of at least $\frac{1+\epsilon}{2}$ of the inputs. We call this fraction of inputs on which the protocol correctly computes $f$ the success rate of the protocol
\begin{equation}
    \mathrm{S}(\Pi) = \frac{1}{2^{kn}}\vert \{\mathbf{v} \in \mathbf{2}^{kn} \mid \Pi(\mathbf{v}) = f(\mathbf{v}) \} \vert. 
\end{equation}
In the $\epsilon$-distributional setting, the communication complexity of $f$, $C_{\epsilon}(f)$, is the minimal communication cost of any protocol that returns $f$ with success rate at least $\frac{1+\epsilon}{2}$. 
Note that random guessing the output for any input leads to a success rate of $1/2$, so protocols with success rate at least $\frac{1+\epsilon}{2}$
have at least an $\epsilon/2$ advantage over random guessing.
The success rate can also be interpreted as success probability when the input is randomly sampled from the uniform distribution on $\mathbf{2}^{kn}$. Similarly, one can also consider the case of drawing inputs from a non-uniform distribution.
One may also consider randomized protocols, in which
players have access to a shared source of randomness. Their actions, in each round, might depend on the output of this source of randomness. Then, one can consider the success probability of such a protocol $\Pi$, for each input $\mathbf{v}$, separately. Note that when such a randomized protocol, for any fixed input $\mathbf{v}$ has success probability 
\begin{equation}
    \Pr[\Pi(\mathbf{v}) = \sgn(f(\mathbf{v}))] \geq \frac{1+\epsilon}{2},
\end{equation}
then when $\mathbf{v}$ is sampled uniformly, $\Pi$ also has success rate at least $\frac{1+\epsilon}{2}$.

Finally, it is worth pointing out that for any $k\geq 2$, the $k$-party inner product modulo $2$ (\cref{def:inner prod mod 2}) has maximal $k$-party communication complexity. Specifically, for $\epsilon=\frac{1}{\mathrm{poly}(n)}$ the deterministic, $\epsilon$-distributional $k$-party NOF communication complexity of the $k$-party inner product modulo $2$ is  
\begin{equation}
    C_{\epsilon}(\mathrm{IP}_k) = \Omega(n).
\end{equation}
Similarly, even when considering randomized protocols that merely are required to succeed with probability at least $\frac{1+\epsilon}{2}$ when the inputs are sampled uniformly, the communication complexity of $\mathrm{IP}_k$ is $\Omega(n)$ \cite{Ba92}.

\section{Proofs of \cref{sec:RBM separation}}\label{sec:proofs}
In this section we provide proofs for \cref{lem:sim PTF} (\cref{ssec:proof 1}) and \cref{lem:rbm ridge} (\cref{ssec:proof 2}).
\subsection{Proof of \cref{lem:sim PTF}}\label{ssec:proof 1}
Subsequently, we prove \cref{lem:sim PTF}, that is, we prove that a function (family) $f$ can be efficiently simulated by a shallow, $k$-local system $S$ if and only if it can be represented as a linear combination of polynomially many PTFs of degree at most $k-1$,    \begin{equation}\label{eq:PTF representation main}
        f(\mathbf{v}) = \sum_I w_I \sgn(P_I(\mathbf{v})).
\end{equation}
The precise statement of \cref{lem:sim PTF} additionally includes the following results on the respective weights of the simulation and parameters of the PTF representation of $f$.
If the simulation weights and energy shifts are
polynomially bounded or have polynomial precision,
then the same applies to the PTF representation. The converse direction only applies to PTF representation with small parameters and the same then applies to the respective simulation.

\paragraph{If $S \vpolysim f$  then $f$ can be represented by a linear combination of polynomially many PTFs.}    
    We construct polynomials $P_I$ and coefficients $w_I$ that satisfy
    \begin{equation}\label{eq:PTF representation}
        \min_{\mathbf{h}}H_S(\mathbf{v},\mathbf{h}) = \sum_I w_I \sgn(P_I(\mathbf{v})),
    \end{equation}
    such that 
    \begin{enumerate}
        \item\label{it:PTF 1} the sum runs over at most $\mathrm{poly}(n)$ terms,
        \item\label{it:PTF 2} the degree of all polynomials $P_I$ is at most $k-1$, and 
        \item\label{it:PTF 3} whenever all weights of $S$ and $\Gamma$ are polynomially bounded or of polynomial precision then the same applies for all coefficients $w_I$ of the linear combination and all parameters of the polynomials $P_I$.
    \end{enumerate}
    \cref{eq:PTF representation} yields a PTF representation of the target function $f$, since $S\simul f$, so $\min_{\mathbf{h}}H_S(\mathbf{v},\mathbf{h})-\Gamma = f(\mathbf{v})$ and the constant energy shift $\Gamma$ can trivially be realized as a PTF with constant polynomial.
    
    We construct the PTF representation of $f$ (\cref{eq:PTF representation}) by first expanding $H_S$ as a linear combination of polynomials in $\mathbf{v}$ and $\mathbf{h}$. Since $S$ is shallow, the effective energy $\min_{\mathbf{h}}H_S(\mathbf{v},\mathbf{h})$ can be obtained by minimizing each hidden spin $h_i$ separately, leading to minimizers $h_{i, \min}(\mathbf{v})$. Thus constructing the PTF representation (\cref{eq:PTF representation}) amounts to finding PTF representations for terms of the form $h_{i,\min}\prod_j v_j$.

    We start by expanding $H_S$.
    Since $S$ is shallow, by \cref{lem:shallow polytime},
    \begin{equation}\label{eq:shallow ham decomp}
        H_S(\mathbf{v},\mathbf{h})=H'(\mathbf{v})+\sum_{i=1}^m H_i(\mathbf{v},h_i),
    \end{equation}
    where $H'$ is a sum of all local interactions $J_S(e)$  that involve no hidden spins and  $H_i$ is a sum of all local interactions $J_S(e)$ that involve the hidden spin $h_i$.
    We now expanded $H_S$ in the Fourier basis (see \cref{ssec:bool basics}). Then $H_S$ is a linear combination of monomials that either only involve visible spins $\prod_j v_j$ or precisely involve one hidden spin $h_i \prod_j v_j$, i.e.\ 
    \begin{equation}\label{eq:Ham fourier expansion}
           H_S(\mathbf{v},\mathbf{h})=\sum_{I_0}w_{I_0}\prod_{j \in I_0}v_j + \sum_{i=1}^m \sum_{I_i}w_{I_i}h_i\prod_{j\in I_i}v_j,
    \end{equation}
    where $I_0$ are all subsets of hyperedges of $S$ that only involve visible spins and $I_i$ are all subsets of hyperedges of $S$ that involve hidden spin $i$.

    Note that \cref{eq:Ham fourier expansion} is a sum of at most $\mathrm{poly}(n)$ terms: Each  $j$-local interaction of $S$ is described by at most $2^j$ non-zero Fourier coefficients so each corresponding hyperedge of $S$ contributes at most $2^j$ summands to \cref{eq:Ham fourier expansion} but also contributes exactly $2^j$ to $\vert S \vert$. 
    Thus, the overall number of summands in \cref{eq:Ham fourier expansion} is upper bounded by $\vert S \vert$ and since $S$ efficiently simulates $f$, $\vert S \vert = \mathrm{poly}(n)$.

    Since $S$ is shallow hidden spins $h_i$ can be minimized independently. 
    Define $P_i(\mathbf{v}) = \sum_{I_i}w_{I_i}\prod_{j\in I_i}v_j$ then the second term in \cref{eq:Ham fourier expansion} can be expressed as $\sum_{i=1}^m h_i P_i(\mathbf{v})$.
    This leads to
    \begin{equation}\label{eq:Ham fourier expansion min}
    \begin{split}
            \min_{\mathbf{h}}H_S(\mathbf{v},\mathbf{h})&=H'(\mathbf{v})+\sum_{i=1}^m h_{i,\min}(\mathbf{v})P_i(\mathbf{v}) \\
            &= \sum_{I_0}w_{I_0}\prod_{j \in I_0}v_j + \sum_{i=1}^m \sum_{I_i}w_{I_i}h_{i,\min}(\mathbf{v})\prod_{j\in I_i}v_j,
    \end{split}
    \end{equation}
    where 
    \begin{equation}
        h_{i,\mathrm{min}}(\mathbf{v}) = -\sgn(P_i(\mathbf{v})).
    \end{equation}
    This is a linear combination of two types of terms, monomials in $\mathbf{v}$, i.e.\ $\prod_j v_j$ and products involving $h_{i,\min}(\mathbf{v})$, i.e.\ $h_{i,\min}(\mathbf{v})\prod_j v_j$. 
    In the following, we prove that terms of both types can be expressed as linear combinations of PTFs.

    First, we construct a PTF representation for terms of the first type, i.e.\ $\prod_j v_j$.
    Note such terms can trivially be expressed as PTF, since $\prod_j v_j = \sgn(\prod_j v_j)$.
    However, for products of $k$ terms this leads to degree $k$ polynomials, i.e.\ violates condition \ref{it:PTF 2}. 
    We now prove that products of $k$ terms can be expressed as linear combination of degree $k-1$ PTFs.
    For $a,b \in \mathbf{2}$ define the four indicator functions
    \begin{equation}
        f_{a,b}(v,w) = \begin{cases}
            +1 \ &\text{if} \ (v,w)=(a,b) \\ 
            -1 \ &\text{else}.
        \end{cases}
    \end{equation}
    These four functions provide a basis of function space  $\mathbb{R}^{\mathbf{2}^2}$.
    Let $I_0$ be the index set of a $k$-body interaction and consider $\prod_{j\in I_0} v_j$. We can express this in terms of $f_{a,b}$, as 
    \begin{equation}
        \prod_{j\in I_0} v_j = v_{j_0} \cdot \prod_{j \in I_0, j\neq j_0}v_j = \sum_{a,b} \frac{ab}{2} f_{a,b}(v_{j_0}, \prod_{j \in I_0, j\neq j_0}v_j).
    \end{equation}

    Constructing a PTF representation for $\prod_{j\in I_0} v_j$ amounts to constructing degree $k-1$ polynomials $P_{a,b,I_0}(\mathbf{v})$ such that 
    \begin{equation}
        \sgn(P_{a,b,I_0}(\mathbf{v})) = f_{a,b}(v_{j_0}, \prod_{j \in I_0, j\neq j_0}v_j),
    \end{equation}
    i.e.\ 
    \begin{equation}
        P_{a,b,I_0}(\mathbf{v}) > 0 \quad \Leftrightarrow \quad v_{j_0}=a \ \ \text{and} \ \  \prod_{j \in I_0, j\neq j_0}v_j=b.
    \end{equation}
    This is satisfied by 
    \begin{equation}
        P_{a,b,I_0}(\mathbf{v}) = av_{j_0} + b\left ( \prod_{j \in I_0, j\neq j_0}v_j\right) - 1.
    \end{equation}

    Second, we construct the PTF representation for terms of the second type, i.e.\ $h_{i,\min}(\mathbf{v})\prod_j v_j$. Again, we expand $h_{i,\mathrm{min}}(\mathbf{v})\prod_j v_j$ as 
    \begin{equation}
       h_{i,\mathrm{min}}(\mathbf{v})\cdot  \prod_{j\in I_i} v_j = \sum_{a,b}\frac{ab}{2}f_{a,b}(h_{i,\mathrm{min}}(\mathbf{v}),\prod_{j\in I_i} v_j).
    \end{equation}
    Therefore, constructing a PTF representation for terms of the second type amounts to 
    constructing degree $k-1$ polynomials $P_{a,b,i,I_i}(\mathbf{v})$ such that 
    \begin{equation}
        \sgn(P_{a,b,i,I_i}(\mathbf{v})) = f_{a,b}(h_{i,\mathrm{min}}(\mathbf{v}),\prod_{j \in I_i} v_j).
    \end{equation}
    The polynomials thus need to satisfy 
    \begin{equation}\label{eq: P target}
        P_{a,b,i,I_i}(\mathbf{v}) > 0 \quad \Leftrightarrow \quad h_{i,\mathrm{min}}(\mathbf{v})=a \ \ \text{and} \ \ \prod_{j\in I_i} v_j=b.
    \end{equation}
    We construct $P_{a,b,i,I_i}(\mathbf{v})$ as follows:
    Choose $0<\eta\leq1$ smaller than every nonzero value
$|P_i(\mathbf v)|$, and set $\widetilde P_i=P_i+\eta$.
Under the polynomial-precision assumption, take
$\eta=1/(2q)$, where $q$ is a polynomially bounded common
denominator of all coefficients $w_{I_0}$ and $w_{I_i}$
in \cref{eq:Ham fourier expansion}, and of the energy
shift $\Gamma$.
    Then,
$\widetilde P_i$ never vanishes on Boolean inputs and
$\sgn(\widetilde P_i)=\sgn(P_i)$.

    Let $Q_i> \frac{1}{2}\max(\vert \widetilde P_i(\mathbf{v})\vert )$, define 
    \begin{equation}
        P_{a,b,i,I_i}(\mathbf{v})=-a \widetilde P_i(\mathbf{v}) - Q_i(1-b\prod_{j\in I_i} v_j).
    \end{equation}
    This is a degree $k-1$ polynomial. It is straightforward to see that $P_{a,b,i,I_i}(\mathbf{v})>0$ if and only if $-a \widetilde P_i(\mathbf{v})>0$ and $b\prod_{j\in I_i} v_j=1$: 
    the second term contributes either $-2Q_i$ or $0$ and since $Q_i> \frac{1}{2}\max(\vert \widetilde P_i(\mathbf{v})\vert)$ any positive contribution of the first term cannot outweigh the negative contribution of $-2Q_i$.
    Now $-a \widetilde P_i(\mathbf{v})>0$ if and only if $a =  -\sgn(\widetilde P_i(\mathbf{v}))=  -\sgn(P_i(\mathbf{v}))= h_{i,\mathrm{min}}(\mathbf{v})$ and $b\prod_{j\in I_i} v_j=1$ if and only if $b=\prod_{j\in I_i} v_j$. Thus $P_{a,b,i,I_i}(\mathbf{v})$ satisfies \cref{eq: P target}.

    The final PTF representation of $f$ thus is
    \begin{equation}\label{eq:final PTF}
    \begin{aligned}
        f(\mathbf{v})
        ={}&
        \sum_{I_0,\, \vert I_0\vert \leq k-1}
        w_{I_0}\sgn\!\left(\prod_{j\in I_0}v_j\right)
        \\
        &+
        \sum_{I_0,\, \vert I_0\vert =k}
        w_{I_0}\sum_{a,b}\frac{ab}{2}
        \sgn\!\left(P_{a,b,I_0}(\mathbf{v})\right)
        \\
        &+
        \sum_{i=1}^m\sum_{I_i}
        w_{I_i}\sum_{a,b}\frac{ab}{2}
        \sgn\!\left(P_{a,b,i,I_i}(\mathbf{v})\right)
        -\Gamma.
    \end{aligned}
    \end{equation}
    
    Finally, we argue that the construction satisfies the three conditions \ref{it:PTF 1}, \ref{it:PTF 2} and \ref{it:PTF 3}.
    We have argued above that the number of summands in \cref{eq:Ham fourier expansion min} is at most $\mathrm{poly}(n)$ so clearly the same holds true for \cref{eq:final PTF}, i.e.\ \cref{eq:final PTF} satisfies condition \ref{it:PTF 1}. 
    Further, all polynomials involved in \cref{eq:final PTF} are of degree at most $k-1$, by construction, so \cref{eq:final PTF} satisfies condition \ref{it:PTF 2}.
    
    To finish the proof, we show that \cref{eq:final PTF} also satisfies condition \ref{it:PTF 3}. Recall that the parameters $w_{I_0}$ and $w_{I_i}$ are Fourier coefficients of local interactions of $S$.
    Write $\widetilde P_i(\mathbf v)
=\sum_J c_{i,J}\prod_{j\in J}v_j$ including $J=\varnothing$ for the constant term, with
$\prod_{j\in\varnothing}v_j=1$, and set
\begin{equation}
    Q_i=1+\sum_J|c_{i,J}|.
\end{equation}
Then $Q_i>\frac12\max_{\mathbf v}|\widetilde P_i(\mathbf v)|$.
If the source weights and energy shift are polynomially
bounded, so are all resulting coefficients, since each
sum contains polynomially many terms and $\eta\leq1$.
If they have polynomial precision, all resulting
coefficients lie in $(2q)^{-1}\mathbb Z$.
Thus the construction preserves polynomially bounded
magnitude and polynomial precision separately.

\paragraph{If $f$ has a PTF representation with small parameters then there exists a shallow, $k$-local system $S$ such that $S \vpolysim f$ with small weights and energy shift.}
We prove the converse using \cref{lem:rbm ridge}.

Let
\begin{equation}
    \mathcal I
    =
    \{J\subseteq[n]:|J|\leq k-1\},
    \qquad
    u_J(\mathbf v)=\prod_{j\in J}v_j,
\end{equation}
with $u_\varnothing=1$.
Since $k$ is fixed, $|\mathcal I|=\mathrm{poly}(n)$.
Write
\begin{equation}
    P_i(\mathbf v)
    =
    \sum_{J\in\mathcal I}a_{i,J}u_J(\mathbf v).
\end{equation}
Treating the coordinates $u_J$ as independent spins,
define
\begin{equation}
    F_i(\mathbf u)
    =
    w_i\sgn
    \left(\sum_{J\in\mathcal I}a_{i,J}u_J\right).
\end{equation}
This is a small ridge function: its ridge vector
has small components, and its values belong to
$\{-w_i,w_i\}$.

Fix a polynomially bounded cut-off $\Delta>0$.
By \cref{lem:rbm ridge} and its construction,
there is a polynomial-size RBM $R_i$ satisfying
\begin{equation}
    \min_{\mathbf h_i}
    H_{R_i}(\mathbf u,\mathbf h_i)
    =
    F_i(\mathbf u)+\Gamma_i,
\end{equation}
with small weights and shift, degeneracy one,
and cut-off at least $\Delta$.

Now substitute $u_J=u_J(\mathbf v)$ in this
Hamiltonian and denote the resulting system by $S_i$.
An RBM interaction $h_{i,a}u_J$ becomes
\begin{equation}
    h_{i,a}\prod_{j\in J}v_j,
\end{equation}
which involves at most $k$ spins and only one hidden
spin. Visible biases become interactions of degree
at most $k-1$.
Thus $S_i$ is shallow and $k$-local, and
\begin{equation}
    \min_{\mathbf h_i}
    H_{S_i}(\mathbf v,\mathbf h_i)
    =
    w_i\sgn(P_i(\mathbf v))+\Gamma_i.
\end{equation}
Restricting the RBM inputs to those of the form
$\mathbf u(\mathbf v)$ preserves degeneracy one
and cannot decrease the cut-off.

Finally, choose disjoint hidden spins for the
different systems $S_i$ and set
\begin{equation}
    H_S(\mathbf v,\mathbf h)
    =
    \sum_{i=1}^{M}
    H_{S_i}(\mathbf v,\mathbf h_i).
\end{equation}
The hidden configurations can be minimized independently,
so
\begin{equation}
    \min_{\mathbf h}H_S(\mathbf v,\mathbf h)
    =
    f(\mathbf v)+\sum_{i=1}^{M}\Gamma_i.
\end{equation}
The resulting system remains shallow and $k$-local,
has degeneracy one, and has cut-off at least $\Delta$.

There are polynomially many polynomial-size summands,
so the total size and parameter magnitudes remain
polynomially bounded.
The ridge constructions can use the common denominator
of the given parameters, up to fixed factors; hence
their addition also preserves polynomial precision.
The total energy shift is small for the same reason.
\subsection{Proof of \cref{lem:rbm ridge}}\label{ssec:proof 2}
Subsequently, we prove \cref{lem:rbm ridge}, that is, we prove that a function (family) $f$ can be efficiently simulated by RBMs with small weights and energy shift if and only if it can be represented as a sum of polynomially many, jointly small ridge functions. 

\paragraph{Ridge representation $\Rightarrow$ RBM simulation.}
\mbox{}
\medskip
For this direction we work with spins in $\{0,1\}$; converting between this convention
and $\{-1,+1\}$ preserves the RBM structure, energy gaps, and
smallness of the weights.

We first construct a simulation of a single small ridge
function. For a sum of jointly small ridge functions,
use their common denominator uniformly in the constructions.
The resulting RBMs have a common polynomially bounded
denominator, up to a fixed factor.
Adding their Hamiltonians with disjoint hidden spins
therefore preserves polynomial size, polynomially bounded
parameter magnitudes, and polynomial precision, including
that of the total energy shift.

Let $f$ be a small ridge function, i.e.\ $f(\mathbf{v})=g(\mathbf{A}\mathbf{v})$ and $\mathbf{A}$ has components $A_i = \frac{p_i}{q}$ and $g(\mathbf{A}\mathbf{v})$ takes values $\frac{g_i}{q}$ for polynomially bounded integers $p_i,g_i,q$.

First, we may assume that
\begin{equation}
    f(\mathbf v)=g(\mathbf A\mathbf v),
    \qquad A_i\in\mathbb Z_{\geq0},
    \qquad K\coloneqq\sum_i A_i=\mathrm{poly}(n).
\end{equation}
Indeed, write the original ridge vector as $A_i=p_i/q$,
with polynomially bounded integers $p_i,q$. Absorb $q$ into
the argument of $g$. For each negative $p_i$, replace $v_i$
by $1-v_i$ and absorb the resulting constant into the argument
of $g$. This yields the nonnegative integer coefficients
$|p_i|$. These spin relabellings can be reversed in the final RBM.

The transformed function has the same attained values as the
original one. Define its values at any unattained integers in
$\{0,\ldots,K\}$ to be zero. Thus, for polynomially bounded
positive integers $G,Q$, we may assume
\begin{equation}
    |g(x)|\leq G,
    \qquad g(x)\in Q^{-1}\mathbb Z
    \qquad \text{for} \ x=0,\ldots,K.
\end{equation}
If $K=0$, the target is constant and the claim is immediate.
Hence assume $K\geq1$.

The construction has two parts. Adding a sufficiently large
quadratic term to $g$ allows us to interpolate it exactly by
ReLU functions with positive coefficients and half-integer breakpoints.
A second collection of ReLUs cancels the quadratic term,
up to a linear term that can be absorbed into visible biases.
Both collections are constructed so that, for every visible
configuration, changing any hidden spin from its minimizing
state increases the energy by at least $\Delta$.

Fix $\Delta>0$ and choose
\begin{equation}
    D=\max\{1,\lceil\Delta\rceil\},
    \qquad L=4KG+D.
\end{equation}
Write $\operatorname{ReLU}(z)=\max\{0,z\}$.

\paragraph{Interpolation with half-integer breakpoints.}
Set
\begin{equation}
    b=g(1)-g(0),\qquad a_1=2L,
\end{equation}
and define recursively, for $j=1,\ldots,K-1$,
\begin{equation}
    a_{j+1}
    =4L+2\bigl(g(j+1)-2g(j)+g(j-1)\bigr)-a_j.
\end{equation}
We claim that, for every integer $x\in\{0,\ldots,K\}$,
\begin{equation}\label{eq:ridge-quadratic-interpolation}
    g(x)+Lx^2
    =
    g(0)+bx+
    \sum_{j=1}^{K}
    a_j\operatorname{ReLU}\left(x-j+\tfrac12\right).
\end{equation}
Both sides agree at $x=0$, and their increments from $0$ to $1$
agree because $a_1=2L$. At each interior integer $j$, the
second difference of the ReLU sum is
$(a_j+a_{j+1})/2$. By the recurrence, this equals
\begin{equation}
    2L+g(j+1)-2g(j)+g(j-1),
\end{equation}
the second difference of the left-hand side.
This proves \eqref{eq:ridge-quadratic-interpolation}.

Since each second difference of $g$ has magnitude at most $4G$,
the recurrence also gives
\begin{equation}
    |a_j-2L|\leq 8(j-1)G.
\end{equation}
Our choice of $L$ therefore ensures that $a_j\geq2D>0$.
Consequently, each summand in
\eqref{eq:ridge-quadratic-interpolation} is a single ReLU with
affine function
\begin{equation}
    \ell_j(\mathbf v)
    =a_j\left(\mathbf A\mathbf v-j+\tfrac12\right).
\end{equation}
Because $\mathbf A\mathbf v$ is an integer,
\begin{equation}\label{eq:ridge-interpolation-margin}
    |\ell_j(\mathbf v)|\geq a_j/2\geq D.
\end{equation}

\paragraph{Cancelling the quadratic term.}
Set
\begin{equation}
    B=LK+D,\qquad M=B+D.
\end{equation}
For each $l$ with $A_l>0$, introduce the affine preactivation
\begin{equation}
    P_l(\mathbf v)
    =
    A_l\bigl(B-L\mathbf A\mathbf v-M(1-v_l)\bigr).
\end{equation}
Writing $x=\mathbf A\mathbf v$, we have
\begin{equation}
    P_l(\mathbf v)
    =
    \begin{cases}
        A_l(B-Lx)\geq D, & v_l=1,\\
        -A_l(Lx+D)\leq-D, & v_l=0.
    \end{cases}
\end{equation}
Thus this ReLU is active precisely when $v_l=1$, and
\begin{equation}\label{eq:ridge-quadratic-cancellation}
    \sum_{l:A_l>0}\operatorname{ReLU}(P_l(\mathbf v))
    =
    \sum_l A_lv_l(B-Lx)
    =
    Bx-Lx^2.
\end{equation}
Combining \eqref{eq:ridge-quadratic-interpolation} and
\eqref{eq:ridge-quadratic-cancellation} yields
\begin{equation}\label{eq:ridge-rbm-representation}
\begin{split}
    f(\mathbf v)
    ={}&g(0)+(b-B)\mathbf A\mathbf v\\
    &+\sum_{j=1}^{K}\operatorname{ReLU}(\ell_j(\mathbf v))
     +\sum_{l:A_l>0}\operatorname{ReLU}(P_l(\mathbf v)).
\end{split}
\end{equation}

\paragraph{RBM realization and cut-off.}
Introduce one hidden spin for each ReLU in
\eqref{eq:ridge-rbm-representation}, and define
\begin{equation}
    H(\mathbf v,\mathbf h,\mathbf t)
    =
    -(b-B)\mathbf A\mathbf v
    -\sum_{j=1}^{K}h_j\ell_j(\mathbf v)
    -\sum_{l:A_l>0}t_lP_l(\mathbf v).
\end{equation}
Every preactivation is affine in the visible spins, so this
is an RBM Hamiltonian. Using
$\min_{h\in\{0,1\}}(-hz)=-\operatorname{ReLU}(z)$ gives
\begin{equation}
    \min_{\mathbf h,\mathbf t}H(\mathbf v,\mathbf h,\mathbf t)
    =-f(\mathbf v)+g(0).
\end{equation}
Thus the RBM simulates $-f$ up to a constant energy shift.
Applying the construction to $-g$ instead gives a simulation
of $f$.

Every preactivation has absolute value at least $D$, by
\eqref{eq:ridge-interpolation-margin} and the gating bounds
above. Hence each hidden spin has a unique minimizing state.
Since hidden spins do not interact, any incorrect hidden
configuration costs at least $D$ above the minimum.
The simulation therefore has constant degeneracy one and
cut-off at least $\Delta$.

\paragraph{Size and weights.}
The construction uses at most $K+n=\mathrm{poly}(n)$ hidden
units and polynomially many interactions, independently of
$\Delta$. Moreover,
\begin{equation}
    a_j=\mathcal{O}(KG+D),\qquad B,M=\mathcal{O}(K^2G+KD),
\end{equation}
so all weights and biases have magnitude
$\mathrm{poly}(n,D)$.

Finally, the recurrence involves only integer multiples and
sums of values in $Q^{-1}\mathbb Z$. Thus every $a_j$ and $b$
belongs to $Q^{-1}\mathbb Z$, and all RBM weights and biases
have the common denominator $2Q$. Returning to the
$\{-1,+1\}$ convention changes this denominator by at most
a fixed factor. The weights therefore have polynomial
precision, and for $\Delta=\mathrm{poly}(n)$ they are small.

\paragraph{RBM Simulation $\Rightarrow$ ridge representation.}
\mbox{}
\medskip
For this direction we work with the standard spin convention $\{-1,+1\}$.
Suppose that an RBM $S$ efficiently simulates $f$
with small weights and energy shift $\Gamma$.
Write its Hamiltonian as
\begin{equation}
    H_S(\mathbf v,\mathbf h)
    =
    c+\mathbf a\mathbf v
    +\sum_{i=1}^{m}
      h_i\bigl(b_i+\mathbf W_i\mathbf v\bigr),
\end{equation}
where $m=\mathrm{poly}(n)$.
Since the hidden spins can be minimized independently,
\begin{equation}
    f(\mathbf v)
    =
    c-\Gamma+\mathbf a\mathbf v
    -\sum_{i=1}^{m}
      \left|b_i+\mathbf W_i\mathbf v\right|.
\end{equation}
The affine term is a ridge function, and each remaining
summand is a ridge function with ridge vector
$\mathbf W_i$ and scalar function
$g_i(t)=-|b_i+t|$.
Smallness of the RBM weights and shift implies that
these ridge vectors and attained values are jointly
small. 

\section{Examples \& remarks}\label{sec:supplement}
In this section we provide additional examples and remarks that go beyond the main results of this work. Specifically, we discuss the relation between simulations and approximations (\cref{rem:sim approx}), we provide a generalization of the efficient simulation of the hard simulations and approximations of \cref{sec:consequences} to target Ising models with inhomogeneous couplings and fields (\cref{rem:generalization NP-hard}), we construct an explicit, efficient DBM simulation of the $\NPclass$-hard GSE function $\mathcal{G}$ following \cref{lem:eff sim GSE} (\cref{ssec: supp DBM}), we discuss potential stronger separation results that would follow from the additional assumption that simulations be efficiently computable (\cref{rem:polytime simulations}) and we argue that complexity theoretic results imply that efficient simulations cannot be used to ``increase'' the ground-state energy gap of certain target spin models and hence the cut-off of the efficient ground-state simulation of \cref{lem:eff sim GSE} cannot be increased arbitrarily (\cref{ssec: supp cut-off}). 
\subsection{Simulation vs approximation}\label{rem:sim approx}
If $S \simul T$ with constant degeneracy then $S$ also approximates $p_T$ in the visible spin marginal of its Boltzmann distribution and hence spin ML models that are universal are also universal approximators, i.e.\ 
\begin{equation}
    \text{universal spin ML model} \quad \Rightarrow \quad \text{universal approximator}.
\end{equation}
In this remark we argue that a partial converse holds true.
Specifically, we argue that under mild additional assumptions universal approximators can approximately simulate arbitrary target systems. Note however, that the  respective simulations do not necessarily have constant degeneracy (cf. \cref{def:mlm}).

First, define the \emph{free energy} of a spin ML system $S$
\begin{equation}
        F_S(\mathbf{v}) = -\log\left(\sum_{\mathbf{h}}e^{-H_S(\mathbf{v},\mathbf{h})} \right).
    \end{equation}
By definition, 
\begin{equation}\label{eq:free energy vis marg}
    \frac{1}{Z_S}e^{-F_S(\mathbf{v})}= \frac{1}{Z_S} \sum_{\mathbf{h}}e^{-H_S(\mathbf{v},\mathbf{h})},
\end{equation}
so whenever the visible spin marginal of $S$ (the RHS of \cref{eq:free energy vis marg}) approximates the Boltzmann distribution of some target system $T$ then 
\begin{equation}\label{eq:free energy approx}
    F_S(\mathbf{v}) \approx H_T(\mathbf{v}) + \log \left(\frac{Z_T}{Z_S}\right).
\end{equation}
In other words, whenever $S$ approximates $p_T$ then the free energy $F_S$ and $H_T$ agree (approximately) up to a constant energy shift.

We now argue that a universal approximator $\mathcal{M}$ suffices to ``approximately'' simulate arbitrary targets $T$ to arbitrary precision, i.e.\  there exists $S \in \mathcal{M}$ with 
\begin{equation}\label{eq:start approx}
\min_{\mathbf{h}}H_S(\mathbf{v},\mathbf{h}) \approx H_{T}(\mathbf{v}) + \Gamma,
\end{equation}

We assume that $\mathcal M$ is closed under
positive rescaling of its Hamiltonians.
Then, for any $\lambda>0$, there exists a $S_{\lambda}\in \mathcal{M}$ that approximates $p_{\lambda T}$, i.e.\ 
\begin{equation}\label{eq: first approx}
        \frac{1}{Z_{S_\lambda}}e^{-F_{S_{\lambda}}(\mathbf{v})}\approx \frac{1}{Z_{\lambda T}}e^{-\lambda H_T(\mathbf{v})},
\end{equation}
Therefore 
\begin{equation}\label{eq:second approx}
        \frac{1}{\lambda}F_{S_\lambda}(\mathbf{v}) \approx  H_T(\mathbf{v}) + \frac{1}{\lambda}\log \left(\frac{Z_{\lambda T}}{Z_{S_\lambda}}\right).
\end{equation}
Note that $-\log(\sum_i e^{-x_i})$ smoothly approximates  the minimum over $\{x_i\}$. In particular, 
\begin{equation}\label{eq:approx min}
\bigl\vert \min_{\mathbf{h}}H_{S_{\lambda}}(\mathbf{v}, \mathbf{h}) - F_{S_{\lambda}}(\mathbf{v})\bigr\vert \leq m\log(2),
\end{equation}
where $m$ is the number of hidden spins of $S_{\lambda}$. We assume that $m$ does not depend on $\lambda$.
Define $S=\frac{1}{\lambda}S_{\lambda}$, then 
\begin{equation}
        \bigl \vert \min_{\mathbf{h}}H_{S}(\mathbf{v}, \mathbf{h}) - \frac{1}{\lambda} F_{S_{\lambda}}(\mathbf{v}) \bigr \vert \leq   \frac{1}{\lambda}m\log(2).
\end{equation}
Choosing $\lambda$ sufficiently large, we thus have
\begin{equation}\label{eq:final approx}
\min_{\mathbf{h}}H_S(\mathbf{v},\mathbf{h}) \approx \frac{1}{\lambda}F_{S_{\lambda}}(\mathbf{v}) \approx  H_T(\mathbf{v}) + \frac{1}{\lambda}\log \left(\frac{Z_{\lambda T}}{Z_{S_\lambda}}\right).
\end{equation}  

\subsection{Extensions of GSE simulation}\label{rem:generalization NP-hard}
    The key idea in \cref{lem:eff sim GSE} is to encode the parameters $\mathbf{m}$ that specify a spin system $S_{\mathbf{m}}$ of a spin model with $\NPclass$-hard ground-state energy problem in the sub-configuration of another spin system $S_n$. In \cref{ssec: separation 2} and \cref{ssec: separation 3} we use
    a similar strategy to prove that (v,p)-maximal models can represent $\SharpPclass$-hard function.

    It is straightforward to generalize these constructions to allow the efficient simulation of arbitrary Ising systems with fields.
    Spin systems $S_i$ with 
    \begin{equation}
        H_{S_i}(v_i,m_i)=\begin{cases}
            v_i \ &\text{if} \ m_i=+1\\
            0 \ &\text{if} \ m_i=-1
        \end{cases}
    \end{equation}
    add fields to the target system $S_{\mathbf{m}}$, using an additional bit of $\mathbf{m}$ per spin to switch the field on and off. 
    $k$-bit couplings $J_{i,j}= \sum_r \frac{(m_r+1)}{2} 2^{r-a}$ (where $a$ is used to tune the range of $J_{i,j}$) can be simulated via 
    \begin{equation}
        J_{i,j}v_iv_j = \sum_r 2^{r-a}H_{S_{i,j}}(v_i,v_j,m_r),
    \end{equation}
    where $S_{i,j}$ is defined in \cref{lem:eff sim GSE}.
    The simulation of $k$-bit fields can be constructed similarly.

\subsection{Explicit DBM simulation}\label{ssec: supp DBM}
     While the concrete simulation of $\mathcal{G}$ of \cref{lem:eff sim GSE} is implicitly obtained from the universality construction \cite{Re24,Re25}, here we explicitly provide this simulation for the (v,p)-maximal model of DBMs. 

    First, we construct an Ising system $T_{i,j}$ that for arbitrary cut-offs simulates $S_{i,j}$ from \cref{lem:eff sim GSE}. This system has one auxiliary spin $a_{i,j}$ and is defined by 
    \begin{equation}
    \begin{split}
        H_{T_{i,j}}(v_i,v_j,m_{i,j},a_{i,j})= &(\Delta+\frac{1}{2})v_iv_j-\Delta v_i m_{i,j}-\Delta v_jm_{i,j}\\
        &+2\Delta v_i a_{i,j}
        +2\Delta v_j a_{i,j}-2\Delta m_{i,j}a_{i,j}\\
        &-\frac{1}{2}v_i-\frac{1}{2}v_j +\frac{1}{2}m_{i,j}-a_{i,j}+3\Delta.
    \end{split}
    \end{equation}
    The values of $T_{i,j}$ are listed in \cref{tab:i-j sim}.
    \begin{table}[t]
    \renewcommand{\arraystretch}{1.2}
        \centering
        \begin{tabular}{l|l}
           \textbf{configuration} $(v_i,v_j,m_{i,j},a_{i,j})$  & \textbf{energy} $H_{T_{i,j}}(v_i,v_j,m_{i,j},a_{i,j})$ \\
           \toprule 
            $(-1,-1,-1,-1)$ &  $2+4\Delta$ \\ 
            $(-1,-1,-1,1)$ &  $0$ \\ 
            $(-1,-1,1,-1)$ &  $3+12\Delta$ \\ 
            $(-1,-1,1,1)$ &  $1$ \\ 
            $(-1,1,-1,-1)$ &  $0$ \\ 
            $(-1,1,-1,1)$ &  $-2+4\Delta$ \\ 
            $(-1,1,1,-1)$ &  $1+4\Delta$ \\ 
            $(-1,1,1,1)$ &  $-1$ \\ 
            $(1,-1,-1,-1)$ &  $0$ \\ 
            $(1,-1,-1,1)$ &  $-2+4\Delta$ \\ 
            $(1,-1,1,-1)$ &  $1+4\Delta$ \\ 
            $(1,-1,1,1)$ &  $-1$ \\ 
            $(1,1,-1,-1)$ &  $0$ \\ 
            $(1,1,-1,1)$ &  $-2+12\Delta$ \\ 
            $(1,1,1,-1)$ &  $1$ \\ 
            $(1,1,1,1)$ &  $-1+4\Delta$ 
        \end{tabular}
        \caption{Values of $H_{T_{i,j}}$ simulating $S_{i,j}$.}
        \label{tab:i-j sim}
    \end{table}
    This proves that $T_{i,j}$ simulates $S_{i,j}$ with cut-off $4\Delta-2$.

    Since $T_{i,j}$ has a fully connected interaction graph, the sum $\sum_{i,j}T_{i,j}$ is not a DBM. To remedy this, we construct a system $T_{i,j}'$ that simulates $T_{i,j}$ for arbitrary cut-offs $\Delta'$ and such that  $\sum_{i,j}T_{i,j}'$ is a DBM.  $T_{i,j}'$ has spins $v_i,v_j,m_{i,j}, a_{i,j}, v_i',v_j'$ and is defined by 
    \begin{equation}
        \begin{split}
            H_{T_{i,j}'}(v_i,v_j,m_{i,j},a_{i,j}, v_i',v_j')= &-\Delta'v_iv_i'-\Delta'v_jv_j'\\
            &+(\Delta+\frac{1}{2})v_i'v_j-\Delta v_i m_{i,j}-\Delta v_jm_{i,j}\\
            &+2\Delta v_i' a_{i,j}
            +2\Delta v_j' a_{i,j}-2\Delta m_{i,j}a_{i,j}\\
            &-\frac{1}{2}v_i-\frac{1}{2}v_j +\frac{1}{2}m_{i,j}-a_{i,j}+3\Delta.
        \end{split} 
    \end{equation}
    $T_{i,j}'$ is illustrated in \cref{fig:T' simulation}. 
    
    Finally, choosing $\Delta'$ large enough the sum $S_n'=\sum_{i,j}T_{i,j}'$ simulates $S_n$ for arbitrary cut-off. $S_n'$ defines a $3$-layer DBM, where the three layers correspond to the configurations $\mathbf{m}$, $(\mathbf{v},\mathbf{a})$, $\mathbf{v}'$ from bottom to top.
    Any interaction of the form $m_{i,j}-v_k$ or $m_{i,j}-a_{k,l}$ is an interaction between layer $0$ and layer $1$. Any interaction $v_i-v_j$ or $v_i-a_{j,k}$ is replaced by an equivalent interaction involving $v_i'$ instead of $v_i$, and such interactions are between layer $1$ and layer $2$.

    \begin{figure}[th]
        \centering
        \begin{tikzpicture}[vis_phys/.style = {draw=black!80, fill=blue!60, thick, circle, minimum size= 0.65em, inner sep=2pt},
        vis_aux/.style = {draw=black!80, fill=blue!20, thick, circle, minimum size= 0.65em, inner sep=2pt},
        hid_phys/.style = {draw=black!80, fill=red!60, thick, circle, minimum size= 0.65em, inner sep=2pt},
        hid_aux/.style = {draw=black!80, fill=red!20, thick, circle, minimum size= 0.65em, inner sep=2pt}, scale = 1]
        \node[vis_phys, label=left: $m_{i,j}$] (m) at (0,0) {};
        \node[hid_phys, label=left: $v_i$] (i) at (-2,2) {};
        \node[hid_phys, label=left: $v_j$] (j) at (0,2) {};
        \node[hid_phys, label=left: $a_{i,j}$] (a) at (2,2) {};
        \node[hid_aux, label=left: $v_i'$] (i') at (-2,4) {};
        \node[hid_aux, label=left: $v_j'$] (j') at (0,4) {};

        \draw[-, thick] (m) to (i);
        \draw[-, thick] (m) to (j);
        \draw[-, thick] (m) to (a);
        \draw[-, thick] (i) -- (i') node[midway, left] {$-\Delta'$};
        \draw[-, thick] (j) -- (j') node[midway, left] {$-\Delta'$};
        \draw[-, thick] (j) to (i');
        \draw[-, thick] (a) to (i');
        \draw[-, thick] (a) to (j');

        \end{tikzpicture}
        \caption{System $T_{i,j}'$ that simulates $T_{i,j}$. The additional spins $v_i',v_j'$ ensure that $T_{i,j}'$ is compatible with the DBM constraints. They
        are coupled to $v_i,v_j$ with strong ferromagnetic coupling $-\Delta'$ such that in the low-energy sector $v_i'=v_i$ and $v_j'=v_j$.}
        \label{fig:T' simulation}
    \end{figure}
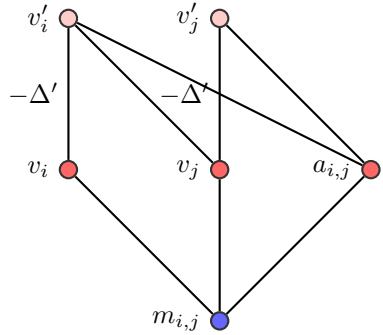

\subsection{On the computability of simulations}\label{rem:polytime simulations}
    We now provide an assumption, termed the computability assumption, which provides even stronger evidence for the separation between shallow and (v,p)-maximal models (cf. \cref{thm:separation 1}). Under the computability assumption, shallow models cannot be (v,p)-maximal unless $\Pclass = \NPclass$, i.e.\ the polynomial hierarchy collapses completely. 
    Specifically, under the computability assumption, \cref{lem:eff sim GSE} implies that for (v,p)-maximal models, computing visible spin effective energies $\min_{\mathbf{h}}H_S(\mathbf{v},\mathbf{h})$ is $\NPclass$-hard, while it is polytime computable for shallow models by \cref{lem:shallow polytime}. 
    
    We first state the computability assumption and argue that it is justified. 
    Note that in \cref{lem:eff sim GSE}, for any $n$ we explicitly construct the system $S_n$ that efficiently simulates $G_n$. This essentially means that the efficient simulation of $\mathcal{G}$ is itself polytime computable. So for given $n$ we can compute the respective $S_n$ and energy shift of $S_n \simul G_n$ in polynomial time.

    In \cite{Re24} we consider a stronger version of simulation between spin models, which additionally requires polytime computability in the above sense. Namely, if $\mathcal{M} \simul \mathcal{N}$ then for any $T\in \mathcal{N}$ the respective $S\in \mathcal{M}$ that satisfies $S \simul T$ can be computed in polynomial time. We then also prove that  universal models satisfy this property for arbitrary targets.

    Now say a spin ML model $\mathcal{M}$ is polytime (v,p)-maximal, that is, it is (v,p)-maximal such that all respective simulations (see \cref{def:efficient universality}) are polytime computable.
    Following \cref{thm:main}, this is the case if $\mathcal{M}$ is visible closed and universal such that
    both the visible closure and the universality simulations can be constructed in polynomial time.
    In particular, this is satisfied for DBMs, using that they are strictly visible closed (cf. \cref{cor:DBM max}) and polytime universal \cite{Re25}. 

    We now prove that for a polytime (v,p)-maximal model $\mathcal{M}$, \cref{lem:eff sim GSE} implies that computing visible spin effective energies $\min_{\mathbf{h}}H_S(\mathbf{v},\mathbf{h})$ is $\NPclass$-hard, because the $\NPclass$-hard ground-state function $\mathcal{G}$ can be reduced to it.
    Given an input $\mathbf{m}$ of $G_n \in \mathcal{G}$, first compute $S_n\in \mathcal{M}$ that simulates $G_n$ and then compute 
    \begin{equation}
       \min_{\mathbf{h}}H_{S_n}(\mathbf{m},\mathbf{h}) = G_n(\mathbf{m}),
    \end{equation}
    up to a potential energy shift that by assumption can also be computed in polynomial time.
    If, additionally, $\mathcal{M}$ were shallow, by \cref{lem:shallow polytime} this would provide a polytime algorithm for the $\NPclass$-hard function $\mathcal{G}$, implying $\Pclass = \NPclass$.

\subsection{Increasing energy gaps through simulations}\label{ssec: supp cut-off}
By \cref{lem:eff sim GSE}, (v,p)-maximal models can efficiently
simulate an $\NPclass$-hard function with cut-off at least
$\Delta=1$. To obtain approximations of the corresponding
probability distributions, however,
we would need constant degeneracy and in particular, sufficiently large cut-off, that ultimately controls the approximation error.
Can these simulations be extended to arbitrary cut-offs?
We argue that, unless the polynomial hierarchy collapses,
such an extension is impossible.

Let
$\mathcal{G}=\{G_n\colon\mathbf{2}^n\to\mathbb{R}
\mid n\in\mathbb{N}\}$
be any family for which the decision problem
\begin{equation}
    L_{\mathcal{G}}
    =\{(\mathbf{v},k)\mid G_n(\mathbf{v})\leq k\}
\end{equation}
is $\NPclass$-complete (see e.g.\ \cref{lem:eff sim GSE}).

Suppose that $\mathcal{G}$ admits efficient simulations
with arbitrary cut-offs. For each $n$, choose
\begin{equation}
    \Delta_n>2\max_{\mathbf{v}}|G_n(\mathbf{v})|,
\end{equation}
and let $S_n$ be a corresponding simulator of size
$\mathrm{poly}(n)$, with energy shift $\Gamma_n$ and
cut-off at least $\Delta_n$. Thus, for every $\mathbf{v}$,
\begin{equation}\label{eq:sim large cut off}
    \begin{aligned}
        H_{S_n}(\mathbf{v},\mathbf{h})
        &=G_n(\mathbf{v})+\Gamma_n,
        &&\mathbf{h}\in\simset(\mathbf{v}),\\
        H_{S_n}(\mathbf{v},\mathbf{h})
        &\geq G_n(\mathbf{v})+\Gamma_n+\Delta_n,
        &&\mathbf{h}\notin\simset(\mathbf{v}).
    \end{aligned}
\end{equation}

The key observation is that such a large cut-off makes
minimizing hidden configurations easy to recognize.
Indeed, our choice of $\Delta_n$ gives
\begin{equation}
    -\frac{\Delta_n}{2}
    <G_n(\mathbf{v})
    <\frac{\Delta_n}{2}.
\end{equation}
Consequently, minimizing configurations have shifted energy
below $\Delta_n/2$, whereas all other configurations have
shifted energy strictly above $\Delta_n/2$. Hence
\begin{equation}
    \mathbf{h}\in\simset(\mathbf{v})
    \quad\Longleftrightarrow\quad
    H_{S_n}(\mathbf{v},\mathbf{h})-\Gamma_n
    <\frac{\Delta_n}{2}.
\end{equation}

We now show directly that
$L_{\mathcal{G}}\in\CoNPclass/\mathrm{poly}$.
Use $(S_n,\Gamma_n,\Delta_n)$ as polynomial-length advice.
A certificate for a no-instance $(\mathbf{v},k)$ is a
hidden configuration $\mathbf{h}$ satisfying
\begin{equation}
    k<H_{S_n}(\mathbf{v},\mathbf{h})-\Gamma_n
    <\frac{\Delta_n}{2}.
\end{equation}
Both inequalities can be checked in polynomial time.
The upper inequality guarantees that $\mathbf{h}$ is
minimizing, so the lower inequality establishes
$G_n(\mathbf{v})>k$.
Conversely, whenever $G_n(\mathbf{v})>k$, any minimizing
hidden configuration provides such a certificate.
The simulator is supplied as advice, so the verifier
does not need to check whether an arbitrary proposed
system satisfies the simulation assumptions.
\footnote{Formally, advice is indexed by the total input
length $N$. It can contain the simulator descriptions
for all $n\leq N$; their combined length remains polynomial.}

Since $L_{\mathcal{G}}$ is $\NPclass$-complete, this implies
\begin{equation}
    \NPclass\subseteq\CoNPclass/\mathrm{poly}.
\end{equation}
By \cite{Ya83}, this inclusion collapses the polynomial
hierarchy to its third level.
Unless such a collapse occurs, efficient simulations
with arbitrary cut-offs therefore cannot exist.
Notice that this obstruction does not require constant
degeneracy: the large cut-off alone would provide
efficiently verifiable certificates of minimality.

A related obstruction applies to polynomial-time constructions (such as simulations)
that preserve the target ground-state energy up to an efficiently
computable shift. If such a construction could produce a source
whose global ground-state energy gap exceeds a known range containing
the target ground-state energy, then source ground states would
provide polynomially verifiable certificates that the target
ground-state energy exceeds a given threshold. For an
$\NPclass$-hard ground-state energy decision problem, this
would imply $\NPclass=\CoNPclass$.

\section{Learning separation between RBMs and 3RBMs}\label{sec:ip-learning-methods}

The separation results in \cref{sec:RBM separation} concern the existence of efficient representations, but they
do not establish whether suitable weights can be learned from a dataset of 
observed configurations. Here, we investigate this question by training
RBMs and 3RBMs (\cref{eq:3RBM}) on samples from known target distributions. We consider an equal budget in terms of the number of adjustable parameters and the amount of training data. We then measure the approximation error as well as the computation time. These finite-size experiments
are a first test toward establishing whether the separation result of \cref{sec:RBM separation} manifests itself in
learning. The code for the learning experiments is publicly available in the
accompanying repository.\footnote{\url{https://github.com/TobiReinhart/boltzmann-learning}}

\subsection{Target distribution and model families}
We consider a target distribution defined in terms of $\mathrm{IP}_2$. \cref{sec:RBM separation} proves that $\mathrm{IP}_2$ can be efficiently simulated by 3RBMs but not by RBMs, and \cref{rem:sim approx} suggests that a similar separation applies to approximating the corresponding probability distribution $q(\mathbf{v}) \propto e^{\beta\mathrm{IP}_2(\mathbf{v})}$ in the respective visible spin marginals.

We work with states from $\{0,1\}$.
Write $\mathbf v=(\mathbf x,\mathbf y)$, where
$\mathbf x,\mathbf y\in\{0,1\}^{n}$, so there are $d=2n$ visible spins.
Define the target $T$ by
\begin{equation}
 s(\mathbf v)=\mathrm{IP}_2(\mathbf x,\mathbf y)
 =(-1)^{\sum_{i=1}^n x_i y_i},\qquad
 H_T(\mathbf v)=-\beta s(\mathbf v),\qquad
 q(\mathbf v)=Z_T^{-1}e^{-H_T(\mathbf v)}.
 \label{eq:lm-target}
\end{equation}
We consider the task of learning $q(\mathbf{v})$ in the visible spin marginals of 3RBMs and RBMs.
Note that $q(\mathbf{v})$ distinguishes configurations based on their inner product modulo $2$, and an even-IP configuration is $e^{2\beta}$ times as probable as an
odd-IP configuration. 

The target can be efficiently sampled as follows:
we propose independent uniform configurations, accepting even-IP proposals
with probability one and odd-IP proposals with probability $e^{-2\beta}$.
The accepted configurations then are independent draws from $q$. We generate a
fixed dataset $\mathcal D=(\mathbf v^{(1)},\ldots,\mathbf v^{(N)})$ before
training. Repeated configurations are allowed. Training receives these
configurations, not their IP values or target energies.

With $m$ hidden spins, the 3RBM Hamiltonian (\cref{eq:3RBM}) can be expressed as 
\begin{equation}
 H_S(\mathbf v,\mathbf h;\theta)
 =\sum_{i=1}^{d}b_i^{(v)}v_i+
 \sum_{r=1}^{m}h_r\ell_r(\mathbf v),\qquad
 \ell_r(\mathbf v)=b_r^{(h)}+\sum_i W_{ri}v_i+
                   \sum_{i<j}C_{rij}v_i v_j.
 \label{eq:lm-model}
\end{equation}
The RBM Hamiltonian is obtained by setting $C=0$. 

Summing over the hidden spins yields free energy (see \cref{rem:sim approx}). The visible spin marginal then can be written as follows:
\begin{equation}
 F_S(\mathbf v;\theta)=\sum_i b_i^{(v)}v_i
       -\sum_r\log\bigl(1+e^{-\ell_r(\mathbf v)}\bigr),\qquad
 p_S^{\mathrm{vis}}(\mathbf v;\theta)=\sum_{\mathbf h}p_S(\mathbf v,\mathbf h;\theta)
       =\frac{e^{-F_S(\mathbf v;\theta)}}{Z_S(\theta)}.
 \label{eq:lm-free}
\end{equation}
Since both RBMs and 3RBMs are shallow, $F_S$ and its derivatives are efficiently computable for any given visible configuration (see \cref{sec:consequences}),
although $Z_S=\sum_{\mathbf v}e^{-F_S(\mathbf v)}$ generally is not. 
\subsection{The likelihood objective and its gradient}

The objective we aim to minimize in training is the \emph{Kullback--Leibler} divergence between the target distribution $q$ and the model distribution $p_S^{\mathrm{vis}}$, 
\begin{equation}
 D_{\mathrm{KL}}(q\Vert p_S^{\mathrm{vis}})
 =\sum_{\mathbf v}q(\mathbf v)\log\frac{q(\mathbf v)}{p_S^{\mathrm{vis}}(\mathbf v)}.
 \label{eq:lm-kl}
\end{equation}
Note that in the training we do not have access to the distribution $q(\mathbf{v})$, but only the dataset $\mathcal{D}$ sampled from it.

We hence approximate \cref{eq:lm-kl} by accessible quantities as follows.
First, inserting the model distribution (\cref{eq:lm-free}) into \cref{eq:lm-kl} yields
\begin{equation}
    D_{\mathrm{KL}}(q\Vert p_S^{\mathrm{vis}})
 =\mathbb E_q[F_S]+\log Z_S-\mathcal H(q),
 \label{eq:lm-kl2}
\end{equation}
where $\mathcal H(q)=-\sum q\log q$ is independent of the model weights $\theta$, so minimization only concerns the first two summands. 
Next, replacing the target
average $\mathbb E_q[F_S]$ by the average over observations defines the negative log likelihood per observation,
\begin{equation}
 \mathcal L_N(\theta)=\frac1N\sum_{a=1}^{N}F_S(\mathbf v^{(a)};\theta)
                          +\log Z_S(\theta).
 \label{eq:lm-empirical}
\end{equation}
This objective is accessible from the dataset and the model parameters.
For fixed parameters, the first term approximates the target
average, by
the law of large numbers. Hence minimizing \cref{eq:lm-empirical} can be seen as a reasonable approximation of minimizing \cref{eq:lm-kl}.

Finally, we minimize \cref{eq:lm-empirical} by \emph{gradient descent}, that is, we compute its gradient w.r.t.\ the model parameters and then adjust these in the direction of the negative gradient, by an amount parametrized by the so-called \emph{learning rate}.
Differentiating \cref{eq:lm-empirical} yields
\begin{equation}
 \nabla_\theta\mathcal L_N
 =\frac1N\sum_a\nabla_\theta F_S(\mathbf v^{(a)})
       -\mathbb E_{p_S^{\mathrm{vis}}}[\nabla_\theta F_S(\mathbf v)].
 \label{eq:lm-gradient}
\end{equation}
Note that here $\mathbb E_{p_S^{\mathrm{vis}}}[\nabla_\theta F_S(\mathbf v)]$ is the model average of the gradient of the free energy.
Thus, computing $\nabla_\theta\mathcal L_N$ requires the computation of the average of $\nabla_\theta F_S(\mathbf v)$ w.r.t.\ the dataset and the average of $\nabla_\theta F_S(\mathbf v)$ w.r.t.\ the model distribution.
This, in principle, would entail iterating over all samples as well as computing the full model distribution (including $Z_S$), for each individual update. To avoid this, we use further approximations.

\subsection{Approximating the gradient --- mini-batches and Monte Carlo evolutions}

First, in order to avoid computing averages over the entire dataset, we divide it into so-called \emph{mini-batches}, that is, groups of $b$ observations from $\mathcal{D}$ and, for each update, compute averages over one minibatch. That is, we use the following approximation
\begin{equation}\label{eq:minbatch approx}
    \frac1N\sum_a\nabla_\theta F_S(\mathbf v^{(a)}) \approx \frac1b\sum_{a=1}^b\nabla_\theta F_S(\mathbf v^{(a)}).
\end{equation}

Second, in order to avoid computing the full model distribution $p_S^{\mathrm{vis}}$, or averages w.r.t.\ it (which is intractable), we sample from it by Markov chain Monte Carlo (MCMC) methods.
That is, we set up $b'$ Markov chains with starting points $\mathbf{v}_{a'}^{(0)}$. At each step $i$ the Markov chain first updates the hidden spins by sampling from the conditional distribution $\mathbf{h}_{a'}^{(i)} \sim p_S(\mathbf{h} \mid \mathbf{v}_{a'}^{(i)})$.

Since both RBMs and 3RBMs are shallow, i.e.\ contain no hidden-hidden interactions, the hidden spins are conditionally independent given the visible configuration and hence can be sampled in parallel, i.e.\ the entire hidden configuration can be sampled at once. 

For RBMs, there are also no visible-visible interactions. Thus also visible spins are conditionally independent given the hidden configuration. The Markov chain then proceeds by 
updating the visible spins, by sampling from the conditional distribution $\mathbf{v}_{a'}^{(i+1)} \sim p_S(\mathbf{v} \mid \mathbf{h}_{a'}^{(i)})$.

For 3RBMs there are visible-visible interactions and therefore the visible spins cannot be sampled in parallel. Fixing the hidden configuration the resulting system consisting of visible spins, only, effectively is an all-to-all Ising system. Its partition function and probability distribution hence remain intractable. For 3RBMs we hence first divide the visible spins into blocks of $4$. We then iterate through these blocks, for each block $j$ we update the respective configuration  $\mathbf{v}_{a',j}^{(i)}$ by sampling from the conditional distribution
\begin{equation}
    \mathbf{v}_{a',j}^{(i+1)} \sim p_S(\mathbf{v}_{j} \mid \mathbf{h}_{a'}^{(i)}, \mathbf{v}_{a',1}^{(i+1)}, \ldots , \mathbf{v}_{a',j-1}^{(i+1)}, \mathbf{v}_{a',j+1}^{(i)}, \ldots ).
\end{equation}
That means, when updating block $j$ in iteration $i$, i.e.\ when sampling $\mathbf{v}_{a',j}^{(i+1)}$ those visible blocks $l$ for $l<j$ are fixed to their post-update configuration $\mathbf{v}_{a',l}^{(i+1)}$ while those blocks with $l>j$ are fixed to their pre-update configuration $\mathbf{v}_{a',l}^{(i)}$.

The total number of iterations in this Markov chain is also referred to as sweeps. Especially for RBMs, the corresponding training algorithm with $k$ update rounds is called contrastive divergence $k$ or CD-$k$ \cite{Hi02}.
It is common practice to use an equal number of observations per minibatch $b$ and Markov chains $b'$ and to further initialize the Markov chains at the respective minibatch observation, i.e.\ use $\mathbf{v}_a^{(0)}=\mathbf{v}^{(a)}$. With this the CD-$k$ gradient approximation becomes 
\begin{equation}
 \widehat g=\frac1b\sum_{a=1}^{b}
 \left[\nabla_\theta F_S(\mathbf v_a^{(0)};\theta)
            -\nabla_\theta F_S(\mathbf v_a^{(k)};\theta)\right].
 \label{eq:lm-cd}
\end{equation}
Importantly, for large enough minibatch size $b$ and large enough $k$, i.e.\   long enough Markov chains $\widehat g$ is a reasonable approximation of the true gradient $\nabla_\theta\mathcal L_N$.

\subsection{Updating the weights}

Learning by gradient descent updates the model parameters as 
$\theta_{t+1}=\theta_t-\eta_t\widehat g_t$, where the positive
\emph{learning rate} $\eta_t$ sets the size of each parameter change.
We use Adam~\cite{Ki15},
which additionally accounts for recent update directions and their typical magnitudes.
Its coordinatewise update is
\begin{equation}
 \theta_{t+1}=\Pi_{[-B,B]^P}\left[
 \theta_t-\eta_t\frac{M_t/(1-\rho_1^t)}
 {\sqrt{V_t/(1-\rho_2^t)}+\varepsilon}\right],
 \label{eq:lm-adam}
\end{equation}
where $M_t=\rho_1M_{t-1}+(1-\rho_1)g_t$, $V_t=\rho_2V_{t-1}+(1-\rho_2)g_t^{2}$, 
$M_0=V_0=0$, $\rho_1=0.9$, $\rho_2=0.999$ and $\varepsilon=10^{-8}$. Products, squares and division act coordinatewise.
Further $g_t=\widehat g_t\min(1,10/\|\widehat g_t\|_2)$, with
$g_t=0$ when the gradient vanishes;  this ensures that unusually large gradients (with $\|\widehat g_t\|_2 > 10$) are rescaled before the update. The projection $\Pi_{[-B,B]^P}$ clips each coefficient to $[-B,B]$; this is compatible with the small weight assumption in \cref{sec:RBM separation}.

Finally, couplings are initialized independently from a centered Gaussian with standard
deviation $0.05$, and visible and hidden fields are initialized to zero.

We now compare the two model families on the IP target, always with
$\beta=1$ and coefficient bound $B=4$. The 3RBMs include all visible pairs
$i<j$ in the hidden fields, rather than being restricted to the products
$x_i y_i$. The numbers of adjustable parameters are
\begin{equation}
 P_{\mathrm{RBM}}=d+(d+1)m,\qquad
 P_{3\mathrm{RBM}}=d+\left(1+d+\binom d2\right)m,
 \qquad d=2n.
 \label{eq:ip-results-parameters}
\end{equation}
We choose different hidden-layer sizes to match these numbers as closely
as possible. The following comparisons separate approximation by direct
population optimization from learning using only observed configurations.
\subsection{Learning results}\label{ssec:results learning}
All experiments were performed on a MacBook Air with an
Apple M4 chip and 16\,GB of unified memory, running
macOS 26.6.2, Python 3.12.14, and PyTorch 2.13.0.
Sample-based training used the Apple MPS backend with
single-precision arithmetic (float32), whereas exact
population optimization used CPU-based L-BFGS-B with
double-precision arithmetic (float64).
\paragraph{Exact population optimization at $n=6$.}
As a small-system reference, we enumerate all $2^{12}=4096$ visible
configurations and evaluate $D_{\mathrm{KL}}(q\Vert p_S^{\mathrm{vis}})$
and its gradient without sampling error. We minimize this objective with
the bounded quasi-Newton method L-BFGS-B, which uses gradients and an
approximation to local curvature. There is no sampled training set and
no CD approximation in this experiment. Each model receives five starts:
one with zero coefficients and four with independent Gaussian
coefficients of standard deviation $0.25$. The limits per start are
$3000$ optimizer iterations and $6000$ objective evaluations; the
function and projected-gradient tolerances are $10^{-13}$ and $10^{-9}$.
\cref{tab:ip-population-results} reports the lowest KL found for each
model. At the smallest budget the RBM is slightly better. At the middle
budget the 3RBM error is about half the RBM error, and at the largest
budget it is about $74$ times smaller.
This is consistent with the explicitly constructed  3RBM $\mathrm{IP}$ simulation that following \cref{lem:IP k sim} uses $2n$ visible spins,
$m=2n+1$ hidden spins, and $\mathcal{O}(n^2)$ interactions (see \cref{rem:3RBM IP}).
The respective KL values are numerical upper
bounds on the best achievable errors, not certified optima: every best
fit in the table reached the iteration limit. This comparison removes
finite-data and Monte Carlo errors, but not the uncertainty of
nonconvex optimization.

\begin{table}[htbp]
 \centering
 \small
 \begin{tabular}{rrrrrr}
  \toprule
  \multicolumn{3}{c}{RBM} & \multicolumn{3}{c}{3RBM}\\
  \cmidrule(lr){1-3}\cmidrule(lr){4-6}
  $m$ & $P$ & Best KL & $m$ & $P$ & Best KL\\
  \midrule
  24  & 324  & $0.247303$ & 4  & 328  & $0.253821$\\
  73  & 961  & $0.136275$ & 12 & 960  & $0.065374$\\
  219 & 2859 & $0.014836$ & 36 & 2856 & $0.000200$\\
  \bottomrule
 \end{tabular}
 \caption{Exact population optimization of the IP target at $n=6$,
 $\beta=1$, $B=4$. Each entry is the best of five starts with the same
 optimization limits. The KL values are evaluated by full enumeration,
 but the minima are not certified.}
 \label{tab:ip-population-results}
\end{table}

\paragraph{Sample-based learning conventions.}

For sample-based learning, \cref{tab:ip-learning-budgets} fixes the
budgets. Each size has three independent repetitions, with a new dataset
and new random initialization. Within a repetition, the two architectures
use the same dataset and minibatch order. Training traverses shuffled
copies of this fixed dataset repeatedly; $Ub$ is the number of processed
observations, counting repeated uses, rather than the number of independent
draws. Both architectures use CD-2 and Adam as described above. In the
3RBM, each visible sweep uses a new random partition into four-spin
blocks and samples each block's conditional distribution exactly.
All learning results below use the \emph{final} update, without selecting
the best intermediate checkpoint. The settings were informed by earlier
exploratory experiments; the three repetitions assess reproducibility
at these settings, not exhaustive optimization of each architecture.

At $n=8$ we evaluate population KL by summing over all $2^{16}$ states.
At $n=12$, where the state space has $2^{24}$ elements, these runs do not
evaluate the model partition function. Instead we compare the learned
free energy with the target Hamiltonian through
\begin{equation}
 R_S^2=
 \frac{\operatorname{Var}_{u_d}[F_S(\mathbf v;\theta)-H_T(\mathbf v)]}
      {\operatorname{Var}_{u_d}[H_T(\mathbf v)]},
 \qquad u_d(\mathbf v)=2^{-d}.
 \label{eq:ip-results-energy-error}
\end{equation}
Subtracting the mean difference inside the variance removes the
irrelevant additive energy constant. Consequently $R_S=0$ means exact
agreement of the distributions, whereas a constant $F_S$ gives $R_S=1$.
We estimate both variances on independent uniform probe configurations,
shared between models within a repetition. This tests the shape of the
energy landscape, including variations within each IP sector; it is
\emph{not} an estimate of KL divergence. In particular, a small average
energy error alone does not exclude rare configurations with large
probability errors. We report this diagnostic also at $n=8$, alongside
exact KL, in \cref{tab:ip-learning-results}.

\begin{table}[htbp]
 \centering
 \small
 \begin{tabular}{lrr}
  \toprule
  Quantity & $n=8$ & $n=12$\\
  \midrule
  Visible spins $d$ & 16 & 24\\
  Hidden spins $m$, RBM / 3RBM & $516\,/\,64$ & $1734\,/\,144$\\
  Parameters $P$, RBM / 3RBM & $8788\,/\,8784$ & $43374\,/\,43368$\\
  Independent training draws $N$ & $131\,072$ & $33\,554\,432$\\
  Minibatch size $b$ (also $b'$) & $256$ & $4096$\\
  Parameter updates $U$ & $50\,000$ & $480\,000$\\
  Processed observations $Ub$ & $12\,800\,000$ & $1\,966\,080\,000$\\
  Equivalent dataset passes $Ub/N$ & $97.65625$ & $58.59375$\\
  CD sweeps $k$ per update & 2 & 2\\
  Initial learning rate $\eta_0$ & $0.002$ & $0.00133333\ldots$\\
  Learning-rate schedule & see caption & constant\\
  Independent repetitions & 3 & 3\\
  Uniform energy-evaluation probes & $8192$ & $32\,768$\\
  Population KL evaluation & exact & not performed\\
  \bottomrule
 \end{tabular}
 \caption{Budgets for learning the IP distribution. All runs use $\beta=1$,
 $B=4$, zero initial fields, Gaussian initial couplings of standard
 deviation $0.05$, and gradient-norm clipping at $10$. At $n=8$ the
 learning rate stays at $0.002$ for the first $25\,000$ updates, then
 decreases along a cosine schedule to $0.0002$. At $n=12$ it remains
 $0.004/3$ throughout. These are matched budgets within each size,
 not a comparison at fixed resources across sizes.}
 \label{tab:ip-learning-budgets}
\end{table}

\paragraph{Learning from samples at $n=8$.}
With $64=n^2$ hidden spins for the 3RBM and a parameter-matched RBM,
all three repetitions show a clear advantage for the 3RBM
(\cref{fig:ip-learning-success}, top row). The median final KL is
$0.10414$ for the 3RBM, compared with $0.32633$ for the RBM. For reference,
the uniform visible distribution has KL $0.32648$ relative to the
target. Thus the 3RBM removes about $68\%$ of this reference error,
whereas the RBM remains close to it. The energy diagnostic is consistent
with this conclusion: the median $R_S$ is $0.562$ for the 3RBM and
$0.999$ for the RBM. This is substantial but incomplete learning of IP,
not near-exact recovery of the distribution. The 3RBM takes longer per
run at this size (median $5.17$ versus $3.36$ minutes), but its error
already improves below the RBM level within the common time interval.
The advantage therefore is not solely due to its longer total runtime.

\begin{table}[htbp]
 \centering
 \small
 \begin{tabular}{llccc}
  \toprule
  Size & Model & Final KL & Final $R_S$ & Training time\\
  \midrule
  $n=8$ & RBM & $0.32633\;[0.32614,0.32647]$
  & $0.999\;[0.998,0.999]$ & $3.36$ min\\
  & 3RBM & $0.10414\;[0.10107,0.10547]$
  & $0.562\;[0.558,0.565]$ & $5.17$ min\\
  \addlinespace
  $n=12$ & RBM & ---
  & $1.003\;[1.002,1.003]$ & $3.07$ h\\
  & 3RBM & ---
  & $0.512\;[0.505,0.553]$ & $1.82$ h\\
  \bottomrule
 \end{tabular}
 \caption{Final-checkpoint results: median and full range in brackets
 over three repetitions. Ranges describe run-to-run variation, not
 confidence intervals. KL is exact at $n=8$ and unavailable here at
 $n=12$; $R_S$ is estimated from uniform probes at both sizes.
 Times are medians of measured training wall time, excluding subsequent checkpoint evaluation and dataset generation.}
 \label{tab:ip-learning-results}
\end{table}

\paragraph{Learning from samples at $n=12$.}
We next increase the 3RBM width to $144=n^2$ and enlarge both the data
and optimization budgets. Relative to $n=8$, the dataset is $256$ times
larger, the minibatch is $16$ times larger, and the number of updates is
$9.6$ times larger. All three 3RBM repetitions now learn substantial IP
structure, with final energy errors between $0.505$ and $0.553$, whereas
all three RBMs remain close to the constant-energy reference
(\cref{fig:ip-learning-success}, bottom row). The median correlation
between $F_S$ and $H_T$ on uniform probes is $0.860$ for the 3RBM and
$-0.003$ for the RBM. As an additional check, the difference between the
mean free energies in the odd- and even-IP sectors is about $70\%$ of
the target difference $2\beta$ for the 3RBM, but essentially zero for
the RBM. These observations show recovery of the IP-dependent energy
structure on independently drawn evaluation configurations.
They do not establish a numerical KL value at this size. The 3RBM is
also faster here: its median training time is $1.82$ hours, compared
with $3.07$ hours for the parameter-matched RBM. Parameter matching
gives the RBM many more hidden spins, so its simpler conditional
sampling does not necessarily make it faster. This timing comparison
depends on the implementation, hardware and minibatch size; it is not
an architecture-independent complexity bound.

\begin{figure}[htbp]
 \centering
 \includegraphics[width=\linewidth]{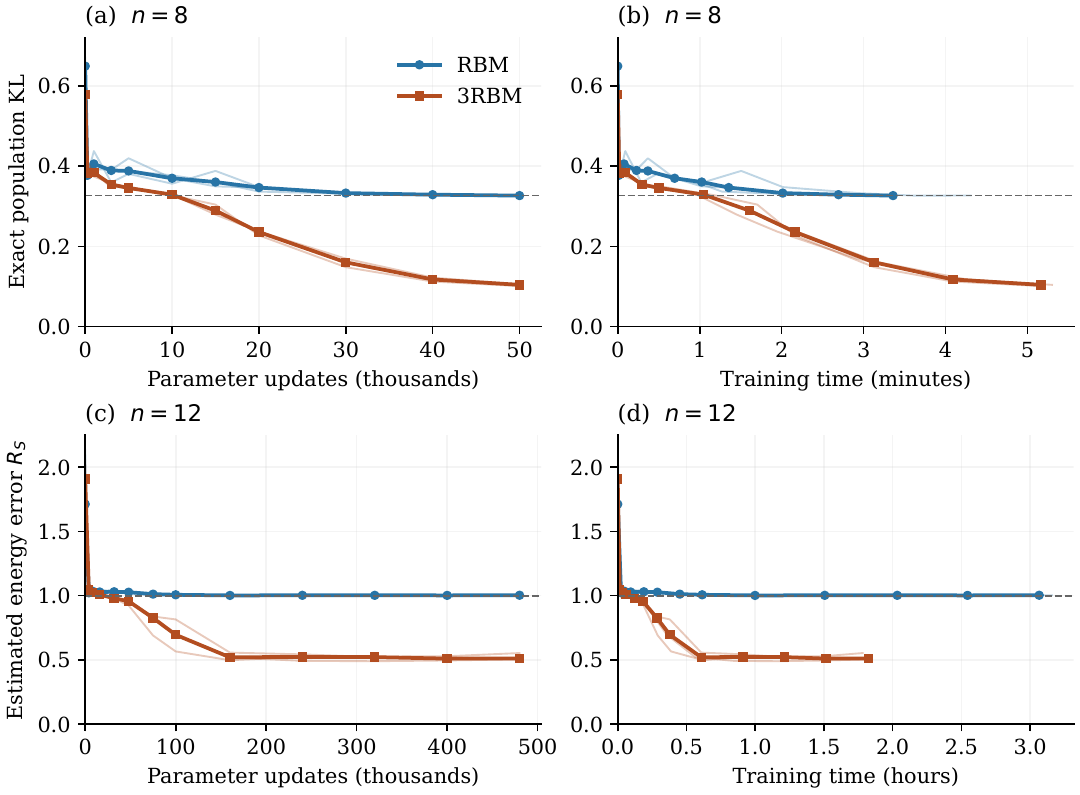}
 \caption{Learning the IP target with parameter-matched RBMs and
 all-pairs 3RBMs. Top: exact population KL at $n=8$. Bottom: estimated
 energy error $R_S$ at $n=12$. Left: error versus parameter updates;
 right: the same checkpoints versus measured training time. Thin lines
 show all three repetitions. Thick lines join checkpoint medians; in
 the time panels, both time and error are medians at each common update
 count, rather than interpolated equal-time averages. Dashed lines give
 the uniform-distribution KL (top) and constant-energy error $R_S=1$
 (bottom). The two rows use different error measures and time units.}
 \label{fig:ip-learning-success}
\end{figure}

\paragraph{Unsuccessful extension to $n=16$.}
We retain the $n=12$ learning procedure and budgets
($N=33\,554\,432$, $b=4096$, $U=480\,000$, CD-2,
constant learning rate $0.004/3$, $\beta=1$, $B=4$), but increase the
3RBM width to $256=n^2$ and the parameter-matched RBM width to $4104$
($135\,456$ and $135\,464$ parameters, respectively). In this single
paired repetition, the final energy errors are $R_S=0.9966$ for the
3RBM and $1.0057$ for the RBM. Both therefore remain close to the
constant-energy reference. The 3RBM recovers only about $0.5\%$ of the
target mean energy difference between IP sectors. This failure is
measured with the same diagnostic and $32\,768$ uniform probes as at
$n=12$, so it cannot be attributed simply to switching error measures.
A possible explanation is that increasing $n$ makes the IP-dependent
direction in the learning gradient harder to resolve, while the
minibatch size, number of updates and CD chain length remain unchanged.
Statistical fluctuations, incomplete equilibration in CD-2 and the
optimization landscape could then prevent the model from finding a
useful representation. A follow-up replaces dataset reuse by fresh
independent observations at every update, keeping the 3RBM initialization
and other learning settings fixed. Even with $1\,966\,080\,000$
independent draws, its final error remains $R_S=0.9966$. This weakens
the explanation that the fixed dataset alone was too small, but does
not test larger minibatches, longer chains or more updates. With only
one realization per condition, these results locate a limitation of
the tested learning procedure; they neither identify its precise cause
nor demonstrate a failure of 3RBM representability.
\paragraph{Larger-budget learning at $n=16$.}
To test whether the preceding failure could be overcome by increasing
the learning resources, we perform a further 3RBM run with
$m=256$ hidden spins, retaining $\beta=1$, $B=4$, the previous random
initialization and the constant learning rate $0.004/3$. We increase
the minibatch size from $4096$ to $32\,768$ and the number of Gibbs
sweeps from two to four. Each update uses fresh independent target
observations. To keep the runtime manageable, we use $160\,000$
parameter updates, giving $5\,242\,880\,000$ observations in total
and a measured training time of $13.8$ hours. The final energy error
is $R_S=0.9948$, compared with $0.9966$ in the previous fresh-data
run. The energy correlation increases slightly to $0.108$, but the
learned mean energy difference between IP sectors reaches only
$0.77\%$ of the target difference. Moreover, the error already reaches
$R_S=0.9964$ after $40\,000$ updates and subsequently improves only
marginally. Thus the larger budget does not recover the substantial
IP structure learned at $n=12$. This weakens the explanation that
the earlier failure resulted solely from an insufficient fixed
dataset. A weak useful gradient, difficulties in the optimization
landscape, bias from the finite CD chains, or the coefficient bound
remain possible explanations. The experiment does not distinguish
these mechanisms: several resources change simultaneously, the
number of updates is smaller than in the preceding $n=16$ runs,
and only one initialization is tested. The result therefore
documents a persistent limitation of the tested learning procedure,
rather than a failure of 3RBM representability.

\paragraph{Other targets: advantages and controls.}
To test whether a learning advantage extends beyond IP, we consider
$d=16$ visible spins $v_i$ with states in $\mathbf{2}$. We sample
$q(\mathbf v)\propto e^{-H_T(\mathbf v)}$ for different target systems $T$ that involve four-body interactions described by hyperedges $A=\{i,j,k,l\}$, i.e.\ 
$H_T=-M^{-1/2}\sum_{A}J_A\prod_{i\in A}v_i$, with
$J_A\in\{-1,1\}$. Specifically, we either consider four disjoint quartets $A_1,\ldots, A_4$ and then set $M=4$ or
we consider a ring with $A_i=\{i,r(i),8+i,8+r(i)\}$ for $i=1,\ldots,8$ and  $M=8$, where $r(i)=1+(i\bmod 8)$.

Considering appropriate product spins, both cases result in a particularly simple structure.
For each disjoint quartet $A_i=\{a_i,b_i,c_i,d_i\}$,
set $z_{i,1}=v_{a_i}v_{b_i}$ and $z_{i,2}=v_{c_i}v_{d_i}$.
Then the respective interaction is proportional to $z_{i,1}z_{i,2}$.
For the ring, set $z_i=v_iv_{8+i}$; then interactions are
proportional to $z_i z_{r(i)}$.
We expect the 3RBM structure to be better suited to encode  the state of such product spins in its hidden configuration and hence to better learn these targets.

Controls are pairwise
Ising models on an open two-leg ladder, with random couplings
$J_{ij}=\pm1$ and energy divided by its standard deviation under
uniform configurations, and independent spins with
$H_T=-\sum_i h_iv_i$, $h_i\sim\mathrm{Unif}[-0.7,0.7]$.
For each family we use two coupling/field instances and three paired
dataset/initialization repetitions. RBMs and all-pairs 3RBMs have
$m=516$ and $64$ hidden spins ($8788$ and $8784$ parameters),
$N=131\,072$, $b=256$ and $B=4$. Both CD-1 and CD-2 are trained
for $30\,000$ updates ($20\,000$ for independent spins), using
Adam with learning rate $0.002$, constant for the first half and
cosine-decayed to $0.0002$ thereafter; initialization, gradient
clipping and four-spin Gibbs blocks are as above.
\cref{fig:other-learning-examples} shows exact population KL
after validation-based selection. Across the four four-spin instances,
median errors are $0.446$--$0.474$ for RBMs and $0.0149$--$0.0191$
for 3RBMs. Independent-spin errors are comparable
($0.0028$--$0.0035$), whereas the two ladder instances give
RBM/3RBM errors $0.00702/0.0104$ and $0.00750/0.00619$:
the 3RBM is worse on one and better on the other.
Thus the observed advantage extends to some non-IP targets, but is
not universal. This is consistent with an architectural preference
for certain higher-order structures under the tested training
procedure; it neither isolates the mechanism nor establishes an
asymptotic representation separation for these targets.

\begin{figure}[htbp]
 \centering
 \includegraphics[width=\linewidth]{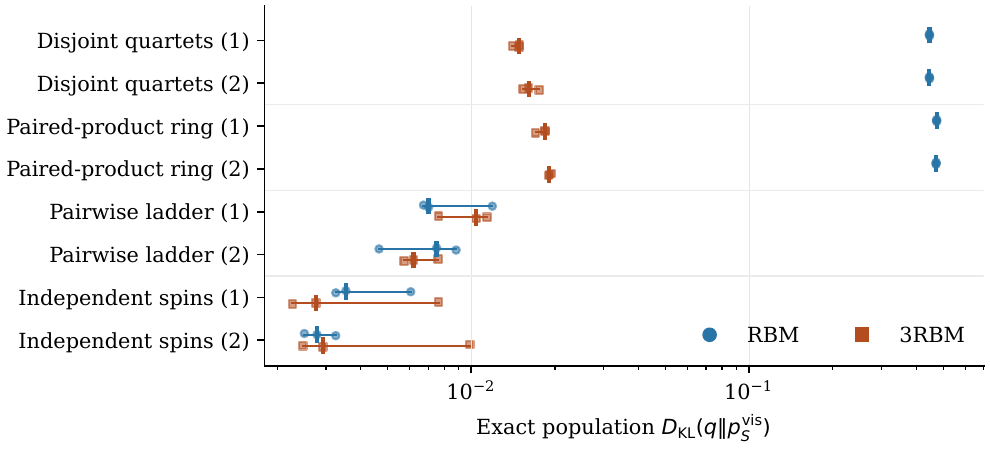}
 \caption{Representative non-IP comparisons at $d=16$. Smaller KL is
 better; the horizontal axis is logarithmic. Each target has two
 instances, each with three training repetitions (small symbols);
 vertical ticks mark medians and horizontal lines span the full range,
 not confidence intervals. Unlike the final-checkpoint IP results,
 CD depth and checkpoint are chosen by maximizing mean single-bit
 conditional log probability on $2048$ independent validation
 configurations, not by minimizing the plotted KL. Candidate
 checkpoints are $250$, $1000$, $3000$, $10\,000$, $20\,000$, and
 additionally $30\,000$ for interacting targets. Both architectures
 receive the same data and search budgets, but not equal wall time.
 These are selected families from a broader exploratory comparison.}
 \label{fig:other-learning-examples}
\end{figure}

\paragraph{Conclusion.}
Taken together, the successful $n=8$ and $n=12$ experiments provide finite-size evidence that an
advantage of 3RBMs on the IP target can survive learning from samples,
including with all visible-pair interactions allowed. The stronger
claim of efficient learning as $n$ grows is not established. In
particular, the $n=12$ success uses $153.6$ times as many processed
observations as the $n=8$ experiment, and several budgets change
simultaneously. These comparisons alone cannot identify which increase
was essential. Nor does a positive learned RBM error certify an
approximation lower bound: it may reflect limited optimization or the
biased CD-2 gradient as well as model capacity. The results therefore
test a learning consequence suggested by the separation theorems,
without constituting a numerical proof of their asymptotic statement.
The additional non-IP comparisons suggest that this learning advantage extends
to some higher-order spin targets, while independent-spin and pairwise controls show comparable performance or, in some instances, lower errors for the RBM, rather than a universal benefit of three-body
couplings.

\clearpage

\printbibliography 

\end{document}